\PassOptionsToPackage{table}{xcolor}
\documentclass[11pt, letterpaper, logo]{appier-ai-research}
\usepackage[comma,numbers,sort,compress]{natbib}
\hypersetup{colorlinks=true, citecolor=magenta}
\pdftrailerid{redacted}
\renewcommand{\today}{2026-09-29}
\AtBeginDocument{}

\usepackage{graphicx}
\usepackage[table]{xcolor}

\usepackage{tikz}
\usetikzlibrary{
    arrows.meta,
    positioning,
    fit,
    calc,
    decorations.pathreplacing
}

\usepackage[utf8]{inputenc} 
\usepackage[T1]{fontenc}    
\usepackage{hyperref}       
\usepackage{url}            
\usepackage{booktabs}       
\usepackage{longtable}      
\usepackage{amsfonts}       
\usepackage{nicefrac}       
\usepackage{microtype}      
\usepackage{xcolor}         
\usepackage{graphicx}
\usepackage{amsmath}
\usepackage{amssymb}
\usepackage{amsthm}
\usepackage{enumitem}       
\usepackage{pifont}         
\newcommand{\cmark}{\ding{51}}

\newtheorem{theorem}{Theorem}

\newtheorem{corollary}{Corollary}

\newtheorem{proposition}{Proposition}
\newtheorem{remark}{Remark}

\usepackage[nottoc]{tocbibind}
\usepackage{minitoc}
\usepackage{bbm}

\newlength{\paraSpace}
\makeatletter
\renewcommand{\paragraph}{\@startsection{paragraph}{4}{\z@}%
  {\paraSpace}{-1em}%
  {\normalfont\normalsize\bfseries}}
\makeatother

\newlength{\figpullup}
\newcommand{\prismlogo}{\IfFileExists{figures/prism_5.png}{\hspace{0.05cm}\raisebox{-.15\height}{\includegraphics[scale=0.02]{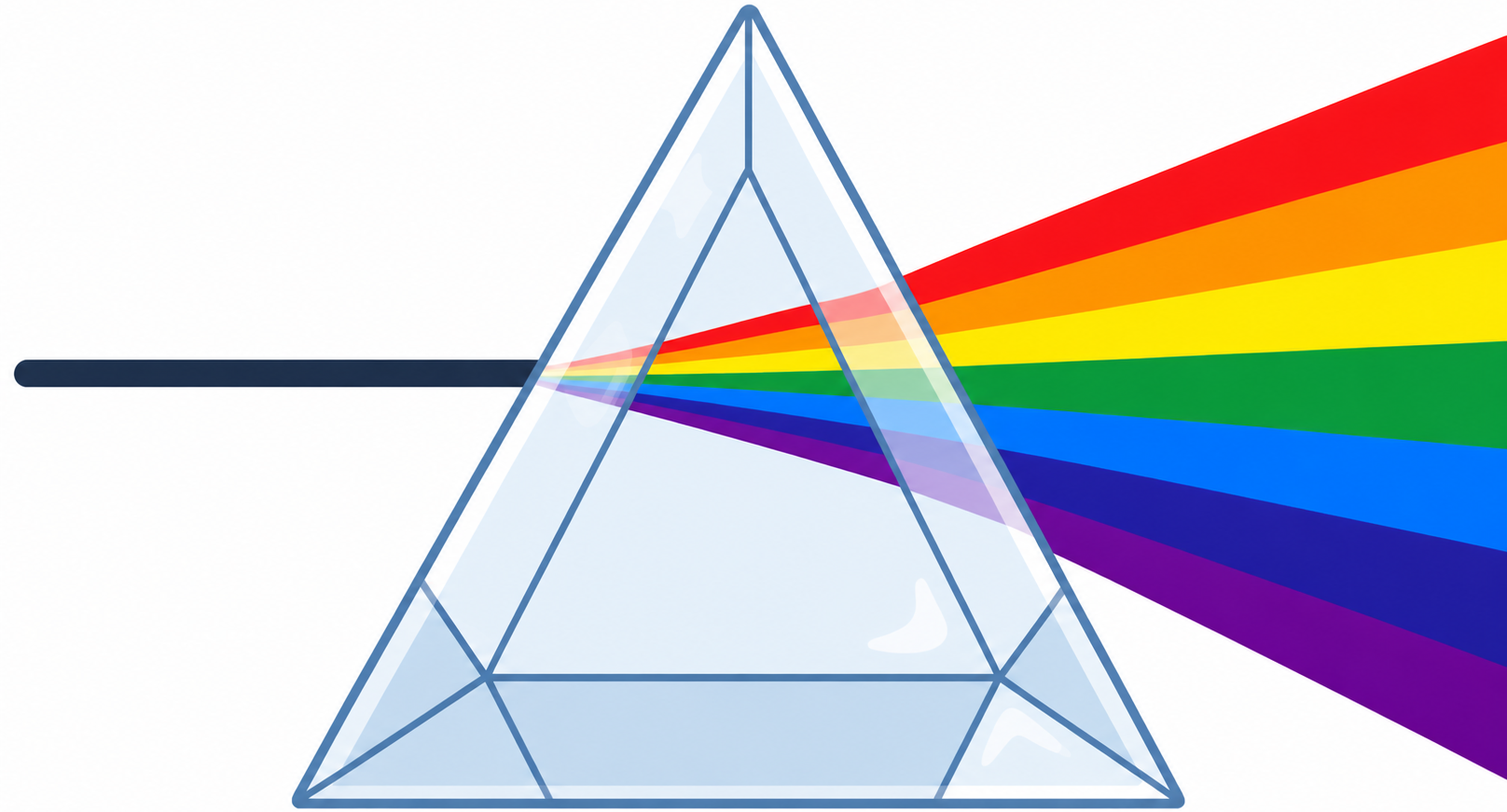}}}{}}

\title{PRISM\prismlogo: A Geometric Risk Bound for Decomposing Drift into Scale, Shape, and Head}

\author[1]{Chieh-Yen Lin}
\author[1,2]{Shao-Hua Sun}
\affil[1]{Appier AI Research}
\affil[2]{National Taiwan University}
\correspondingauthor{cylin@appier.com}
\begin{abstract}
A single base LLM now comes with dozens of post-training variants, quantized, LoRA-adapted, or distilled, and each has to be checked before release.
Existing evaluations provide only a partial picture: benchmark scores and likelihood screens say \emph{that} a variant has degraded, similarity scores such as CKA and SVCCA say \emph{how} its features moved, and nothing connects the two.
We connect them with one structural fact and one design choice: the prediction head is linear, so feature geometry reaches the loss, and we compare the two feature sets through an orthogonal map, which leaves the geometry being measured unchanged.
From these we derive \textbf{PRISM}, a closed-form upper bound on the cross-entropy risk gap between a target model and a proxy variant, and prove that it splits \emph{exactly} into three measurable axes: scale, shape, and head.
Each axis names a failure mode and where to intervene: low-bit quantization distorts shape and, at the lowest bit-widths, also inflates activation scale; quantizing the output embedding inflates the head term, which dominates the bound at high bit-widths.
Because the shape term is differentiable, the same geometry becomes a regularizer that curbs catastrophic forgetting more than experience replay.
Across two model families and five benchmarks, PRISM ranks quantized and fine-tuned variants from a single forward pass at mean Spearman above $0.8$, and eight reference sequences already recover the ranking that the risk gap itself needs over a hundred to reach.

\end{abstract}

\begin{document}
\doparttoc 
\faketableofcontents 

\maketitle

\section{Introduction}

As large language models (LLMs) move from pre-training to deployment, a new bottleneck emerges: a single base model now produces many post-training variants, quantized (GPTQ~\citep{frantar2022gptq}, GGUF, BitsAndBytes~\citep{dettmers2022gpt3}), LoRA-adapted~\citep{hu2022lora}, or distilled~\citep{hinton2015distilling}, that must be evaluated before release~\citep{grattafiori2024llama, qwen2025qwen3}. Existing evaluation largely relies on aggregate accuracy or perplexity, which can reveal that a variant has degraded but not \emph{why}, leaving developers to debug low-bit quantization failures, catastrophic forgetting, or prediction-head corruption by costly trial-and-error. What is missing is a diagnostic that not only predicts degradation, but also identifies which component of the model has drifted from the base checkpoint.

A natural alternative is to compare internal representations directly. Prior representational similarity methods such as SVCCA~\citep{raghu2017svcca}, CKA~\citep{kornblith2019similarity}, and generalized shape metrics summarize two feature matrices with a single alignment score, which is descriptive but not diagnostic. Yet the prediction head of an LLM is linear, so any change in feature geometry reaches the loss through a fixed linear map. Can the loss gap between a model and its variant therefore be bounded, and attributed to its source, from feature geometry alone? Three obstacles have stood in the way: (i) no prior similarity has been tied to downstream cross-entropy risk on the deployed prediction head~\citep{ding2021grounding, davari2023reliability, klabunde2023similarity}; (ii) feature drift does bound the loss through a Lipschitz argument, but the standard constant grows with vocabulary size and is $50$ to $60$ times larger than ours on 8B models (Appendix~\ref{app:kfeat}); and (iii) a usable diagnostic must further decompose the bound into geometrically interpretable axes that map to distinct failure mechanisms.

Our key insight is that one structural property of modern LLMs, paired with one choice of comparison map, resolves all three obstacles. The linear \texttt{lm\_head} over a non-linear backbone yields a sharper Lipschitz constant, set by pairwise distances between token embeddings rather than by the spectral norm of the head, and therefore independent of vocabulary size (Sec.~\ref{subsec:unified_bound}). Comparing the two feature sets through an orthogonal map then leaves their internal distances and angles untouched, so the alignment residual splits \emph{exactly} into a scale part and a shape part read at the original feature scale. Together with a covariance-weighted head-side term, these give \textbf{PRISM} (\textbf{P}roxy \textbf{R}isk \textbf{I}nference via \textbf{S}tructural \textbf{M}apping; Fig.~\ref{fig:prism_concept}), a closed-form upper bound on the cross-entropy risk gap between a \emph{target} model (the base checkpoint) and a \emph{proxy} model (a post-training variant).

\begin{figure}[t]
\vspace*{\figpullup}
\centering
\setlength{\abovecaptionskip}{3pt}
\setlength{\belowcaptionskip}{1pt}

\colorlet{Tcol}{blue!65!black}
\colorlet{Pcol}{orange!80!black}
\colorlet{Acol}{violet!65!black}
\colorlet{panelcol}{black!38}

\begin{tikzpicture}[
    x=1cm, y=1cm,
    font=\sffamily\footnotesize,
    >={Stealth[length=1.9mm,width=1.5mm]},
    line cap=round, line join=round,
    blk/.style={draw=black!55, line width=0.7pt, rounded corners=2pt, fill=white,
                minimum height=0.55cm, minimum width=1.10cm, inner sep=1pt, align=center,
                font=\sffamily\scriptsize},
    blkT/.style={blk, draw=Tcol, fill=Tcol!5},
    blkP/.style={blk, draw=Pcol, fill=Pcol!6},
    arrT/.style={->, line width=0.85pt, Tcol},
    arrP/.style={->, line width=0.85pt, Pcol},
    termbox/.style={draw=panelcol, line width=0.6pt, rounded corners=3pt, fill=white,
                    minimum width=4.20cm, minimum height=1.75cm, align=center, inner sep=2pt}
]

\useasboundingbox (0, -0.05) rectangle (14, 4.45);

\begin{scope}[xshift=0.4cm]

\node[font=\tiny\bfseries, text=black!50] at (2.05, 4.08) {Backbone};
\node[font=\tiny\bfseries, text=black!50] at (3.70, 4.08) {Features};
\node[font=\tiny\bfseries, text=black!50] at (5.35, 4.08) {Head};
\node[font=\tiny\bfseries, text=black!50] at (7.00, 4.08) {Risk};

\node[anchor=east, text=Tcol, font=\scriptsize\bfseries] at (0.98, 3.60) {Target~$T$};
\node[anchor=east, text=Pcol, font=\scriptsize\bfseries] at (0.98, 2.50) {Proxy~$P$};

\node[align=center, font=\tiny\bfseries] (in) at (0.35, 3.05)
    {Input\\$(x,y){\sim}\mathcal{D}$};

\node[blkT] (phiT) at (2.05, 3.60) {$\phi_T$};
\node[blkT] (ZT)   at (3.70, 3.60) {$Z_T$};
\node[blkT] (HT)   at (5.35, 3.60) {$H_T$};
\node[blkT] (RT)   at (7.00, 3.60) {$\mathcal{R}_T$};
\draw[arrT] (phiT) -- (ZT);
\draw[arrT] (ZT) -- (HT);
\draw[arrT] (HT) -- (RT);

\node[blkP] (phiP) at (2.05, 2.50) {$\phi_P$};
\node[blkP] (ZP)   at (3.70, 2.50) {$Z_P$};
\node[blkP] (HP)   at (5.35, 2.50) {$H_P$};
\node[blkP] (RP)   at (7.00, 2.50) {$\mathcal{R}_P$};
\draw[arrP] (phiP) -- (ZP);
\draw[arrP] (ZP) -- (HP);
\draw[arrP] (HP) -- (RP);

\coordinate (inJ) at (1.10, 3.05);
\draw[line width=0.7pt, black!60] (in.east) -- (inJ);
\draw[arrT] (inJ) |- (phiT.west);
\draw[arrP] (inJ) |- (phiP.west);

\draw[<->, line width=0.85pt, Acol, densely dashed]
    (ZP.north) -- node[midway, anchor=west, xshift=1pt,
                       font=\tiny\bfseries, text=Acol] {$W$}
    (ZT.south);

\draw[<->, line width=0.85pt, Acol, densely dashed]
    (HT.south) -- node[midway, anchor=west, xshift=1pt,
                       font=\tiny\bfseries, text=Acol] {$W$}
    (HP.north);

\draw[decoration={brace, mirror, amplitude=4pt}, decorate,
      Acol, line width=0.85pt]
    ([xshift=8pt]RT.north east) -- ([xshift=8pt]RP.south east);

\node[anchor=north west, font=\scriptsize, inner sep=1pt] at (8.0, 4.10)
{$\begin{aligned}
{\color{Acol}\Delta\mathcal{R}}&\;=\;{\color{Acol}|\mathcal{R}_T{-}\mathcal{R}_P|} \;\le\; \delta \;+\; \gamma \\[1pt]
\delta &\;=\; K_{\mathrm{feat}}\,\sqrt{\,\smash[b]{\underbrace{(\Delta\rho)^{2}}_{\text{scale}} \;+\; \underbrace{2\rho_T\rho_P(1{-}\Omega)}_{\text{shape}}}\,} \\[13pt]
\gamma &\;=\; K_{\mathrm{pred}}\,\bigl\|\Sigma_P^{1/2}\underbrace{(WH_T{-}H_P)}_{\Delta H}\bigr\|_F
\end{aligned}$};

\end{scope}


\node[termbox, anchor=south west] (sbox) at (0.55, 0.02) {};
\node[anchor=north, font=\scriptsize\bfseries]
    at ([yshift=-0.10cm]sbox.north) {Scale mismatch \,$(\rho_T{-}\rho_P)^{2}$};
\coordinate (sC) at ($(sbox.center)+(0,-0.22)$);
\draw[Tcol, line width=0.95pt, dashed, fill=Tcol!7] (sC) circle (0.55);
\draw[Pcol, line width=0.95pt, dashed, fill=Pcol!15] (sC) circle (0.30);
\draw[->, line width=0.7pt, Tcol] (sC) -- ++(195:0.55);
\draw[->, line width=0.7pt, Pcol] (sC) -- ++(0:0.30);
\node[Tcol, font=\tiny\bfseries, anchor=east]
    at ($(sC)+(195:0.60)$) {$\rho_T$};
\node[Pcol, font=\tiny\bfseries, anchor=west]
    at ($(sC)+(0.60,0.08)$) {$\rho_P$};

\node[termbox, anchor=south west] (hbox) at (4.90, 0.02) {};
\node[anchor=north, font=\scriptsize\bfseries]
    at ([yshift=-0.10cm]hbox.north) {Shape mismatch \,$2\rho_T\rho_P(1{-}\Omega)$};
\coordinate (hC) at ($(hbox.center)+(0,-0.10)$);
\filldraw[Tcol, line width=0.85pt, fill=Tcol!10]
    ($(hC)+(-1.30,-0.22)$)
    .. controls +(0.16,0.38) and +(-0.20,0.26) ..
    ($(hC)+(-0.77,0.10)$)
    .. controls +(0.22,-0.14) and +(-0.08,0.20) ..
    ($(hC)+(-0.69,-0.38)$)
    .. controls +(-0.32,0.02) and +(0.22,-0.08) .. cycle;
\filldraw[Pcol, line width=0.85pt, fill=Pcol!12]
    ($(hC)+(0.61,-0.17)$)
    .. controls +(0.15,0.34) and +(-0.20,0.18) ..
    ($(hC)+(1.21,0.13)$)
    .. controls +(0.08,-0.20) and +(0.12,0.22) ..
    ($(hC)+(1.32,-0.37)$)
    .. controls +(-0.32,0.0) and +(0.20,-0.06) .. cycle;
\draw[<->, line width=0.8pt, Acol, densely dashed]
    ($(hC)+(-0.55,-0.10)$) -- node[midway, above=-2pt,
         font=\tiny, text=Acol] {$\Omega$}
    ($(hC)+(0.48,-0.10)$);
\node[Tcol, font=\tiny\bfseries, anchor=north]
    at ($(hC)+(-0.97,-0.36)$) {$Z_T$};
\node[Pcol, font=\tiny\bfseries, anchor=north]
    at ($(hC)+(1.00,-0.36)$) {$Z_PW$};

\node[termbox, anchor=south west] (dbox) at (9.25, 0.02) {};
\node[anchor=north, font=\scriptsize\bfseries]
    at ([yshift=-0.10cm]dbox.north) {Head discrepancy \,$\|\Sigma_P^{1/2}\Delta H\|_F$};
\coordinate (dC) at ($(dbox.center)+(-0.05,-0.52)$);
\draw[draw=black!50, line width=0.65pt, fill=black!10]
    (dC) ellipse (1.00 and 0.24);
\node[font=\tiny, text=black!55]
    at ($(dC)+(0.72,-0.02)$) {$\Sigma_P$};
\fill[black] (dC) circle (1pt);
\draw[->, line width=0.95pt, Tcol] (dC) -- ++(-0.48, 0.62);
\draw[->, line width=0.95pt, Pcol] (dC) -- ++(0.66, 0.42);
\node[Tcol, font=\tiny\bfseries, anchor=east]
    at ($(dC)+(-0.50,0.62)$) {$WH_T$};
\node[Pcol, font=\tiny\bfseries, anchor=west]
    at ($(dC)+(0.68,0.42)$) {$H_P$};
\draw[<->, line width=0.85pt, Acol, dashed]
    ($(dC)+(-0.48,0.62)$) -- node[midway, above=-2pt,
         font=\tiny\bfseries, text=Acol] {$\Delta H$}
    ($(dC)+(0.66,0.42)$);

\end{tikzpicture}
\caption{\textbf{PRISM} (\textbf{P}roxy \textbf{R}isk \textbf{I}nference via \textbf{S}tructural \textbf{M}apping) \textbf{decomposition of the risk-gap bound.} \emph{Top:} target $T$ and proxy $P$ run parallel pipelines on a shared input (backbone $\to$ features $Z$ $\to$ head $H$ $\to$ risk $\mathcal{R}$), related by an orthogonal alignment $W \in \mathcal{O}(d)$; Thm.~\ref{thm:unified_bound} bounds their risk gap by a feature alignment error $\delta$ plus a head discrepancy $\gamma$. \emph{Bottom:} the three axes the bound separates. Scale compares the feature radii $\rho_T, \rho_P$; shape compares the arrangement of the two feature clouds at fixed radius, and Prop.~\ref{prop:exact_decomposition} splits $\delta$ exactly into these two; head compares the aligned heads through their residual $\Delta H {=} W H_T {-} H_P$, measured only where the data has support. Each axis localizes a distinct empirical failure mode (Sec.~\ref{subsec:decompose}).}
\label{fig:prism_concept}
\end{figure}

PRISM exposes three diagnostic axes: \emph{scale mismatch} (activation-magnitude drift), \emph{shape mismatch} (feature-geometry distortion), and a covariance-weighted \emph{head discrepancy} (prediction-head divergence). Unlike a scalar similarity score, these axes correspond to distinct empirical failure modes. Low-bit quantization mainly distorts shape. LoRA fine-tuning drifts mostly in shape as well, but its scale axis adds a task-dependent signal that shape alone misses: on Llama, TruthfulQA fine-tuning inflates activation magnitude while BBQ shrinks it on ARC and MMLU. Quantizing the output projection inflates the head term.

Beyond post-hoc diagnosis, PRISM also provides a training signal. Under frozen-head LoRA fine-tuning the head term vanishes, leaving a differentiable shape term that is a clean regularization target for backbone drift. Penalizing feature-geometry distortion during fine-tuning cuts catastrophic forgetting further than experience replay does on the same reference data (Sec.~\ref{subsec:shape_reg_exp}).

Our contributions are threefold.
\begin{enumerate}[leftmargin=24pt,itemsep=2pt, topsep=2pt]
    \item \textbf{Theory: a closed-form CE risk bound with three diagnostic axes.} The cross-entropy risk gap between a target model and its proxy admits a closed-form upper bound that splits into a scale term and a shape term, obtained from an exact Procrustes residual decomposition, plus a covariance-weighted head-discrepancy term (Theorem~\ref{thm:unified_bound}; Fig.~\ref{fig:prism_concept}).

    \item \textbf{Framework: a unified diagnostic that doubles as a training objective.} Computed from features and head weights alone, the bound applies across both post-training quantization (PTQ) and frozen-head LoRA; its differentiable shape term further turns it into a training-time regularizer against catastrophic forgetting (Sec.~\ref{subsec:action}).

    \item \textbf{Empirical: rank-consistent, axis-localizing, and robust on little data.} On Llama and Qwen over five benchmarks, the bound ranks PTQ variants at mean Spearman $0.875$ and LoRA checkpoints at $0.831$, from a single forward pass, with no decoding or accuracy scoring (\emph{predictiveness}); it ranks as well as CKA and SVCCA on a full reference set, and with only eight reference sequences it still reproduces the full-slice ordering, where CKA, SVCCA, and the directly measured risk gap all degrade (\emph{efficiency}); its axes separate failure modes, namely shape distortion at Q2/Q3 PTQ, scale separability under cross-task LoRA drift, and head divergence at Qwen3 Q6\_K/Q8\_0 (\emph{decomposability}); and the shape regularizer outperforms experience replay in aggregate at mitigating downstream forgetting (\emph{actionability}).
\end{enumerate}

\section{Related work}

\paragraph{Representational similarity and Procrustes shape metrics.}
Representational similarity statistics such as SVCCA~\citep{raghu2017svcca}, CKA~\citep{kornblith2019similarity}, and generalized shape metrics~\citep{williams2021generalized} quantify geometric resemblance between activation matrices, but tying them to downstream behavior remains open~\citep{ding2021grounding, davari2023reliability, klabunde2023similarity}. The closest attempt relates shape distance to decodability~\citep{harvey2024what}, yet it bounds the distance between \emph{optimal} linear readouts fit to a distribution of decoding tasks, not the loss of the head a model actually ships with. PRISM bounds cross-entropy risk on the deployed head itself, and returns three separately measurable components in place of one scalar (Sec.~\ref{sec:geometric_foundation}).

\paragraph{Post-training quantization and its evaluation.}
Post-training quantization spans calibration-based reconstruction (GPTQ~\citep{frantar2022gptq}) and data-free round-to-nearest schemes (GGUF k-quants, BitsAndBytes~\citep{dettmers2022gpt3}), with low-bit quality often limited by activation outliers~\citep{xiao2023smoothquant}. The diagnostics available at quantization time are layer-wise: reconstruction error~\citep{frantar2022gptq} or weight quantization error, computed per block and blind to how that error accumulates through the remaining non-linear layers. PRISM measures the accumulated drift end-to-end on the assembled model, in one forward pass (Sec.~\ref{subsec:quant_exp}).

\paragraph{Catastrophic forgetting under parameter-efficient fine-tuning.}
LoRA fine-tuning~\citep{hu2022lora} can degrade pre-training knowledge even with a frozen \texttt{lm\_head}. Remedies constrain either the weights (EWC~\citep{kirkpatrick2017overcoming}, L2-SP~\citep{li2018explicit}) or the data (experience replay); neither constrains the feature geometry that the frozen head actually reads. PRISM adds a third option: the shape term of its bound is differentiable in the current features, so the same quantity that certifies forgetting post hoc can be minimized during training (Sec.~\ref{subsec:shape_reg}).

\paragraph{Efficient evaluation and performance forecasting.}
Scaling laws~\citep{kaplan2020scaling, hoffmann2022training} forecast aggregate loss but not variant-specific degradation, and efficient-evaluation subsets~\citep{polo2024tinybenchmarks, perlitz2024efficient} preserve variant orderings without explaining them. Both still require the benchmark to be run and scored. PRISM bounds the target-vs-proxy gap from a feature comparison alone, with no decoding and no accuracy scoring (Sec.~\ref{sec:geometric_foundation}).

\section{A geometric bound on the cross-entropy risk gap}
\label{sec:geometric_foundation}

This section derives \textbf{PRISM}, a closed-form upper bound on the cross-entropy risk gap $|\mathcal{R}_T - \mathcal{R}_P|$ that decomposes into three diagnostic axes: scale, shape, and head.

\subsection{Setup}
\label{subsec:problem_setup}

Let $T$ be a \emph{target model} (e.g., a reference checkpoint) and $P$ a \emph{proxy model} (e.g., a quantized or fine-tuned variant), sharing hidden dimension $d$. Each model $M \in \{T, P\}$ factors into a Transformer backbone $\phi_M: \mathcal{X} \to \mathbb{R}^d$ followed by a linear prediction head $H_M \in \mathbb{R}^{d \times V}$ (the \texttt{lm\_head}, vocabulary size $V$). The risk of $M$ under data distribution $\mathcal{D}$ is the cross-entropy expectation
\begin{equation}
\label{eq:ce_def}
\mathcal{R}_M \;=\; \mathbb{E}_{(x,y) \sim \mathcal{D}}\!\big[\ell\big(\phi_M(x) H_M,\, y\big)\big], \qquad \ell(v, y) = -v_y + \log \textstyle\sum_{j=1}^{V} e^{v_j}, \; v \in \mathbb{R}^V.
\end{equation}
All geometry is read off a finite \emph{diagnostic reference set} $\mathcal{D}_{\mathrm{diag}} = \{(x_i,y_i)\}_{i=1}^{N}$: per benchmark, one held-out subset drawn once and reused by every variant of a given base, on which both the geometry below and the risk gap it bounds are computed, so every variant is scored on identical inputs. We write $\widehat{\mathcal{R}}_M = \frac1N\sum_i \ell(\phi_M(x_i)H_M, y_i)$ for the empirical risk on $\mathcal{D}_{\mathrm{diag}}$, and $|\Delta\mathcal{R}| := |\widehat{\mathcal{R}}_T - \widehat{\mathcal{R}}_P|$ for the empirical risk gap. Stacking the $\phi_M$ outputs on $\mathcal{D}_{\mathrm{diag}}$ as rows, we form $Z_T, Z_P \in \mathbb{R}^{n \times d}$, where $n$ counts scored tokens, and define the RMS feature scale $\rho_M = \|Z_M\|_F/\sqrt{n}$ and the empirical (uncentered) covariance $\Sigma_P = Z_P^\top Z_P/n$. We restrict the alignment between the two backbones to $W \in \mathcal{O}(d)$ because an orthogonal map leaves distances and angles within $Z_P$ unchanged, which is what lets Proposition~\ref{prop:exact_decomposition} split the residual into scale and shape at the original feature scale; a general linear map would absorb both into the alignment. This is a design choice, not an isometry assumption: the fit of $W$ affects only how tight the feature term is, and the measured $\Omega$ reports that fit (Sec.~\ref{subsec:decompose}). Prior reports of near-isometry across related representations~\citep{park2023linear, moschella2022relative} suggest it is often good.

\subsection{The unified risk bound}
\label{subsec:unified_bound}

To bound $|\mathcal{R}_T - \mathcal{R}_P|$ we introduce a hybrid risk $\mathcal{R}_{P \to T} := \mathbb{E}[\ell(\phi_P(x)\,W \cdot H_T, y)]$ that evaluates the target's head on the proxy's aligned features. The triangle inequality $|\mathcal{R}_T - \mathcal{R}_P| \le |\mathcal{R}_T - \mathcal{R}_{P\to T}| + |\mathcal{R}_{P\to T} - \mathcal{R}_P|$ splits the gap into two components, each upper-bounded by a closed-form geometric quantity: a \emph{feature alignment error} $\delta$ and a \emph{head discrepancy} $\gamma$ (full derivations in Appendix~\ref{app:detailed_proofs}).

\paragraph{Feature error $\delta$.}
The cross-entropy loss is Lipschitz in features with constant $K_{\mathrm{feat}} = \max_{j,k}\|h_{T,j}-h_{T,k}\|_2$, where $h_{T,j}$ is the $j$-th column of $H_T$ (Appendix~\ref{app:kfeat}; tighter than the spectral bound $\sqrt{2}\|H_T\|_2$, which grows with vocabulary size). The alignment residual $\|Z_T - Z_P W\|_F^2 / n$ admits an exact identity for every orthogonal $W$:

\begin{proposition}[Exact Scale--Shape Decomposition]
\label{prop:exact_decomposition}
For any $W \in \mathcal{O}(d)$,
\begin{equation}
\label{eq:exact_equality}
\frac{1}{n} \| Z_T - Z_P W \|_F^2
\;=\;
\underbrace{(\rho_T - \rho_P)^2}_{\text{\rm Scale Mismatch } (\Delta\rho)^2}
\;+\;
\underbrace{2 \rho_T \rho_P \big( 1 - \Omega_W \big)}_{\text{\rm Shape Mismatch}},
\end{equation}
where $\Omega_W$ is the \emph{trace Procrustes similarity}:
\begin{equation}
\label{eq:omega_def}
\Omega_W(Z_T, Z_P) \;:=\; \frac{\operatorname{Tr}(Z_T^\top Z_P\, W)}{\| Z_T \|_F\, \| Z_P \|_F} \;\in [-1,\, 1].
\end{equation}
\end{proposition}

The Lipschitz bound then gives $|\mathcal{R}_T - \mathcal{R}_{P\to T}| \le K_{\mathrm{feat}}\sqrt{\tfrac{1}{n}\|Z_T - Z_P W\|_F^2} =: \delta$, expanded via Eq.~(\ref{eq:exact_equality}) in Theorem~\ref{thm:unified_bound} below.

\paragraph{Head error $\gamma$.}
Let $\Delta H := W H_T - H_P$; we define $\gamma := K_{\mathrm{pred}}\,\|\Sigma_P^{1/2}\Delta H\|_F$ with $K_{\mathrm{pred}} \le \sqrt{2}$ (Appendix~\ref{app:head_bound}), so that $|\mathcal{R}_{P\to T} - \mathcal{R}_P| \le \gamma$. The covariance weighting $\Sigma_P^{1/2}$ ensures that only head misalignment in the \emph{active subspace} (directions where the data has support) contributes to $\gamma$.

\begin{theorem}[Unified Risk Bound]
\label{thm:unified_bound}
For any $W \in \mathcal{O}(d)$,
\begin{equation}
\label{eq:unified_bound}
|\mathcal{R}_T - \mathcal{R}_P| \;\le\; \mathcal{B} \;:=\; \underbrace{K_{\mathrm{feat}} \sqrt{ (\rho_T - \rho_P)^2 + 2\rho_T \rho_P \big(1 - \Omega_W\big) }}_{\delta:\ \text{\rm feature alignment error}} \;+\; \underbrace{K_{\mathrm{pred}} \,\| \Sigma_P^{1/2} (W H_T - H_P) \|_F}_{\gamma:\ \text{\rm head discrepancy}}.
\end{equation}
We refer to $\mathcal{B} = \delta + \gamma$ as the \emph{PRISM bound}.
\end{theorem}

\paragraph{Choice of alignment.}
Any $W \in \mathcal{O}(d)$ gives a valid bound; we evaluate two. The trace form $\Omega := \Omega_{W=I}$ is the main-text default and the Procrustes-optimal nuclear form $\Omega_N := \Omega_{W=W_N}$, with bound $\mathcal{B}_N$, is in Appendix~\ref{app:general_W}. $W_N$ minimizes the alignment residual but requires a per-step SVD and leaves $\gamma$ nonzero even when $H_T \approx H_P$; $W{=}I$ sends $\gamma$ to exactly zero whenever $H_T{=}H_P$ (frozen-head LoRA and FP16-head PTQ, though not GGUF k-quant) and is directly differentiable, which enables the regularizer of Sec.~\ref{subsec:shape_reg}. Both rank variants comparably well (Sec.~\ref{subsec:ablation}). PTQ perturbs weights in place and LoRA adapts a frozen base, so the inherited coordinates, $W{=}I$, are the natural convention.

\paragraph{Empirical form.}
Every step of the proof is valid with each expectation replaced by a mean over $\mathcal{D}_{\mathrm{diag}}$, so Theorem~\ref{thm:unified_bound} holds verbatim as $|\Delta\mathcal{R}| \le \widehat{\mathcal{B}}$ on the hatted quantities, which is the form the experiments test. The population gap follows under two explicit assumptions (i.i.d.\ sampling and an almost-surely bounded per-sample gap), at the cost of an $O(N^{-1/2})$ concentration term, with $N$ counting reference sequences rather than tokens (Corollary~\ref{cor:population}).

\paragraph{Autoregressive form.}
The bound applies unchanged to autoregressive next-token prediction over a whole target sequence, scored with teacher forcing. For a sequence $(c, y)$ with context $c$ and target continuation $y = (y_1, \ldots, y_{|y|})$, the autoregressive risk is
\begin{equation}
\label{eq:ar_risk}
\mathcal{R}_M^{\mathrm{AR}} \;=\; \mathbb{E}_{(c, y) \sim \mathcal{D}} \!\left[ \frac{1}{|y|} \sum_{\tau=1}^{|y|} \ell\!\big(\phi_M(c, y_{<\tau}) \cdot H_M,\; y_\tau\big) \right],
\end{equation}
and teacher forcing extracts its $|y|$ token-level features in one forward pass, stacked into $Z_M^{\mathrm{AR}} \in \mathbb{R}^{|y| \times d}$, to which Theorem~\ref{thm:unified_bound} applies directly (full corollary in Appendix~\ref{app:ar_extension}).

\subsection{Three diagnostic axes}
\label{subsec:interpretation}

Theorem~\ref{thm:unified_bound} splits the bound along three \emph{independently measurable} axes (Fig.~\ref{fig:prism_concept}), each attached to a distinct failure mode of the proxy. \textbf{Scale} $\Delta\rho$ isolates divergence in activation magnitude, and its sign is informative: the most aggressive bit-width inflates $\rho_P$ in most cells we measure, whereas fine-tuning moves it in either direction (Sec.~\ref{subsec:decompose}). \textbf{Shape} $1-\Omega_W$ is the normalized feature mismatch left once scale is accounted for, measured under the chosen alignment: at $W{=}I$ it includes any change of orientation in the inherited coordinates, and at $W{=}W_N$ it is what survives the best orthogonal fit. \textbf{Head} $\|\Sigma_P^{1/2}(W H_T - H_P)\|_F$ measures how differently the two prediction heads read the same features.

\subsection{From diagnostic to training: shape regularization}
\label{subsec:shape_reg}

Standard LoRA leaves the prediction head frozen, so at $W{=}I$ we have $\gamma{=}0$ and Theorem~\ref{thm:unified_bound} reduces to backbone drift alone between the base $\theta_0$ and a checkpoint $\theta_t$: $|\mathcal{R}_0 - \mathcal{R}_t| \le \delta$ with $\gamma{=}0$ and $\Omega = \Omega(Z_0, Z_t)$. Since $\Omega(Z_0, Z_t)$ falls as fine-tuning proceeds and $1-\Omega$ is differentiable in $Z_t$, the shape term is not only a diagnostic but a training signal: penalizing it on a fixed reference set holds the backbone geometry near the base, which is what the bound says forgetting depends on. We therefore add a \emph{shape regularizer} to the objective,
\begin{equation}
\label{eq:shape_reg}
L_{\mathrm{total}} \;=\; \underbrace{L_{\mathrm{CE}}\bigl(\theta_t;\,\mathcal{D}_{\mathrm{FT}}\bigr)}_{\text{task loss}} \;+\; \lambda\,\bigl(1 - \Omega\!\bigl(Z_0^{\text{ref}},\, Z_t^{\text{ref}}\bigr)\bigr),
\end{equation}
where $Z_0^{\text{ref}}$ and $Z_t^{\text{ref}}$ are the base and current feature matrices on the \emph{regularizer reference set} $\mathcal{D}_{\mathrm{ref}}$: $32$ sequences drawn once from the held-out split of the fine-tuning task, disjoint from $\mathcal{D}_{\mathrm{FT}}$ and from every benchmark on which forgetting is measured, and held fixed for the whole run. $Z_0^{\text{ref}}$ is computed once from the frozen base; $Z_t^{\text{ref}}$ is recomputed at every optimizer step. The result is insensitive to which sequences are drawn (Appendix~\ref{app:reg_baselines}).

\section{Experiments}
\label{sec:experiments}

We organize the experiments around four claims. \textbf{Predictiveness}: the bound tracks $|\Delta\mathcal{R}|$ in rank order across both PTQ and LoRA variants (Sec.~\ref{subsec:predict}). \textbf{Efficiency}: one forward pass and a handful of reference sequences suffice, and the bound adds what the risk gap from the same pass cannot (Sec.~\ref{subsec:cost}). \textbf{Decomposability}: its three axes localize three distinct failure modes (Sec.~\ref{subsec:decompose}). \textbf{Actionability}: the differentiable shape term doubles as a training-time regularizer that suppresses forgetting (Sec.~\ref{subsec:action}).

\subsection{Experimental setup}
\label{subsec:exp_setup}

\paragraph{Models.}
The main analysis spans two model families: Llama-3.1-8B~\citep{grattafiori2024llama} and Qwen3-8B-Base~\citep{qwen2025qwen3} (written Qwen3-8B below); PTQ replications on Ministral-3-8B~\citep{liu2026ministral3}, DeepSeek-R1-Distill-Llama-8B~\citep{deepseekai2025r1}, and three instruction-tuned counterparts (seven 8B families total) are in Appendix~\ref{app:per_model_tables}.

\paragraph{Quantization protocols.}
\label{subsec:quantization}
The target $T$ is the base model and the proxy $P$ a quantized variant from one of three PTQ families across bit-widths 2--8: \textbf{GGUF} (round-to-nearest, \texttt{Q2\_K}--\texttt{Q8\_0}), \textbf{GPTQ}~\citep{frantar2022gptq} (4-bit, second-order reconstruction), and \textbf{BitsAndBytes}~\citep{dettmers2022gpt3} (INT8, NF4, FP4). Only GGUF quantizes \texttt{lm\_head} as well, so $\gamma$ vanishes for GPTQ and BitsAndBytes.

\paragraph{Fine-tuning tasks.}
LoRA fine-tunes on \textbf{TruthfulQA}~\citep{lin2022truthfulqa} (factual grounding) and \textbf{BBQ}~\citep{parrish2022bbq} (Bias Benchmark for QA, social-context reasoning). The target $T{=}\theta_0$ is the frozen base and the proxy $P{=}\theta_t$ a LoRA checkpoint; downstream forgetting is measured on benchmarks disjoint from both fine-tuning tasks, so that $|\Delta\mathcal{R}|$ isolates forgetting.

\paragraph{Benchmarks and scoring.}
Five benchmarks: \textbf{MMLU}~\citep{hendrycks2021measuring}, \textbf{ARC}~\citep{clark2018think} (multiple-choice knowledge), \textbf{TriviaQA}~\citep{joshi2017triviaqa}, \textbf{SQuAD}~\citep{rajpurkar2016squad} (short-horizon QA), and \textbf{GSM8K}~\citep{cobbe2021training} (multi-step reasoning). $\mathcal{D}_{\mathrm{diag}}$ holds $512$ held-out examples per benchmark for PTQ and $256$ for LoRA forgetting. All risks are computed teacher-forced: prompt $c$ and targets $y$ are scored in a single forward pass over the gold span, giving a deterministic per-sample CE loss.

\paragraph{Hyperparameters.}
LoRA (rank $32$) fine-tunes with frozen head (AdamW, lr $10^{-5}$, batch size $16$, bf16). The shape regularizer sweeps $\lambda \in \{0.01, 0.05, 0.1, 0.5, 1.0\}$ and the replay-CE baseline sweeps $\lambda \in \{0.001, 0.005, 0.01, 0.05, 0.1\}$, each range matched to its penalty's natural scale ($1{-}\Omega \sim 0.1$ vs.\ CE reference loss $\sim 1$~nat). Checkpoints every $25$ steps, analysis at step $300$; all experiments use a single NVIDIA RTX 5090 (32~GB).

\subsection{Predictiveness: the bound tracks the risk gap}
\label{subsec:predict}
\label{subsec:quant_exp}
\label{subsec:forget_exp}

\begin{figure}[t]
    \vspace*{\figpullup}
    \centering
    \includegraphics[width=0.98\linewidth]{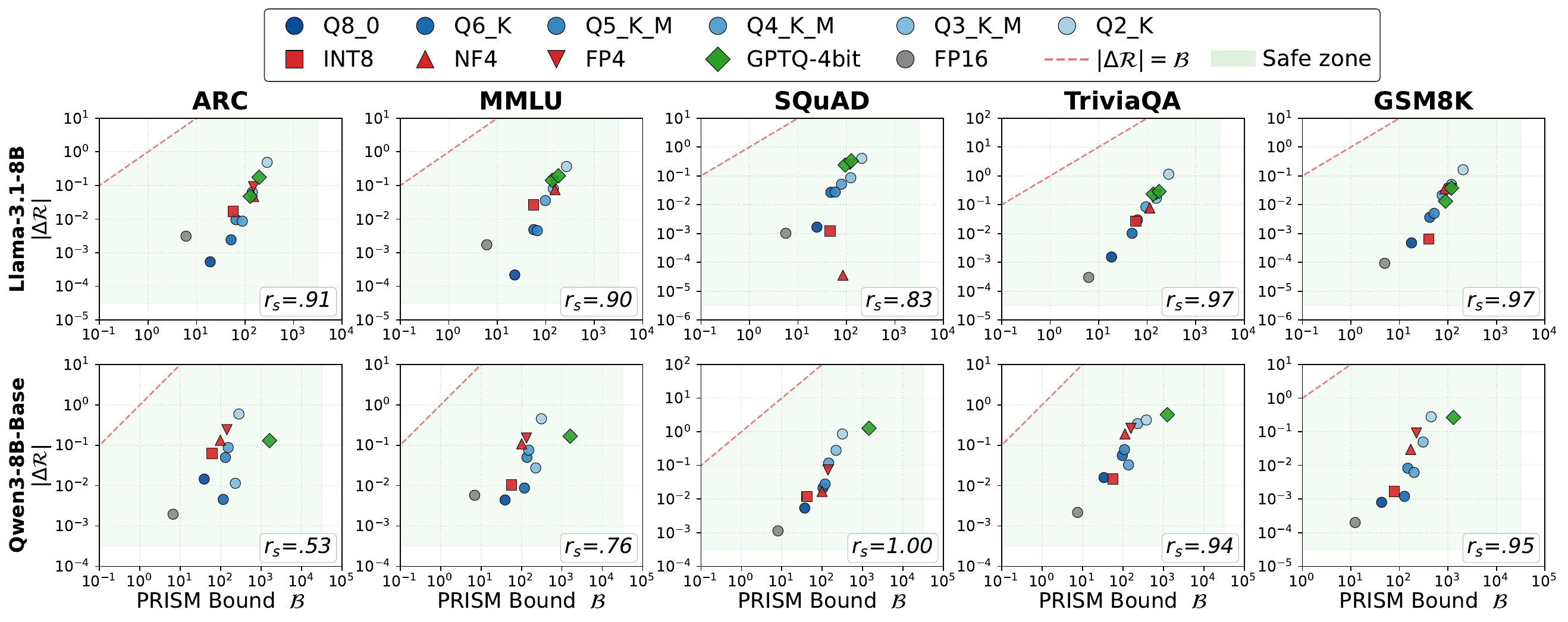}
    \caption{\textbf{The PRISM bound $\mathcal{B}$ tracks the empirical risk gap across two model families and five benchmarks.} Each subplot scatters the PRISM bound $\mathcal{B}$ (x-axis, log) against the empirical cross-entropy risk gap $|\Delta\mathcal{R}|$ (y-axis, log). Each point is one quantization variant; colors denote PTQ family (GGUF / GPTQ / BitsAndBytes). Rows: Llama-3.1-8B, Qwen3-8B. Columns: ARC, MMLU, SQuAD, TriviaQA, GSM8K. Spearman $r_s$ is annotated per panel (mean $r_s{=}0.875$ across the 2{$\times$}5 grid); the red dashed $y{=}x$ line marks bound tightness, and the \textcolor{green!50!black}{green shaded region below it} is where the measured gap respects the bound.}
    \label{fig:quant_grid_bound}
\end{figure}

\begin{table}[t]
\centering
\caption{Geometric decomposition for \textbf{Llama-3.1-8B} under identity alignment ($W{=}I$) on MMLU. The first column reports Spearman's $r_s(\mathcal{B},\,|\Delta\mathcal{R}|)$ across all quantization variants. Shading: \colorbox{red!35}{$\Omega{<}0.80$} / \colorbox{red!18}{$\Omega{<}0.95$} on $(\Omega,\delta,\mathcal{B},|\Delta\mathcal{R}|)$; \colorbox{cyan!12}{$\gamma{=}0$} when the head is preserved ($H_T{=}H_P$).}
\label{tab:llama_decomposition_main_bound}
\resizebox{\textwidth}{!}{%
\begin{tabular}{ll l rrrrrrr}
\toprule
\multicolumn{1}{c}{\textbf{Dataset}} & \multicolumn{1}{c}{\textbf{Family}} & \multicolumn{1}{c}{\textbf{Method}} & \multicolumn{1}{c}{$\boldsymbol{\rho_T}$} & \multicolumn{1}{c}{$\boldsymbol{\rho_P}$} & \multicolumn{1}{c}{$\boldsymbol{\Omega}$} & \multicolumn{1}{c}{$\boldsymbol{\delta}$} & \multicolumn{1}{c}{$\boldsymbol{\gamma}$} & \multicolumn{1}{c}{\textbf{PRISM} $\boldsymbol{\mathcal{B}}$} & \multicolumn{1}{c}{$\boldsymbol{|\Delta\mathcal{R}|}$} \\
\midrule
MMLU & -- & FP16 & 138.95 & 138.97 & 0.9999 & 6.0918 & \cellcolor{cyan!12} $0$ & 6.0918 & 0.0017 \\
{\small ($\mathbf{r_s{=}0.90}$)} & GGUF & Q8\_0 & 138.95 & 139.17 & 0.9989 & 17.0010 & 5.9609 & 22.9620 & 0.0002 \\
 &  & Q6\_K & 138.95 & 139.82 & 0.9946 & 37.9361 & 18.7745 & 56.7106 & 0.0049 \\
 &  & Q5\_K\_M & 138.95 & 139.93 & 0.9913 & 48.2559 & 18.7916 & 67.0475 & 0.0045 \\
 &  & Q4\_K\_M & 138.95 & 140.35 & 0.9765 & 79.1674 & 18.8939 & 98.0613 & 0.0356 \\
 &  & Q3\_K\_M & 138.95 & 140.32 & \cellcolor{red!18} 0.9414 & \cellcolor{red!18} 124.9549 & 18.8431 & \cellcolor{red!18} 143.7980 & \cellcolor{red!18} 0.0808 \\
 &  & Q2\_K & 138.95 & 143.86 & \cellcolor{red!35} 0.7751 & \cellcolor{red!35} 248.1663 & 19.5549 & \cellcolor{red!35} 267.7212 & \cellcolor{red!35} 0.3658 \\
 & BnB & INT8 & 138.95 & 139.09 & 0.9881 & 56.1660 & \cellcolor{cyan!12} $0$ & 56.1660 & 0.0265 \\
 &  & NF4 & 138.95 & 144.15 & \cellcolor{red!18} 0.9125 & \cellcolor{red!18} 155.3570 & \cellcolor{cyan!12} $0$ & \cellcolor{red!18} 155.3570 & \cellcolor{red!18} 0.0750 \\
 &  & FP4 & 138.95 & 138.10 & \cellcolor{red!18} 0.9197 & \cellcolor{red!18} 145.0905 & \cellcolor{cyan!12} $0$ & \cellcolor{red!18} 145.0905 & \cellcolor{red!18} 0.1306 \\
 & GPTQ & GPTQ-4bit & 138.95 & 140.36 & \cellcolor{red!18} 0.9299 & \cellcolor{red!18} 136.6923 & \cellcolor{cyan!12} $0$ & \cellcolor{red!18} 136.6923 & \cellcolor{red!18} 0.1422 \\
\bottomrule
\end{tabular}}
\end{table}

\paragraph{Quantization.}
The PRISM bound $\mathcal{B}$ tracks $|\Delta\mathcal{R}|$ with strong rank correlation across Llama and Qwen (mean Spearman $r_s = 0.875 \pm 0.0425$ SEM over the $2{\times}5$ grid; Fig.~\ref{fig:quant_grid_bound}). Bit-width drives the bound monotonically, Q8/Q6 $<$ Q5/Q4 $<$ Q3/Q2 in both families, with Q2 in the upper-right corner where shape distortion dominates (Table~\ref{tab:llama_decomposition_main_bound}). Every evaluated variant lies below the $y{=}x$ line, as Theorem~\ref{thm:unified_bound} requires. Per-benchmark tables, the feature-only ($\delta$) scatter, and the Ministral and DeepSeek replications are in Appendix~\ref{app:quant_tables}.

\begin{figure}[t]
    \centering
    \includegraphics[width=1.0\linewidth]{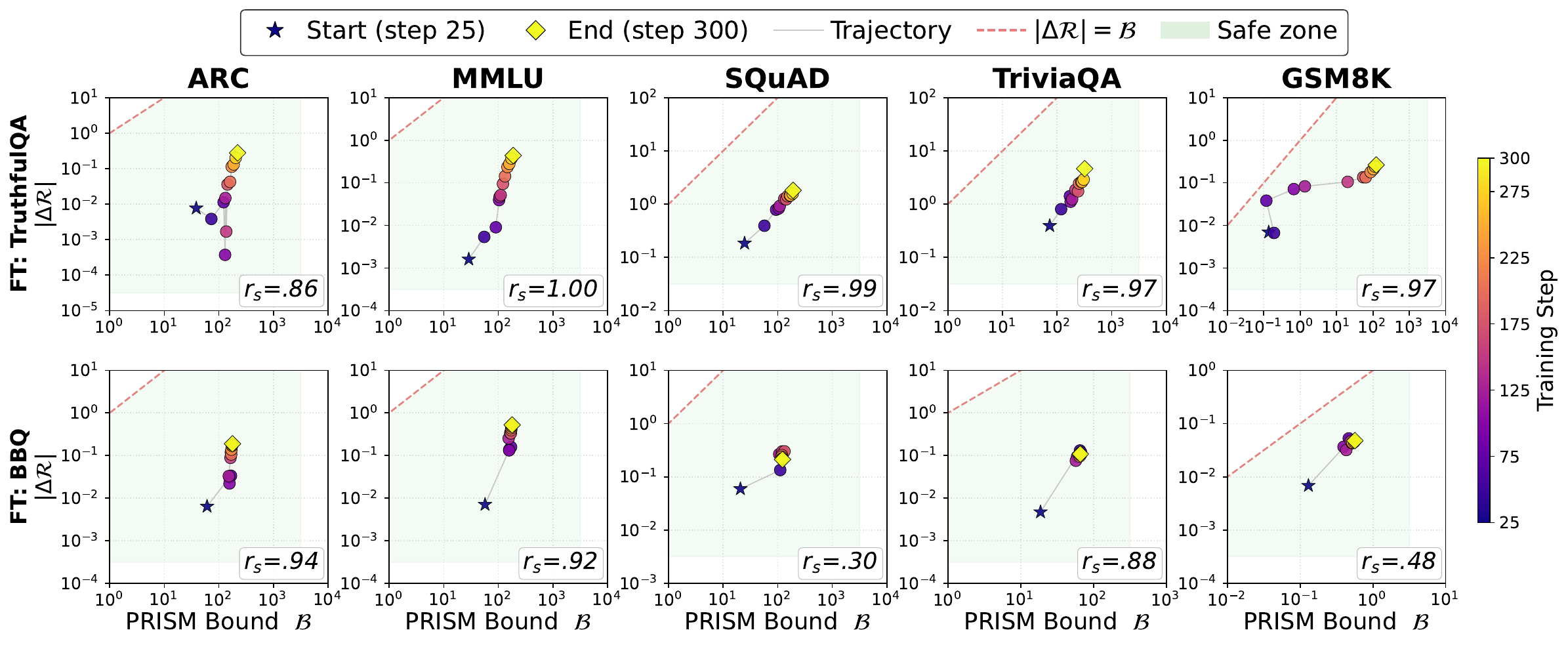}
    \caption{\textbf{Llama-3.1-8B: the PRISM bound tracks catastrophic forgetting across LoRA fine-tuning steps.} Each subplot scatters the bound $\mathcal{B}$ (x-axis, log) against the empirical forgetting $|\Delta\mathcal{R}|$ (y-axis, log) on a downstream benchmark, with one point per LoRA checkpoint colored by training step. Rows: fine-tuning task (TruthfulQA, BBQ). Columns: downstream benchmark (ARC, MMLU, SQuAD, TriviaQA, GSM8K). Under LoRA's frozen \texttt{lm\_head}, $\gamma{=}0$ at $W{=}I$, so $\mathcal{B}$ reduces to backbone scale ($\Delta\rho$) and shape ($1{-}\Omega$) drift. Spearman $r_s$ is annotated per panel.}
    \label{fig:forget_grid_llama}
\end{figure}

\paragraph{LoRA forgetting.}
The same holds for fine-tuning drift (Fig.~\ref{fig:forget_grid_llama}): across LoRA checkpoints on TruthfulQA or BBQ, the bound computed on \emph{downstream} benchmarks tracks $|\Delta\mathcal{R}|$ at mean Spearman $r_s=0.831 \pm 0.0722$ over the $2{\times}5$ cells, comparable to the PTQ grid, and at $0.958$ on the TruthfulQA row alone. The Qwen3-8B replication (Appendix~\ref{app:qwen_forgetting}) reproduces this in 8 of 10 cells; the two exceptions, BBQ fine-tuning evaluated on MMLU and TriviaQA, leave nothing to rank: TriviaQA's forgetting stays near $10^{-3}$ nats per token and MMLU's plateaus by step 75, so later checkpoints differ by noise.

\paragraph{The bound is loose, and why ranking survives it.}
$\mathcal{B}$ exceeds $|\Delta\mathcal{R}|$ by a large factor (Table~\ref{tab:llama_decomposition_main_bound}), most for the variants that drift least. Two steps of the proof account for this. The Lipschitz step charges every token the worst-case sensitivity $K_{\mathrm{feat}}$, although the realized gradient $\|H_T(\hat p - e_y)\|_2$ is near zero on confidently and correctly predicted tokens. Moving the absolute value inside the expectation then bounds the absolute per-token change, whereas signed changes partly cancel in the mean: when perturbations have a negligible mean first-order contribution, the mean loss change is second order, which helps explain why the gap shrinks faster than the residual as drift decreases. These sources of looseness neither establish nor rule out ranking fidelity; that is an empirical question. Within a cell $K_{\mathrm{feat}}$ and $\rho_T$ are fixed, so the feature term is $K_{\mathrm{feat}}$ times the RMS alignment residual $\|Z_T - Z_P W\|_F/\sqrt{n}$, and the question is whether $|\Delta\mathcal{R}|$ increases with this residual, which the feature-only scatter (Appendix~\ref{app:feature_only}) confirms cell by cell, and Fig.~\ref{fig:quant_grid_bound} with the head term included. We therefore use $\mathcal{B}$ to rank variants and diagnose their drift, rather than as a tight estimate of the risk gap.

\paragraph{PRISM ranks on par with CKA and SVCCA, and adds what they cannot.}
\label{subsec:ablation}
Against $1{-}$CKA~\citep{kornblith2019similarity} and $1{-}$SVCCA~\citep{raghu2017svcca} on the same ten cells, every PRISM quantity, down to the shape term $1{-}\Omega_N$ alone ($0.898$), lands in the same band: mean $r_s$ of $0.875$--$0.906$, with $\mathcal{B}_N$ best at $0.906$ and both baselines at $0.899$ (Table~\ref{tab:baseline_combined}, Appendix~\ref{app:similarity_baselines}). The per-metric SEM is $0.02$--$0.05$ and no pairwise difference survives multiplicity correction, so ranking quality does not separate these metrics. What separates PRISM is what comes with the score: an exact split into scale, shape and head (Sec.~\ref{subsec:decompose}), a shape term that can be trained on (Sec.~\ref{subsec:action}), and stability on reference sets too small for these statistics or for the risk gap itself (Sec.~\ref{subsec:cost}).

\subsection{Efficiency: one forward pass and eight reference sequences}
\label{subsec:cost}

\paragraph{One pass, no decoding.}
PRISM needs one teacher-forced forward pass per benchmark, with no decoding and no grading. On $256$ matched GSM8K prompts that pass takes $8.9$\,s where greedy decoding takes $556.7$\,s, $62.6\times$ more, and the gap is structural: scoring is one pass parallel over the sequence, generation is one pass per token (Appendix~\ref{app:cost}). The saving is against generative benchmarks; likelihood-scored ones such as MMLU already cost one pass.

\paragraph{The risk gap comes from the same pass, so what does the bound add?}
The pass that yields $Z_P$ also yields $\widehat{\mathcal{R}}_P$, so on a full reference set $|\Delta\mathcal{R}|$ is free, and Sec.~\ref{subsec:predict} is a check that the bound tracks what it bounds, not a cheaper substitute for measuring it. What the pass does not give without PRISM is the attribution of the gap to scale, shape and head (Sec.~\ref{subsec:decompose}), a differentiable surrogate for it (Sec.~\ref{subsec:action}), and, as we show next, an ordering that survives a small reference set.

\paragraph{Eight sequences fix the ordering; the risk gap itself needs over a hundred.}
In an independent round on Llama-3.1-8B (Appendix~\ref{app:refset_size}) we shrink $\mathcal{D}_{\mathrm{diag}}$ from $512$ to $8$ sequences per benchmark and ask how well each quantity computed on the small slice recovers the ordering that $|\Delta\mathcal{R}|$ on the full $512$ gives. The bound does not move: $\mathcal{B}_N$ scores $r_s = 0.932$ at both sizes, and the $W{=}I$ feature term $\delta$ scores $0.930$ and $0.925$. The risk gap measured on the same $8$ sequences reaches only $0.795$ and needs $128$ to match what the bound gives with $8$; $1{-}$CKA drops to $0.903$; $1{-}$SVCCA is degenerate in this regime, since canonical correlations need more samples than dimensions, and Table~\ref{tab:refset_size} reports it for completeness. Plausibly, the bound reads $d{=}4096$ feature coordinates per scored token where the gap reads one loss value. Qwen3-8B replicates the pattern ($0.827$ against $0.755$ at $n{=}8$).

\subsection{Decomposability: three axes, three failure modes}
\label{subsec:decompose}

Ranking says \emph{whether} a variant degraded; the decomposition says \emph{which} axis moved, read directly from the per-variant numbers.

\paragraph{Shape distortion (low-bit PTQ).}
Low-bit PTQ leaves $\rho$ almost untouched and moves $\Omega$ instead. At Llama-Q2\_K on MMLU (Table~\ref{tab:llama_decomposition_main_bound}), $\Delta\rho \approx 4.9$ puts the scale term near $24$, yet $\Omega$ drops to $0.78$ and the shape term $2\rho_T\rho_P(1{-}\Omega)$ reaches $\approx 9{,}000$; Ministral-Q2\_K on SQuAD shows the same ($\sim 3.5$ vs $\sim 9{,}000$), and at Q2 and Q3 shape dominates scale across all four families tested, often by orders of magnitude (Appendix~\ref{app:per_model_tables}). LoRA forgetting shows the same pattern: at Llama TruthfulQA-FT on TriviaQA (no-reg; Table~\ref{tab:reg_compare_llama_truthfulqa}), $\Omega$ drops to $0.76$ and shape ($\approx 10{,}000$) outweighs scale ($\approx 35$) by ${\sim}280\times$, which motivates the shape regularizer of Sec.~\ref{subsec:action}.

\paragraph{Scale-axis separability.}
The scale axis matters where $\Omega$ holds but $\rho$ does not, or where its sign separates variants that shape alone would not. In PTQ, Qwen3-8B Q2\_K on GSM8K moves $\rho_P$ from $273$ to $318$ ($|\Delta\rho|{=}45$ against below $7.5$ for every other Qwen3 variant; Table~\ref{tab:qwen_decomposition_all_bound}). In LoRA forgetting on Llama the sign depends on the source task: TruthfulQA-FT drives $\rho_P > \rho_T$ on all five benchmarks while BBQ-FT drives $\rho_P < \rho_T$ on ARC and MMLU; on Qwen3-8B both tasks shrink $\rho$ on nine of ten cells, so the sign is model-dependent as well. Magnitude also matters: on Qwen3-8B fine-tuned on BBQ, $\rho$ falls by $11.0$ on SQuAD while $\Omega$ holds at $0.9985$, putting the scale term at roughly half the shape term, against ${\sim}1/280$ for Llama TruthfulQA above (Tables~\ref{tab:reg_compare_llama_truthfulqa}--\ref{tab:reg_compare_qwen_bbq}).

\paragraph{Head divergence (GGUF k-quant tiers that quantize \texttt{lm\_head}).}
At Qwen3-8B Q6\_K on SQuAD (Table~\ref{tab:qwen_decomposition_all_bound}) the backbone is nearly intact, $(\Delta\rho)^2{\approx}0.18$ and $\Omega{=}0.9997$, yet the bound stays high: the feature term is only $\delta{=}26.83$ while the quantized output embedding contributes $\gamma{=}77.42$, three quarters of the bound (Q8\_0: $\delta{=}12.49$, $\gamma{=}24.48$). BnB INT8 on the same model keeps $\gamma{\equiv}0$ by construction and its whole bound is $\delta{=}42.06$, $2.5\times$ below Q6\_K's $\mathcal{B}{=}104.25$; the difference is entirely which protocol quantizes \texttt{lm\_head}.

\paragraph{From diagnosis to remediation.}
Each axis suggests a different fix. A scale mismatch points to activation outliers, which per-channel scaling~\citep{xiao2023smoothquant} or keeping outlier channels in higher precision~\citep{dettmers2022gpt3} addresses. A dominant head term says to quantize \texttt{lm\_head} more gently, or to leave it unquantized, which sets $\gamma$ to zero by construction. Shape distortion calls for more bits, or mixed precision, on the layers that introduce it at quantization time, and for the shape regularizer of Sec.~\ref{subsec:action} at fine-tuning time. Only the last is tested in this paper (Sec.~\ref{sec:conclusion}).

\subsection{Actionability: from diagnostic to training objective}
\label{subsec:action}
\label{subsec:shape_reg_exp}

\begin{figure}[t]
    \vspace*{\figpullup}
    \centering
    \includegraphics[width=1.0\linewidth]{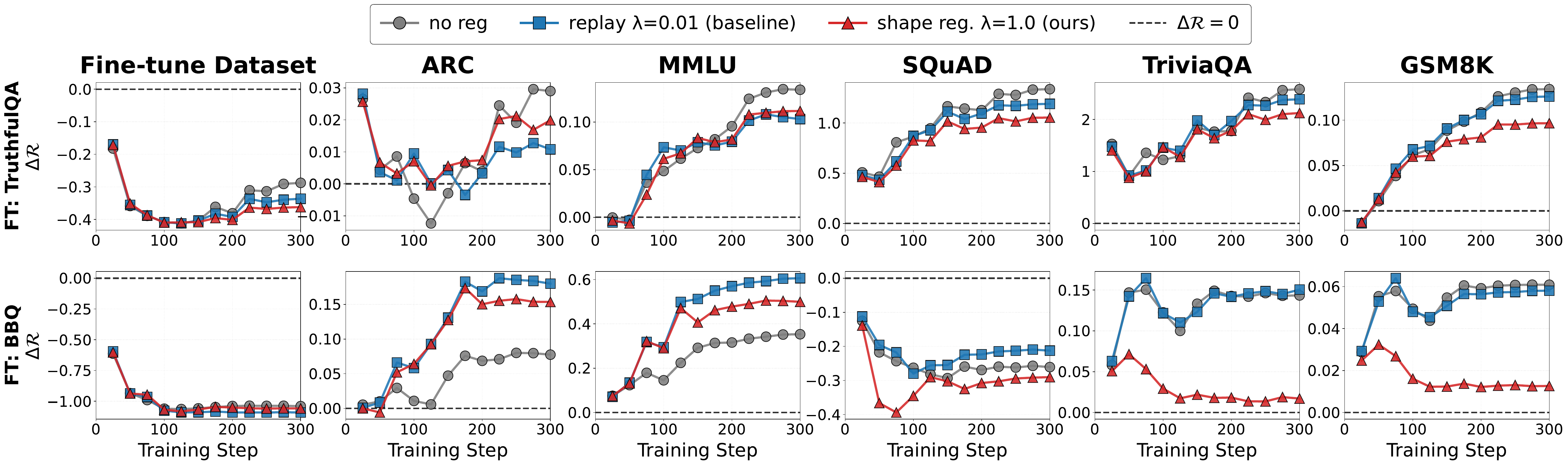}
    \caption{\textbf{Shape regularization vs.\ replay-CE on Llama-3.1-8B.} Signed CE change $\Delta\mathcal{R}$ of each LoRA checkpoint relative to the frozen base; first column: the fine-tuning task (negative $=$ learning), other columns: downstream benchmarks (positive $=$ forgetting). Three configurations: \emph{no reg} (anchor), the \emph{replay} baseline, and \emph{our shape regularizer} (Eq.~\ref{eq:shape_reg}); the latter two share a $32$-sample reference set; each is shown at the best $\lambda$ in its sweep ($1.0$ and $0.01$). Per-benchmark values for TruthfulQA: Table~\ref{tab:reg_compact_llama_truthfulqa}; BBQ analysis: Appendix~\ref{app:reg_task_dependence}; Qwen3-8B replication: Appendix~\ref{app:qwen_forgetting}.}
    \label{fig:shape_reg_combined_llama}
\end{figure}

\begin{table}[t]
\centering
\caption{Regularization comparison for \textbf{Llama-3.1-8B} fine-tuned on \textbf{TruthfulQA} (identity alignment $W{=}I$; metrics at step 300). Rows: $\Omega$ (higher = more shape preserved) and $|\Delta\mathcal{R}|$ (lower = less forgetting). \textbf{Bold} / \underline{underline}: 1st / 2nd-best per benchmark group. Columns per benchmark: no regularization, replay ($\lambda{=}0.01$), and our shape regularizer ($\lambda{=}1.0$); across-benchmark means are given in the text. In this frozen-\texttt{lm\_head} LoRA setting $\gamma \equiv 0$ and $\rho_T \approx \rho_P$, so $\delta$ and $\mathcal{B}$ track $\Omega$; the full decomposition is in Appendix Tables~\ref{tab:reg_compare_llama_truthfulqa}--\ref{tab:reg_compare_qwen_bbq}.}
\label{tab:reg_compact_llama_truthfulqa}
\setlength{\tabcolsep}{4pt}
\resizebox{\textwidth}{!}{%
\begin{tabular}{l ccccccccccccccc}
\toprule
 & \multicolumn{3}{c}{\textbf{ARC}} & \multicolumn{3}{c}{\textbf{MMLU}} & \multicolumn{3}{c}{\textbf{SQuAD}} & \multicolumn{3}{c}{\textbf{TriviaQA}} & \multicolumn{3}{c}{\textbf{GSM8K}} \\
\cmidrule(lr){2-4} \cmidrule(lr){5-7} \cmidrule(lr){8-10} \cmidrule(lr){11-13} \cmidrule(lr){14-16}
 & \textbf{no reg} & \textbf{replay} & \textbf{shape reg} & \textbf{no reg} & \textbf{replay} & \textbf{shape reg} & \textbf{no reg} & \textbf{replay} & \textbf{shape reg} & \textbf{no reg} & \textbf{replay} & \textbf{shape reg} & \textbf{no reg} & \textbf{replay} & \textbf{shape reg} \\
\midrule
$\Omega \uparrow$ & 0.916 & \underline{0.918} & \textbf{0.932} & 0.943 & \underline{0.945} & \textbf{0.954} & 0.923 & \underline{0.935} & \textbf{0.948} & 0.759 & \underline{0.781} & \textbf{0.821} & 0.991 & \underline{0.996} & \textbf{1.000} \\
\midrule
$|\Delta\mathcal{R}|\downarrow$ & 0.029 & \textbf{0.011} & \underline{0.020} & 0.134 & \textbf{0.103} & \underline{0.112} & 1.337 & \underline{1.191} & \textbf{1.054} & 2.583 & \underline{2.388} & \textbf{2.124} & 0.134 & \underline{0.126} & \textbf{0.097} \\
\bottomrule
\end{tabular}}
\end{table}

We compare the shape regularizer of Eq.~\ref{eq:shape_reg} against replay-CE on Llama-3.1-8B fine-tuned on TruthfulQA, with no regularization as the anchor (Fig.~\ref{fig:shape_reg_combined_llama}, Table~\ref{tab:reg_compact_llama_truthfulqa}). Both use the same reference set and cadence, so replay separates shape preservation from ordinary data re-fitting, and each is reported at the best $\lambda$ in its sweep ($1.0$ for the shape regularizer, $0.01$ for replay). The shape regularizer cuts mean downstream $|\Delta\mathcal{R}|$ from $0.843$ to $0.681$ ($-19\%$) where replay reaches only $0.764$ ($-9\%$), and it lifts mean $\Omega$ from $0.906$ to $0.931$ against replay's $0.915$. This is consistent with the objectives: replay re-fits reference data, whereas the shape regularizer directly targets the shape term. The effect is regime-dependent. On the low-drift settings, Llama BBQ and both Qwen3-8B tasks, where baseline $\Omega$ is already at or above $0.93$, the regularizer moves mean $|\Delta\mathcal{R}|$ by less than $10\%$ in either direction; on Llama BBQ it ends slightly above no regularization: it cuts forgetting on TriviaQA and GSM8K but adds some on ARC and MMLU, where shape drift had not turned into forgetting (Appendix~\ref{app:reg_task_dependence}).

Replay is the natural baseline because it sees the same reference data as the shape regularizer, but not the strictest test, so Appendix~\ref{app:reg_baselines} also compares EWC~\citep{kirkpatrick2017overcoming}, L2-SP~\citep{li2018explicit} and layer freezing, each at the sweep configuration whose target-task loss matches the shape-regularized run's. In this separate round the shape regularizer forgets least among all methods that learn the task to a comparable level: mean downstream $|\Delta\mathcal{R}|$ of $0.680$ against $0.751$ for EWC and $0.763$ for L2-SP. It also has the lowest target-task loss of every method in that table, so the reduction does not come at the expense of plasticity.

\section{Discussion}
\label{sec:conclusion}

\paragraph{Screening, diagnosis, intervention.}
PRISM is not a substitute for evaluation. Benchmarks measure what a variant can do, and the teacher-forced risk gap measures its change in cross-entropy relative to the base; PRISM supports decisions about what to evaluate next. From one forward pass, it helps practitioners prioritize variants for benchmarking, identify what changed in their representations and heads, and constrain shape drift during training. Its closed-form bound and exact decomposition into scale, shape, and head connect screening, diagnosis, and intervention (Secs.~\ref{subsec:predict}--\ref{subsec:action}).

\paragraph{Scope and limitations.}
Three limitations qualify these claims. First, the bound provides an upper limit on the risk gap, but our empirical validation supports its use for ranking rather than tight absolute estimation. Second, all features are extracted under teacher forcing, which needs no decoding and is the default we recommend. The theorem also covers trajectories the proxy generates itself, and on GSM8K rollouts the ranking is essentially unchanged ($r_s = 0.947$ against $0.958$ teacher-forced; Appendix~\ref{app:ar_extension}). Third, only the shape axis has an implemented mitigation, the regularizer of Sec.~\ref{subsec:action}; the scale and head remediations that Sec.~\ref{subsec:decompose} points to follow from the diagnosis but remain untested.

\paragraph{Future work.}
We describe future directions in detail in Appendix~\ref{app:future_work}, including \textbf{beyond LoRA forgetting} (analyze and regularize SFT/distillation drift), \textbf{diagnostic applications} (OOD detection, hyperparameter transfer, drift monitoring), and \textbf{beyond LLMs} (ViT, contrastive image encoders).

\subsubsection*{Acknowledgments}
We thank Chih-Han Yu for inspiring this work, and the Appier AI Research team
members for their support. This work was supported in part by the National
Science and Technology Council, Taiwan, under Grants 114-2628-E-002-021- and
115-2634-F-002-012-, and the Taiwan Centers of Excellence. Shao-Hua Sun was
supported by the Yushan Fellow Program of the Ministry of Education, Taiwan.

\renewcommand{\bibname}{References}
\bibliographystyle{abbrvnat}
\bibliography{paper}

\clearpage
\appendix
\section*{Appendix}

\begingroup
\hypersetup{colorlinks=false, linkcolor=black}
\hypersetup{pdfborder={0 0 0}}
\part{} 
\parttoc 
\endgroup



\section{Proof of the unified risk bound (Theorem~\ref{thm:unified_bound})}
\label{app:detailed_proofs}

We provide the complete proof of Theorem~\ref{thm:unified_bound}, which the main text states for general $W \in \mathcal{O}(d)$ (trace family $\Omega_W$, Eq.~\ref{eq:omega_def}). The proof proceeds in six steps: risk decomposition via the triangle inequality, Lipschitz analysis of cross-entropy in the features, the alignment-residual decomposition (whose Procrustes-optimal specialization $W = W_N$ yields the nuclear form $\Omega_N$ via SVD), the covariance-adjusted head bound, the final assembly, and the empirical form with its population corollary. The identity-alignment specialization $W = I$ (used throughout Sec.~\ref{sec:experiments} and the regularizer of Sec.~\ref{subsec:shape_reg}) and the $W = W_N$ specialization (used in the Sec.~\ref{subsec:ablation} ablation) follow as corollaries of the same proof.

\subsection{Relation to classical Procrustes shape metrics}
\label{app:shape_metrics_relation}

At $W = W_N$, the left-hand side of Eq.~(\ref{eq:exact_equality}) is the squared Procrustes size-and-shape distance of the generalized shape metrics framework~\citep{williams2021generalized, harvey2024what}. Our contribution is not that distance but its exact split into a scale term and a shape term, which that literature treats as one descriptive dissimilarity, and its lifting to a bound on cross-entropy risk (Theorem~\ref{thm:unified_bound}).

\subsection{Step 1: risk decomposition via triangle inequality}
\label{app:risk_decomp}

\paragraph{Setup.}
Recall that the risk of model $M$ is $\mathcal{R}_M = \mathbb{E}_{(x,y)}[\ell(\phi_M(x) \cdot H_M, y)]$, where $\ell$ is the cross-entropy loss (Eq.~\ref{eq:ce_def}). For any orthogonal alignment $W \in \mathcal{O}(d)$, introduce the hybrid risk:
\begin{equation}
\mathcal{R}_{P \to T} := \mathbb{E}_{(x,y)}\!\Big[\ell\!\big(\phi_P(x)\, W \cdot H_T,\; y\big)\Big].
\end{equation}
This evaluates the target's head $H_T$ on proxy features aligned by $W$, serving as a bridge between $\mathcal{R}_T$ and $\mathcal{R}_P$. The proof below holds for any $W \in \mathcal{O}(d)$; the two specializations relevant to the paper are $W = I$ (identity alignment, trace form $\Omega$, used throughout the main text) and $W = W_N := \arg\min_W \|Z_T - Z_P W\|_F^2$ (Procrustes-optimal, nuclear form $\Omega_N$, Sec.~\ref{subsec:ablation} ablation).

\paragraph{Decomposition.}
By the triangle inequality on absolute values:
\begin{equation}
\label{eq:app_triangle}
|\mathcal{R}_T - \mathcal{R}_P| \;\le\; \underbrace{\big|\mathcal{R}_T - \mathcal{R}_{P \to T}\big|}_{\text{Feature Alignment Error } \delta} \;+\; \underbrace{\big|\mathcal{R}_{P \to T} - \mathcal{R}_P\big|}_{\text{Head Discrepancy } \gamma}.
\end{equation}
$\delta$ measures how features differ through the same head $H_T$; $\gamma$ measures how heads differ on the same aligned features.

\subsection{Step 2: bounding the feature alignment error $\delta$}
\label{app:kfeat}

\paragraph{Lipschitz property of cross-entropy with respect to features.}
We bound $\delta$ by showing that the cross-entropy loss is Lipschitz continuous in the feature vector $z$. Define the end-to-end composition $g_T(z, y) := \ell(z \cdot H_T, y)$, treating the classification head $H_T$ and the cross-entropy loss as a single map from features to scalar loss. We bound $g_T$ via direct gradient analysis rather than the Lipschitz composition rule (loss-Lipschitz $\sqrt{2}$ times head operator norm $\|H_T\|_2$): the composition rule is agnostic to the softmax output and gives the spectral constant $\sqrt{2}\|H_T\|_2$ (Remark below), whereas the gradient $\nabla_z g_T = H_T(\hat{p} - e_y)$ exposes the predicted distribution $\hat{p}$ and lets us derive a constant set by pairwise token-embedding distances instead. Our goal is to find the tightest constant $K_{\mathrm{feat}}$ such that:
\begin{equation}
|g_T(z_1, y) - g_T(z_2, y)| \le K_{\mathrm{feat}} \|z_1 - z_2\|_2, \quad \forall z_1, z_2, y.
\end{equation}

\paragraph{Computing the gradient.}
The logit vector is $v = z \cdot H_T \in \mathbb{R}^V$, and the cross-entropy loss is $\ell(v, y) = -v_y + \log \sum_{j=1}^{V} e^{v_j}$. The gradient with respect to the logits is:
\begin{equation}
\frac{\partial \ell}{\partial v} = \hat{p} - e_y,
\end{equation}
where $\hat{p} = \mathrm{softmax}(v) \in \mathbb{R}^V$ and $e_y$ is the one-hot vector at index $y$. Applying the chain rule to obtain the gradient with respect to features:
\begin{equation}
\nabla_z \ell = H_T (\hat{p} - e_y) = \sum_{j=1}^{V} (\hat{p}_j - \delta_{jy})\, h_{T,j},
\end{equation}
where $h_{T,j}$ denotes the $j$-th column of $H_T$ (the embedding vector for token $j$).

\paragraph{Simplex polarization.}
We now derive a tight bound on $\|\nabla_z \ell\|_2$ by exploiting the probability simplex constraint $\sum_j \hat{p}_j = 1$. Expanding the gradient for true class $y$:
\begin{equation}
\nabla_z \ell = (\hat{p}_y - 1)\, h_{T,y} + \sum_{j \neq y} \hat{p}_j\, h_{T,j}.
\end{equation}
Since $1 - \hat{p}_y = \sum_{j \neq y} \hat{p}_j$, we can rewrite:
\begin{equation}
\nabla_z \ell = -\Big(\sum_{j \neq y} \hat{p}_j\Big) h_{T,y} + \sum_{j \neq y} \hat{p}_j\, h_{T,j} = \sum_{j \neq y} \hat{p}_j\, (h_{T,j} - h_{T,y}).
\end{equation}
This reveals that the gradient is a \emph{non-negative linear combination} of pairwise differences between incorrect-class embeddings and the correct-class embedding, with weights $\hat{p}_j$ summing to $1 - \hat{p}_y \le 1$. Taking the norm:
\begin{equation}
\|\nabla_z \ell\|_2 \le \sum_{j \neq y} \hat{p}_j \|h_{T,j} - h_{T,y}\|_2 \le \max_{j \neq y} \|h_{T,j} - h_{T,y}\|_2 \cdot \underbrace{\sum_{j \neq y} \hat{p}_j}_{\le\, 1}.
\end{equation}
The bound depends on the true class $y$ only through $\max_{j \neq y} \|h_{T,j} - h_{T,y}\|_2$. Taking the supremum over $y \in \{1,\ldots,V\}$ extends the maximum to all ordered pairs $(j,y)$ with $j\neq y$ (and the $j=y$ case contributes $0$), giving the full pairwise diameter:
\begin{equation}
\label{eq:kfeat_tight}
K_{\mathrm{feat}} = \max_{j,k} \|h_{T,j} - h_{T,k}\|_2.
\end{equation}

\paragraph{Remark.} This bound depends on the \emph{relative distances} between token embeddings, not their absolute magnitudes. A uniform shift $H_T \to H_T + c\mathbf{1}^\top$ does not change $K_{\mathrm{feat}}$. A naive Cauchy--Schwarz bound gives $K_{\mathrm{feat}}^{\mathrm{naive}} = \sqrt{2}\|H_T\|_2$, which is substantially looser. Since the decomposition of Eq.~(\ref{eq:app_triangle}) routes through $\mathcal{R}_{P\to T}$ (which uses $H_T$), $K_{\mathrm{feat}}$ is determined by $H_T$ alone; it takes a single value across all proxy variants of a fixed target model and enters as a constant scaling on the feature term, not as a proxy-dependent quantity that could perturb variant ranking.

\paragraph{From gradient bound to Lipschitz inequality.}
Since $g_T(\cdot, y)$ is continuously differentiable, the gradient bound $\|\nabla_z g_T(z, y)\|_2 \le K_{\mathrm{feat}}$ implies the Lipschitz inequality stated above by the mean value theorem: for any $z_1, z_2$,
\begin{equation*}
g_T(z_2, y) - g_T(z_1, y) = \int_0^1 \nabla_z g_T\!\big(z_1 + t(z_2 - z_1),\, y\big)^\top (z_2 - z_1)\, dt,
\end{equation*}
and Cauchy--Schwarz gives $|g_T(z_2, y) - g_T(z_1, y)| \le K_{\mathrm{feat}}\, \|z_2 - z_1\|_2$.

\paragraph{Applying the Lipschitz bound.}
Using the Lipschitz property with $z_1 = \phi_T(x)$ and $z_2 = \phi_P(x) W$:
\begin{equation}
\begin{aligned}
\delta &= \big|\mathcal{R}_T - \mathcal{R}_{P \to T}\big| = \Big|\mathbb{E}\big[g_T(\phi_T(x), y) - g_T(\phi_P(x) W, y)\big]\Big| \\
&\le \mathbb{E}\Big[|g_T(\phi_T(x), y) - g_T(\phi_P(x) W, y)|\Big] \\
&\le K_{\mathrm{feat}} \cdot \mathbb{E}_x\!\big[\|\phi_T(x) - \phi_P(x) W\|_2\big].
\end{aligned}
\end{equation}

Applying Jensen's inequality ($\mathbb{E}[\|X\|] \le \sqrt{\mathbb{E}[\|X\|^2]}$) to connect this to the alignment residual:
\begin{equation}
\label{eq:delta_lipschitz}
\delta \le K_{\mathrm{feat}} \cdot \sqrt{\mathbb{E}_x\!\big[\|\phi_T(x) - \phi_P(x) W\|_2^2\big]} = K_{\mathrm{feat}} \cdot \sqrt{\frac{1}{n}\|Z_T - Z_P W\|_F^2}.
\end{equation}

\paragraph{Empirical values.} Table~\ref{tab:lipschitz_constants} reports $K_{\mathrm{feat}}$ for all 8B models studied here, ranging from $0.93$ (Mistral family) to $3.46$ (Qwen3-8B-Base); $K_{\mathrm{pred}}=\sqrt{2}$ holds universally. These values are bounded by pairwise token-embedding distances and remain independent of $V$ (Llama $V{=}128{,}256$, Qwen3 $V{=}151{,}936$), in contrast to the naive bound $\sqrt{2}\|H_T\|_2$, whose spectral norm grows with vocabulary size: it equals $133.8$ on Llama-3.1-8B and $220.9$ on Qwen3-8B-Base, $51$ and $64$ times $K_{\mathrm{feat}}$ respectively.

\begin{table}[t]
\centering
\caption{Empirical Lipschitz constants of the PRISM bound (Theorem~\ref{thm:unified_bound}) for each evaluated model. $K_{\mathrm{pred}} = \sqrt{2}$ holds universally (Appendix~\ref{app:head_bound}, simplex polarization), independent of model. $K_{\mathrm{feat}}$ is empirical per model, depending on the spread of token embeddings in $H_T$; for the 8B-scale models studied here, $K_{\mathrm{feat}}$ ranges from $\approx 0.93$ (Mistral family) to $\approx 3.46$ (Qwen3-8B-Base). Since $K_{\mathrm{feat}}$ is a constant per model, it scales the bound's magnitude but does not affect the within-model rank correlation that PRISM is used for.}
\label{tab:lipschitz_constants}
\setlength{\tabcolsep}{8pt}
\begin{tabular}{l cc}
\toprule
\multicolumn{1}{c}{\textbf{Model}} & \multicolumn{1}{c}{$\boldsymbol{K_{\mathrm{feat}}}$} & \multicolumn{1}{c}{$\boldsymbol{K_{\mathrm{pred}}}$} \\
\midrule
Llama-3.1-8B & 2.61 & $\sqrt{2}$ \\
Ministral-3-8B & 0.98 & $\sqrt{2}$ \\
Qwen3-8B & 3.46 & $\sqrt{2}$ \\
DeepSeek-R1-8B & 2.60 & $\sqrt{2}$ \\
Llama-3.1-8B-Instruct & 2.60 & $\sqrt{2}$ \\
Ministral-3-8B-Instruct & 0.93 & $\sqrt{2}$ \\
Qwen3-8B-Instruct & 3.41 & $\sqrt{2}$ \\
\bottomrule
\end{tabular}
\end{table}

\subsection{Step 3: exact geometric decomposition}
\label{app:procrustes}

We derive the exact identity linking the alignment residual to scale and shape mismatches. The identity holds for every $W \in \mathcal{O}(d)$; we then specialize to the two cases relevant to this paper.

\paragraph{General identity.}
For any $W \in \mathcal{O}(d)$, expand the squared Frobenius norm:
\begin{equation}
\begin{aligned}
\|Z_T - Z_P W\|_F^2 &= \operatorname{Tr}\!\big[(Z_T - Z_P W)^\top (Z_T - Z_P W)\big] \\
&= \|Z_T\|_F^2 + \|Z_P W\|_F^2 - 2\operatorname{Tr}(Z_T^\top Z_P W).
\end{aligned}
\end{equation}
Since $W$ is orthogonal, $\|Z_P W\|_F^2 = \operatorname{Tr}(W^\top Z_P^\top Z_P W) = \operatorname{Tr}(Z_P^\top Z_P) = \|Z_P\|_F^2$ (cyclic trace + $W^\top W = I_d$). Therefore:
\begin{equation}
\label{eq:procrustes_expand}
\|Z_T - Z_P W\|_F^2 = \|Z_T\|_F^2 + \|Z_P\|_F^2 - 2\operatorname{Tr}(Z_T^\top Z_P W).
\end{equation}
Recall $\rho_M = \|Z_M\|_F/\sqrt{n}$ (so $\|Z_M\|_F^2/n = \rho_M^2$) and $\Omega_W := \operatorname{Tr}(Z_T^\top Z_P W)/(\|Z_T\|_F\|Z_P\|_F)$ (main-text Eq.~\ref{eq:omega_def}). Dividing Eq.~(\ref{eq:procrustes_expand}) by $n$:
\begin{equation}
\frac{1}{n}\|Z_T - Z_P W\|_F^2 = \rho_T^2 + \rho_P^2 - 2\rho_T\rho_P\, \Omega_W.
\end{equation}
Adding and subtracting $2\rho_T\rho_P$ completes the square:
\begin{equation}
\label{eq:exact_decomposition_general}
\frac{1}{n}\|Z_T - Z_P W\|_F^2 = \underbrace{(\rho_T - \rho_P)^2}_{\text{Scale Mismatch}} + \underbrace{2\rho_T\rho_P(1 - \Omega_W)}_{\text{Shape Mismatch}}.
\end{equation}
This is Proposition~\ref{prop:exact_decomposition} in its general form. \qed

\paragraph{Specialization 1: $W = I$ (identity alignment, main text).}
Setting $W = I$ gives the trace form $\Omega \equiv \Omega_{W=I} = \operatorname{Tr}(Z_T^\top Z_P)/(\|Z_T\|_F \|Z_P\|_F)$, and Eq.~(\ref{eq:exact_decomposition_general}) becomes
\begin{equation}
\frac{1}{n}\|Z_T - Z_P\|_F^2 = (\rho_T - \rho_P)^2 + 2\rho_T\rho_P(1 - \Omega),
\end{equation}
which is the form used throughout the main-text experiments and the regularizer of Sec.~\ref{subsec:shape_reg}.

\paragraph{Specialization 2: $W = W_N$ (Procrustes-optimal, ablation).}
The alignment that minimizes the residual is the classical orthogonal Procrustes problem~\citep{schonemann1966generalized, gower2004procrustes}:
\begin{equation}
W_N := \arg\min_{W \in \mathcal{O}(d)} \|Z_T - Z_P W\|_F^2 = \arg\max_{W \in \mathcal{O}(d)} \operatorname{Tr}(Z_T^\top Z_P W),
\end{equation}
where the second equality follows from Eq.~(\ref{eq:procrustes_expand}) since the first two terms are $W$-independent. Its closed-form SVD solution is well known; we restate it here for completeness. Let $Z_T^\top Z_P = U\Sigma V^\top$ be the SVD with $U, V \in \mathcal{O}(d)$ and $\Sigma = \mathrm{diag}(\sigma_1, \ldots, \sigma_d)$, $\sigma_i \ge 0$. By cyclic trace:
\begin{equation}
\operatorname{Tr}(Z_T^\top Z_P W) = \operatorname{Tr}(\Sigma\, V^\top W U) = \operatorname{Tr}(\Sigma R), \quad R := V^\top W U \in \mathcal{O}(d).
\end{equation}
Since $\sigma_i \ge 0$ and $|R_{ii}| \le 1$ for any orthogonal $R$, $\operatorname{Tr}(\Sigma R) = \sum_i \sigma_i R_{ii}$ is maximized at $R = I_d$, giving:
\begin{equation}
\label{eq:W_optimal}
W_N = V U^\top, \qquad \max_{W} \operatorname{Tr}(Z_T^\top Z_P W) = \sum_i \sigma_i = \|Z_T^\top Z_P\|_*.
\end{equation}
Therefore $\Omega_{W_N} = \Omega_N := \|Z_T^\top Z_P\|_*/(\|Z_T\|_F\|Z_P\|_F)$ (nuclear form, Eq.~\ref{eq:omega_nuclear}), and Eq.~(\ref{eq:exact_decomposition_general}) at $W = W_N$ becomes
\begin{equation}
\min_{W \in \mathcal{O}(d)} \frac{1}{n}\|Z_T - Z_P W\|_F^2 = (\rho_T - \rho_P)^2 + 2\rho_T\rho_P(1 - \Omega_N),
\end{equation}
the form underlying the Sec.~\ref{subsec:ablation} ablation.
\begin{remark}[$W_N$ is the feature-side optimum, not the joint optimum]
The Procrustes alignment $W_N$ minimizes only the feature alignment residual $\delta(W)$, not the full bound $\delta(W) + \gamma(W)$, since the head term also depends on $W$ via $W H_T - H_P$. Appendix~\ref{app:joint_opt} discusses the joint optimization.
\end{remark}

\paragraph{Assembling the feature bound.}
Combining Eq.~(\ref{eq:delta_lipschitz}) with the general decomposition Eq.~(\ref{eq:exact_decomposition_general}):
\begin{equation}
\label{eq:delta_final}
\delta \le K_{\mathrm{feat}} \sqrt{(\rho_T - \rho_P)^2 + 2\rho_T\rho_P(1 - \Omega_W)}.
\end{equation}
At $W = I$, replace $\Omega_W$ with $\Omega$; at $W = W_N$, with $\Omega_N$.

\subsection{Step 4: bounding the head discrepancy $\gamma$}
\label{app:head_bound}

\paragraph{Lipschitz property with respect to logits.}
The head discrepancy measures how much predictions differ when two different ``effective heads'' ($W H_T$ vs.\ $H_P$) operate on the same proxy features. Let $z = \phi_P(x)$ denote the (unaligned) proxy feature. Then:
\begin{equation}
\gamma = \big|\mathcal{R}_{P \to T} - \mathcal{R}_P\big| = \Big|\mathbb{E}\!\big[\ell(z \cdot W H_T, y) - \ell(z \cdot H_P, y)\big]\Big|.
\end{equation}
To bound this, we use the Lipschitz property of cross-entropy with respect to \emph{logits}. The gradient of $\ell$ with respect to $v$ is $\nabla_v \ell = \hat{p} - e_y$, with norm:
\begin{equation}
\|\nabla_v \ell\|_2 = \|\hat{p} - e_y\|_2.
\end{equation}
Since $\hat{p}$ lies on the probability simplex, $\|\hat{p}\|_2 \le \|\hat{p}\|_1 = 1$, hence
\[
\|\hat{p} - e_y\|_2^2 \;=\; \|\hat{p}\|_2^2 - 2\hat{p}_y + 1 \;\le\; 1 - 2\hat{p}_y + 1 \;=\; 2(1 - \hat{p}_y) \;\le\; 2.
\]
The supremum $\sqrt{2}$ is approached as the model's predicted distribution concentrates on an incorrect class ($\hat{p}_y \to 0$). Thus $K_{\mathrm{pred}} \le \sqrt{2}$, yielding the Lipschitz inequality $|\ell(v_1, y) - \ell(v_2, y)| \le K_{\mathrm{pred}}\, \|v_1 - v_2\|_2$ on logits.

\paragraph{Deriving the covariance projection.}
Define the head misalignment matrix $\Delta H := W H_T - H_P \in \mathbb{R}^{d \times V}$. The logit difference for a single sample is $z \cdot \Delta H = \phi_P(x)(W H_T - H_P)$, and:
\begin{equation}
\begin{aligned}
\gamma &\le \mathbb{E}\!\big[|\ell(z \cdot W H_T, y) - \ell(z \cdot H_P, y)|\big] \\
&\le K_{\mathrm{pred}} \cdot \mathbb{E}_z\!\big[\|z \cdot \Delta H\|_2\big].
\end{aligned}
\end{equation}
Applying Jensen's inequality on the concave square root:
\begin{equation}
\big(\mathbb{E}[\|z \cdot \Delta H\|_2]\big)^2 \le \mathbb{E}\!\big[\|z \cdot \Delta H\|_2^2\big].
\end{equation}
Expanding the squared norm using the trace operator:
\begin{equation}
\begin{aligned}
\mathbb{E}\!\big[\|z \cdot \Delta H\|_2^2\big] &= \mathbb{E}\!\Big[\operatorname{Tr}\!\big(\Delta H^\top z^\top z\, \Delta H\big)\Big] \\
&= \operatorname{Tr}\!\Big(\Delta H^\top \underbrace{\mathbb{E}[z^\top z]}_{\Sigma_P} \Delta H\Big) \\
&= \|\Sigma_P^{1/2} \Delta H\|_F^2.
\end{aligned}
\end{equation}
Taking the square root:
\begin{equation}
\label{eq:gamma_final}
\gamma \le K_{\mathrm{pred}} \cdot \|\Sigma_P^{1/2}(W H_T - H_P)\|_F.
\end{equation}

\paragraph{Interpretation.} The covariance weighting $\Sigma_P^{1/2}$ projects the head error onto the \emph{active subspace} of the data. If the heads disagree only in directions orthogonal to the feature distribution (the null space of $\Sigma_P$), this disagreement has zero cost, since it does not affect any prediction. This is sharper than the loose spectral bound $\gamma \le K_{\mathrm{pred}} \|W H_T - H_P\|_2 \cdot \mathbb{E}[\|z\|_2]$, which conservatively assumes worst-case alignment between features and head error. The covariance weighting $\Sigma_P^{1/2}$ instead aligns the bound with the data-active directions: head misalignment in directions the features rarely traverse contributes little, and disagreement on $\mathrm{null}(\Sigma_P)$ contributes nothing.

\subsection{Step 5: assembling the unified bound}
\label{app:assembly}

For any $W \in \mathcal{O}(d)$, combining the risk decomposition (Eq.~\ref{eq:app_triangle}), the feature bound (Eq.~\ref{eq:delta_final}), and the head bound (Eq.~\ref{eq:gamma_final}):
\begin{equation}
|\mathcal{R}_T - \mathcal{R}_P| \;\le\; \delta + \gamma \;\le\; K_{\mathrm{feat}} \sqrt{(\rho_T - \rho_P)^2 + 2\rho_T\rho_P(1 - \Omega_W)} \;+\; K_{\mathrm{pred}} \cdot \|\Sigma_P^{1/2}(W H_T - H_P)\|_F.
\end{equation}
This is Theorem~\ref{thm:unified_bound} in its general form. The two specializations used in the paper:
\begin{itemize}
    \item \textbf{$W = I$ (identity alignment, main text):} $\Omega_W = \Omega$ (trace form) and $W H_T - H_P = H_T - H_P$, giving the bound used throughout Sec.~\ref{sec:experiments} and the regularizer of Sec.~\ref{subsec:shape_reg}.
    \item \textbf{$W = W_N$ (Procrustes-optimal, ablation):} $\Omega_W = \Omega_N$ (nuclear form, achieves the tightest feature-side specialization), giving the bound underlying the Sec.~\ref{subsec:ablation} ablation.
\end{itemize}
This completes the proof of Theorem~\ref{thm:unified_bound}. \qed

\subsection{Step 6: empirical form and the population corollary}
\label{app:empirical_population}

The proof above is written with expectations under $\mathcal{D}$, but nothing in it requires them: the triangle inequality through $\mathcal{R}_{P\to T}$ and the simplex-polarization Lipschitz step (Appendix~\ref{app:kfeat}) are pointwise, Jensen applies to any average, and Prop.~\ref{prop:exact_decomposition} is already an exact finite-sample identity. Replacing each expectation by the mean over a reference sample $S=\{(x_i,y_i)\}_{i=1}^{N}$ therefore gives, with $\widehat{\mathcal{R}}_M = \frac1N\sum_i \ell(\phi_M(x_i)H_M, y_i)$ and $\widehat\rho_M, \widehat\Omega_W, \widehat\Sigma_P$ computed from the $Z_T, Z_P$ extracted on $S$,
\begin{equation}
\label{eq:empirical_bound}
\big|\widehat{\mathcal{R}}_T - \widehat{\mathcal{R}}_P\big| \;\le\; \widehat{\mathcal{B}} \;:=\; K_{\mathrm{feat}}\sqrt{(\widehat\rho_T-\widehat\rho_P)^2 + 2\widehat\rho_T\widehat\rho_P(1-\widehat\Omega_W)} \;+\; K_{\mathrm{pred}}\,\big\|\widehat\Sigma_P^{1/2}(W H_T - H_P)\big\|_F .
\end{equation}
This is the form the experiments test: every quantity reported in this paper is hatted and measured on $\mathcal{D}_{\mathrm{diag}}$, and the $360/360$ satisfied rows of Sec.~\ref{subsec:predict} verify Eq.~(\ref{eq:empirical_bound}), not a population statement. Reaching the population gap is a separate and standard step, whose two extra assumptions we state rather than absorb.

\begin{corollary}[Population risk gap]
\label{cor:population}
Assume $S$ is drawn i.i.d.\ from $\mathcal{D}$ and that the per-sample loss gap $D_i := \ell(\phi_T(x_i)H_T, y_i) - \ell(\phi_P(x_i)H_P, y_i)$ satisfies $|D_i| \le b$ almost surely. Since $\mathcal{R}_T - \mathcal{R}_P = \mathbb{E}[D_1]$ and $\widehat{\mathcal{R}}_T - \widehat{\mathcal{R}}_P = \frac1N\sum_i D_i$, McDiarmid's inequality for that mean (bounded differences $2b/N$), combined with Eq.~(\ref{eq:empirical_bound}), gives for any $\delta_{\mathrm{c}} \in (0,1)$, with probability at least $1-\delta_{\mathrm{c}}$ over the draw of $S$,
\begin{equation}
\label{eq:population_bound}
|\mathcal{R}_T - \mathcal{R}_P| \;\le\; \widehat{\mathcal{B}} \;+\; b\sqrt{\frac{2\ln(2/\delta_{\mathrm{c}})}{N}} .
\end{equation}
Splitting $D_i$ at the hybrid loss as in Step 1 bounds the constant explicitly: if $\sup_x\|\phi_T(x)-\phi_P(x)W\|_2 \le D_Z$ and $\sup_x\|\phi_P(x)\|_2 \le C_z$, then $b \le K_{\mathrm{feat}} D_Z + \sqrt{2}\,C_z\,\|W H_T - H_P\|_{\mathrm{op}}$.
\end{corollary}

The deviation term is a generic $N^{-1/2}$ concentration penalty, the price of any finite-sample estimate of a population gap rather than a property of PRISM, and for sequence data the per-sample loss is the token-average within a sequence, so $N$ counts reference sequences rather than tokens. We state the empirical form as the theorem because it is what ranking on a fixed reference set needs, and because $b$, unlike $\widehat{\mathcal{B}}$, is a worst-case constant that is not directly observable.

\section{Choosing the orthogonal alignment}
\label{app:general_W}

Theorem~\ref{thm:unified_bound} is stated for every $W \in \mathcal{O}(d)$, so the question this appendix answers is not whether the bound generalizes but which member of the family to report. We write out the bound as a function of $W$, characterize the two specializations the paper uses, the $W = I$ trace form and the Procrustes-optimal $W = W_N$ nuclear form (the latter in the Sec.~\ref{subsec:ablation} ablation), and treat the joint optimization over the feature and head terms. The two specializations encode different priorities: $W_N$ minimizes the feature alignment residual but generally inflates the head term, while $W = I$ makes the head term vanish exactly under a frozen \texttt{lm\_head} and is SVD-free differentiable, properties that the LLM settings studied here (PTQ with FP16 head and frozen-head LoRA) require.

\subsection{The bound as a function of $W$}
\label{app:general_bound}

The exact decomposition derived in Appendix~\ref{app:procrustes} (Eq.~\ref{eq:exact_decomposition_general}) holds for every $W \in \mathcal{O}(d)$ via the alignment-dependent similarity $\Omega_W$ of Eq.~\ref{eq:omega_def}, yielding a family of risk bounds parameterized by $W$:
\begin{equation}
|\mathcal{R}_T - \mathcal{R}_P| \;\le\; \underbrace{K_{\mathrm{feat}} \sqrt{(\rho_T - \rho_P)^2 + 2\rho_T\rho_P(1 - \Omega_W)}}_{\delta(W)} + \underbrace{K_{\mathrm{pred}} \|\Sigma_P^{1/2}(W H_T - H_P)\|_F}_{\gamma(W)}.
\end{equation}
The Procrustes solution $W_N$ minimizes $\delta(W)$ alone, yielding the nuclear form
\begin{equation}
\label{eq:omega_nuclear}
\Omega_{W_N} = \Omega_N := \frac{\|Z_T^\top Z_P\|_*}{\|Z_T\|_F \|Z_P\|_F} = \max_{W \in \mathcal{O}(d)} \frac{\operatorname{Tr}(Z_T^\top Z_P\, W)}{\|Z_T\|_F \|Z_P\|_F},
\end{equation}
the maximum of the trace family over orthogonal alignments (Appendix~\ref{app:procrustes}). However, $W_N$ enters the head term $\gamma(W_N)$, which can be inflated when the rotation is applied to a nearly preserved head ($H_T \approx H_P$). This creates a fundamental trade-off: $W_N$ gives the tightest feature bound but may worsen the head bound, while $W = I$ preserves head coherence at the cost of a slightly looser feature bound.

\subsection{The trace specialization ($W = I$)}
\label{app:trace_specialization}

Setting $W = I$ recovers the main-text trace similarity $\Omega = \Omega_{W=I}$ (Eq.~\ref{eq:omega_def}). Two structural properties characterize this specialization.

\paragraph{Ordering of similarities.}
Since the Procrustes solution maximizes $\operatorname{Tr}(Z_T^\top Z_P W)$ over $\mathcal{O}(d)$:
\begin{equation}
\operatorname{Tr}(Z_T^\top Z_P) \;\le\; \|Z_T^\top Z_P\|_*,
\end{equation}
which gives $\Omega \le \Omega_N$. Equality holds if and only if $W_N = I$, i.e., $Z_T^\top Z_P$ is symmetric positive semidefinite.

\paragraph{When does $\Omega = \Omega_N$?}
Let $Z_T^\top Z_P = U\Sigma V^\top$ be the SVD. Then $\Omega = \Omega_N$ iff $VU^\top = I$, i.e., $U = V$, which holds when $Z_T^\top Z_P$ is symmetric positive semidefinite (SPSD). In post-training quantization with mild noise, $Z_T^\top Z_P \approx Z_T^\top Z_T$ is approximately SPSD, so $\Omega \approx \Omega_N$. As quantization becomes more aggressive, asymmetric perturbations cause $\Omega$ to fall below $\Omega_N$, and this gap reflects genuine asymmetric distortion that a rotation-invariant metric would mask.

\paragraph{Relation to CKA.}
The Frobenius-form similarity $\Omega_F := \|Z_T^\top Z_P\|_F/(\|Z_T\|_F\|Z_P\|_F)$ satisfies $\Omega_F \le \Omega_N$, since $\|\sigma\|_2 \le \|\sigma\|_1$ on the singular values of $Z_T^\top Z_P$, but it does not arise from the alignment residual at any $W$. Linear CKA~\citep{kornblith2019similarity}, $\|Z_T^\top Z_P\|_F^2/(\|Z_T^\top Z_T\|_F\,\|Z_P^\top Z_P\|_F)$, differs from $\Omega_F^2$ only in its denominator and satisfies $\mathrm{CKA} \ge \Omega_F^2$ because $\|A\|_F \le \operatorname{Tr}(A)$ for positive semidefinite $A$. Neither quantity is derived from an alignment, so neither can be substituted into the residual of Proposition~\ref{prop:exact_decomposition}; $\Omega_N$ is the member of that family that exactly solves the orthogonal Procrustes problem.

\paragraph{Head-side simplification.}
At $W = I$, the head term $K_{\mathrm{pred}}\,\|\Sigma_P^{1/2}(H_T - H_P)\|_F$ carries no rotational contamination and vanishes whenever the prediction head is preserved (FP16 head under PTQ, frozen head under LoRA), recovering the regimes used throughout Sec.~\ref{sec:experiments} and the regularizer of Sec.~\ref{subsec:shape_reg}.

\subsection{Joint optimization over $W$}
\label{app:joint_opt}

Minimizing $\delta(W) + \gamma(W)$ jointly over $\mathcal{O}(d)$ has no closed form, since the head term is a weighted orthogonal Procrustes objective~\citep{gower2004procrustes}, and a numerical solution would fold scale, shape and head into the alignment itself, defeating the purpose of reading them separately. The question does not arise in the regimes studied here: with $H_T = H_P$ (GPTQ, BitsAndBytes, frozen-head LoRA) the head term vanishes exactly at $W = I$ and is generically positive elsewhere, so $W = I$ is the natural choice for analysis and for the regularizer of Sec.~\ref{subsec:shape_reg}, while $W_N$ tightens the feature term at the price of a nonzero head term; jointly rotated features and divergent heads (e.g., full-parameter SFT) are left to future work.

\section{Formal extension to autoregressive generation}
\label{app:ar_extension}

The bound of Theorem~\ref{thm:unified_bound} applies to any matrix pair $(Z_T, Z_P) \in \mathbb{R}^{n \times d}$ of feature vectors, whatever the origin of the rows. We spell out here the autoregressive instantiation informally summarized in Sec.~\ref{subsec:unified_bound}.

\paragraph{From tokens to matrices.}
Fix a single sequence $(c, y_1, \ldots, y_{|y|})$. Each position $\tau$ supplies a (context, target-token) pair instantiating Eq.~(\ref{eq:ce_def}) with feature $\phi_M(c, y_{<\tau}) \in \mathbb{R}^d$ and loss $\ell(\phi_M(c, y_{<\tau}) H_M,\, y_\tau)$. Stacking the $|y|$ feature vectors into a $|y| \times d$ block recovers exactly the per-row setting that Theorem~\ref{thm:unified_bound} controls.

\begin{corollary}[Autoregressive Risk Bound]
\label{cor:ar_bound}
For a single sequence $(c, y_1, \ldots, y_{|y|})$, let $Z_T^{\mathrm{AR}}, Z_P^{\mathrm{AR}} \in \mathbb{R}^{|y| \times d}$ be the feature matrices stacked as above, with RMS scales $\rho_T^{\mathrm{AR}}, \rho_P^{\mathrm{AR}}$, Procrustes similarity $\Omega^{\mathrm{AR}}$, and covariance $\Sigma_P^{\mathrm{AR}}$. Then
\begin{equation}
\label{eq:ar_bound}
\begin{aligned}
|\mathcal{R}_T^{\mathrm{AR}} - \mathcal{R}_P^{\mathrm{AR}}| \;\le\;& K_{\mathrm{feat}} \sqrt{(\rho_T^{\mathrm{AR}} - \rho_P^{\mathrm{AR}})^2 + 2\,\rho_T^{\mathrm{AR}}\,\rho_P^{\mathrm{AR}}(1 - \Omega^{\mathrm{AR}})} \\
&+ K_{\mathrm{pred}} \,\big\|(\Sigma_P^{\mathrm{AR}})^{1/2}(H_T - H_P)\big\|_F.
\end{aligned}
\end{equation}
\end{corollary}

\begin{proof}
Each row of $Z_M^{\mathrm{AR}}$ is a feature vector of dimension $d$, and the Lipschitz analysis of Appendix~\ref{app:kfeat} together with the Procrustes decomposition of Appendix~\ref{app:procrustes} depend only on the matrix shape, not on the provenance of its rows. Applying Theorem~\ref{thm:unified_bound} to $(Z_T^{\mathrm{AR}}, Z_P^{\mathrm{AR}})$ yields the stated inequality.
\end{proof}

This unification enables one bound to cover point-wise classification (MMLU, ARC, where $|y|{=}1$), short-horizon QA (SQuAD, TriviaQA, where $y$ is a 1--5 token span), and multi-step reasoning (GSM8K, where $y$ is a chain-of-thought solution).

\paragraph{Proxy-generated trajectories.}
Nothing in Corollary~\ref{cor:ar_bound} requires $y$ to be a reference answer: any sequence that both models score is a valid $y$, including one the proxy generated itself, in which case its own early errors shape its later context. We test this on GSM8K, the only one of our benchmarks whose answers are long enough for trajectories to diverge. For each of the $12$ Llama-3.1-8B PTQ variants, the proxy greedily generates a continuation for $100$ prompts (at most $128$ new tokens; $72$--$81$ on average), and both models are then scored on that trajectory with the generated tokens as targets. Because the proxy chose these tokens greedily, its own loss on them is small, and the gap mostly measures how far the target's predictions depart from the proxy's along its trajectory: a divergence between the two models' generative distributions rather than a task loss, valid as a bound for each variant on its own trajectory. Over five disjoint prompt subsets, the bound computed on these rollouts ranks the rollout gaps at $r_s = 0.947 \pm 0.015$, against $0.958 \pm 0.011$ for the teacher-forced bound and gap on the same prompts; the teacher-forced bound alone already orders the rollout gaps at $0.958 \pm 0.011$, and the two bounds agree on the variant ordering at $0.959 \pm 0.010$. Free-running use is therefore covered, but it requires decoding, which is what teacher forcing avoids, so teacher forcing remains our recommended default.


\section{Model and quantization details}
\label{app:model_details}

\subsection{Target models}

Table~\ref{tab:target_models} lists all target (full-precision) models evaluated in the quantization experiments.
All targets are loaded in BF16 as the reference precision; an FP16 proxy is additionally included for each model to measure the BF16$\to$FP16 precision-drift baseline.

\begin{table}[h]
\centering
\caption{Target models used in quantization experiments. All models are loaded in BF16.}
\label{tab:target_models}
\small
\resizebox{\linewidth}{!}{
\begin{tabular}{@{}clllrrr@{}}
\toprule
\multicolumn{1}{c}{\textbf{\#}} & \multicolumn{1}{c}{\textbf{Family}} & \multicolumn{1}{c}{\textbf{HuggingFace ID}} & \multicolumn{1}{c}{\textbf{Variant}} & \multicolumn{1}{c}{\textbf{Params}} & \multicolumn{1}{c}{$\boldsymbol{d}$} & \multicolumn{1}{c}{$\boldsymbol{|V|}$} \\
\midrule
1  & Qwen 3     & \texttt{Qwen/Qwen3-8B-Base}                         & Base      & 8B   & 4096 & 151K \\
2  & Llama 3.1  & \texttt{meta-llama/Meta-Llama-3.1-8B}               & Base      & 8B   & 4096 & 128K \\
3  & Ministral 3 & \texttt{mistralai/Ministral-3-8B-Base-2512}         & Base      & 8B   & 4096 & 131K \\
4  & Qwen 3     & \texttt{Qwen/Qwen3-8B}                              & Instruct  & 8B   & 4096 & 151K \\
5  & Llama 3.1  & \texttt{meta-llama/Meta-Llama-3.1-8B-Instruct}      & Instruct  & 8B   & 4096 & 128K \\
6  & Ministral 3 & \texttt{mistralai/Ministral-3-8B-Instruct-2512}    & Instruct  & 8B   & 4096 & 131K \\
7  & DeepSeek R1 & \texttt{deepseek-ai/DeepSeek-R1-Distill-Llama-8B} & Distilled & 8B   & 4096 & 128K \\
\bottomrule
\end{tabular}}
\end{table}

\subsection{Quantization backends and bit-widths}
\label{app:quant_tiers}

Every target model is evaluated against up to three quantization backends.
Table~\ref{tab:quant_tiers} summarises the shared GGUF and BitsAndBytes (BnB) configurations applied uniformly across all models.

\begin{table}[h]
\centering
\caption{Shared quantization tiers applied to all target models. GGUF quants use the \texttt{llama.cpp} k-quant scheme; BnB quants use on-the-fly quantization from the full-precision checkpoint.}
\label{tab:quant_tiers}
\small
\begin{tabular}{@{}llll@{}}
\toprule
\multicolumn{1}{c}{\textbf{Tier}} & \multicolumn{1}{c}{\textbf{Backend}} & \multicolumn{1}{c}{\textbf{Tag}} & \multicolumn{1}{c}{\textbf{Description}} \\
\midrule
FP16 reference & dtype   & \texttt{dtype:float16} & BF16$\to$FP16 precision drift baseline \\
\midrule
8-bit          & GGUF    & \texttt{Q8\_0}         & Round-to-nearest, 8-bit \\
               & BnB     & \texttt{bnb:int8}      & LLM.int8() absmax \\
6-bit          & GGUF    & \texttt{Q6\_K}         & k-quant, 6-bit (super-blocks) \\
5-bit          & GGUF    & \texttt{Q5\_K\_M}      & k-quant, 5-bit (medium) \\
4-bit          & GGUF    & \texttt{Q4\_K\_M}      & k-quant, 4-bit (medium) \\
               & BnB     & \texttt{bnb:nf4}       & NormalFloat4 (QLoRA format) \\
               & BnB     & \texttt{bnb:fp4}       & FP4 with FP16 compute \\
3-bit          & GGUF    & \texttt{Q3\_K\_M}      & k-quant, 3-bit (medium) \\
2-bit          & GGUF    & \texttt{Q2\_K}         & k-quant, 2-bit \\
\bottomrule
\end{tabular}
\end{table}

\subsection{Backend coverage and GPTQ checkpoints}
\label{app:gptq_awq}

Table~\ref{tab:gptq_awq} lists, per model, the backends evaluated and the pre-quantized GPTQ checkpoint loaded where one was available. GGUF (six bit-widths) and BitsAndBytes (three configurations) are run on every model; GPTQ is 4-bit (the setting reported in Sec.~\ref{subsec:exp_setup}) except for Qwen3-8B-Instruct, where the only public checkpoint we could load uses 8-bit. Cells marked ``---'' were outside the compute budget of this study.

\begin{table}[h]
\centering
\caption{Quantization backend coverage per model, with the GPTQ checkpoint used. \cmark\ = included; --- = not evaluated.}
\label{tab:gptq_awq}
\label{tab:coverage}
\small
\setlength{\tabcolsep}{4pt}
\resizebox{\textwidth}{!}{%
\begin{tabular}{@{}clccl@{}}
\toprule
\multicolumn{1}{c}{\textbf{\#}} & \multicolumn{1}{c}{\textbf{Model}} & \multicolumn{1}{c}{\textbf{BnB}} & \multicolumn{1}{c}{\textbf{GGUF}} & \multicolumn{1}{c}{\textbf{GPTQ repository}} \\
\midrule
1  & Qwen3-8B-Base        & \cmark & \cmark & --- \\
2  & Llama-3.1-8B         & \cmark & \cmark & \texttt{ModelCloud/Meta-Llama-3.1-8B-gptq-4bit} \\
3  & Ministral-3-Base     & \cmark & \cmark & --- \\
4  & Qwen3-8B             & \cmark & \cmark & \texttt{JunHowie/Qwen3-8B-GPTQ-Int8} \\
5  & Llama-3.1-8B-Inst.   & \cmark & \cmark & \texttt{ModelCloud/Meta-Llama-3.1-8B-Instruct-gptq-4bit} \\
6  & Ministral-3-Inst.    & \cmark & \cmark & --- \\
7  & DeepSeek-R1-Distill  & \cmark & \cmark & \texttt{jakiAJK/DeepSeek-R1-Distill-Llama-8B\_GPTQ-int4} \\
\bottomrule
\end{tabular}}
\end{table}

\subsection{GGUF repositories}
\label{app:gguf_repos}

Table~\ref{tab:gguf_repos} lists the GGUF repositories and filename conventions for each model.

\begin{table}[h]
\centering
\caption{GGUF repositories and filename templates. Each repository provides quants \texttt{Q8\_0}, \texttt{Q6\_K}, \texttt{Q5\_K\_M}, \texttt{Q4\_K\_M}, \texttt{Q3\_K\_M}, and \texttt{Q2\_K} unless otherwise noted.}
\label{tab:gguf_repos}
\small
\setlength{\tabcolsep}{4pt}
\resizebox{\textwidth}{!}{%
\begin{tabular}{@{}clll@{}}
\toprule
\multicolumn{1}{c}{\textbf{\#}} & \multicolumn{1}{c}{\textbf{Model}} & \multicolumn{1}{c}{\textbf{GGUF Repository}} & \multicolumn{1}{c}{\textbf{Filename Template}} \\
\midrule
1  & Qwen3-8B-Base        & \texttt{mradermacher/Qwen3-8B-Base-GGUF}       & \texttt{Qwen3-8B-Base.\{Q\}.gguf} \\
2  & Llama-3.1-8B         & \texttt{QuantFactory/Meta-Llama-3.1-8B-GGUF}   & \texttt{Meta-Llama-3.1-8B.\{Q\}.gguf} \\
3  & Ministral-3-Base     & \texttt{mradermacher/Ministral-3-8B-Base-2512-GGUF} & \texttt{Ministral-3-8B-Base-2512.\{Q\}.gguf} \\
4  & Qwen3-8B             & \texttt{Qwen/Qwen3-8B-GGUF}                    & \texttt{Qwen3-8B-\{Q\}.gguf} \\
5  & Llama-3.1-8B-Inst.   & \texttt{bartowski/Meta-Llama-3.1-8B-Instruct-GGUF} & \texttt{...-Instruct-\{Q\}.gguf} \\
6  & Ministral-3-Inst.    & \texttt{bartowski/mistralai\_Ministral-3-8B-Instruct-2512-GGUF} & \texttt{...-2512-\{Q\}.gguf} \\
7  & DeepSeek-R1-Distill  & \texttt{bartowski/DeepSeek-R1-Distill-Llama-8B-GGUF} & \texttt{...-Llama-8B-\{Q\}.gguf} \\
\bottomrule
\end{tabular}}
\end{table}


\section{Quantization replication, feature-only scatter, and per-model tables}
\label{app:quant_tables}

\subsection{Replication on Ministral and DeepSeek}
\label{app:replication_mistral_deepseek}

Figure~\ref{fig:quant_grid_bound_extra} replicates the main quantization scatter (Fig.~\ref{fig:quant_grid_bound}) on the two additional 8B families, Ministral-3-8B and DeepSeek-R1-Distill-Llama-8B. The bound tracks $|\Delta\mathcal{R}|$ on these two families with patterns consistent with the Llama--Qwen main analysis, supporting the cross-family generality of the PRISM framework. Per-variant numerical decompositions for all four families are in Appendix~\ref{app:per_model_tables}.

\begin{figure}[h]
    \centering
    \includegraphics[width=0.98\linewidth]{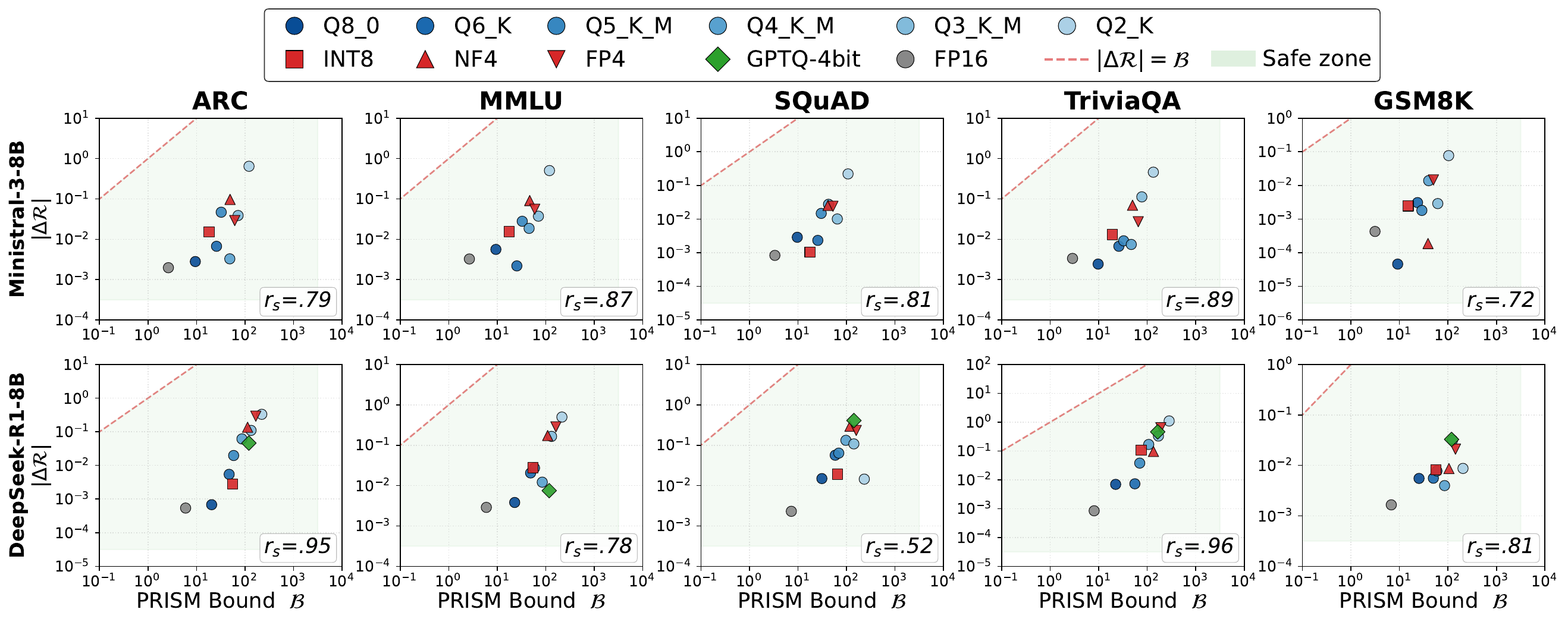}
    \caption{\textbf{PTQ replication on Ministral-3-8B and DeepSeek-R1-Distill-Llama-8B.} Counterpart of main-text Fig.~\ref{fig:quant_grid_bound} (Llama-3.1-8B and Qwen3-8B), with identical axes, PTQ family color coding, and benchmark set. Rows: Ministral-3-8B, DeepSeek-R1-Distill-Llama-8B. Columns: ARC, MMLU, SQuAD, TriviaQA, GSM8K. Per-subplot Spearman $r_s$ is annotated in each panel.}
    \label{fig:quant_grid_bound_extra}
\end{figure}

\subsection{Feature-only scatter}
\label{app:feature_only}

Figure~\ref{fig:quant_grid_feature} reports the feature-alignment-only replication of the main quantization result (Fig.~\ref{fig:quant_grid_bound}), with the x-axis replaced by the head-free term $\delta = K_{\mathrm{feat}}\sqrt{(\rho_T{-}\rho_P)^2 + 2\rho_T\rho_P(1-\Omega)}$. The grid isolates the backbone contribution to the bound.

\begin{figure}[h]
    \centering
    \includegraphics[width=0.98\linewidth]{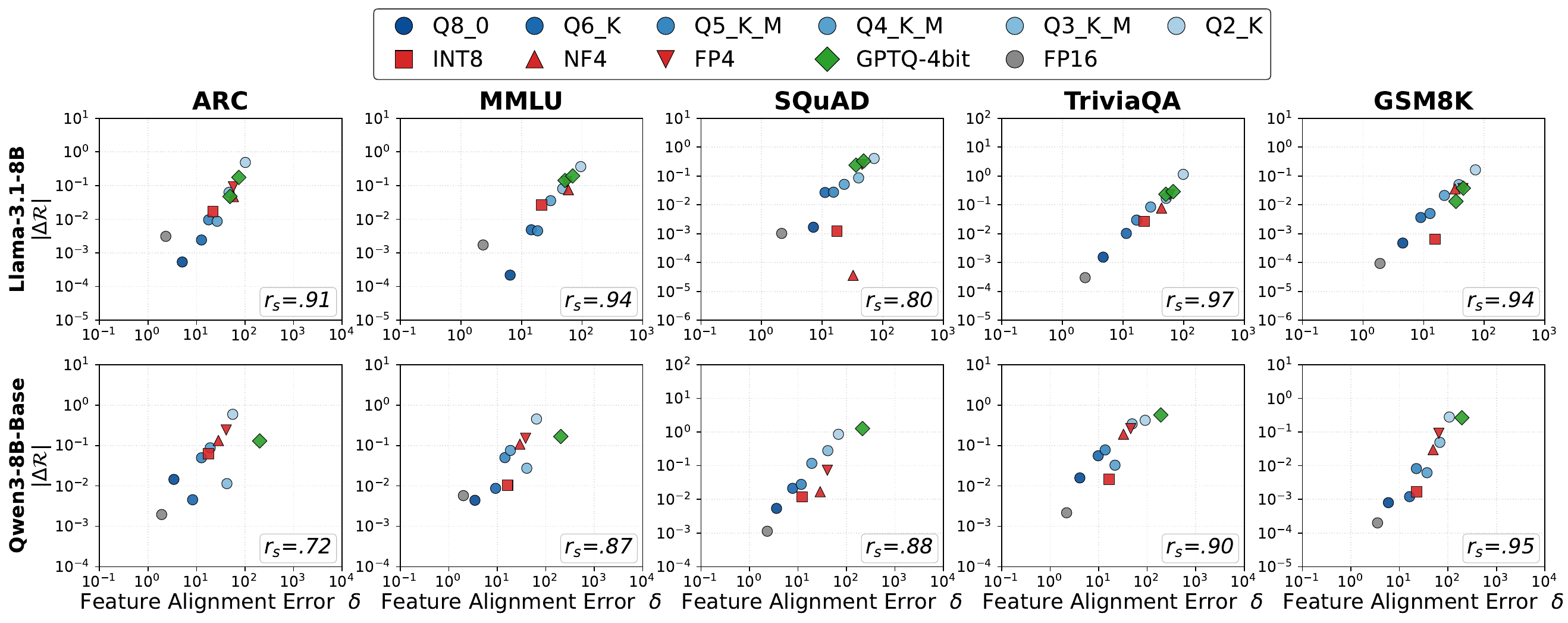}
    \caption{\textbf{Feature alignment error $\delta$ alone is already highly predictive of the risk gap.} Identical layout to Fig.~\ref{fig:quant_grid_bound} but with x-axis replaced by the feature-only term $\delta$. Because most protocols retain the \texttt{lm\_head} in FP16, the head-discrepancy term $\gamma$ is zero for GPTQ and BitsAndBytes and smaller than $\delta$ for most GGUF configurations that quantize output embeddings; $\delta$ therefore captures most of the predictive signal of Fig.~\ref{fig:quant_grid_bound}, so scale drift $(\Delta\rho)^2$ and shape mismatch $(1-\Omega)$ account for most of the degradation signal under PTQ (Sec.~\ref{subsec:quantization}).}
    \label{fig:quant_grid_feature}
\end{figure}

\subsection{Per-model decomposition tables}
\label{app:per_model_tables}

We report the full geometric decomposition (Eq.~\ref{eq:exact_equality}) for each evaluated model, covering all quantization variants and benchmarks. Each table column lists the base RMS scale $\rho_T$, proxy RMS scale $\rho_P$, identity-alignment similarity $\Omega$, feature alignment error $\delta$, head discrepancy $\gamma$, full PRISM bound $\mathcal{B}$, and empirical cross-entropy gap $|\Delta\mathcal{R}|$. Per-dataset Spearman rank correlations $r_s(\mathcal{B},|\Delta\mathcal{R}|)$ appear in the left-hand column of each benchmark block.

The base-model tables cover the four 8B families: Llama-3.1-8B extended benchmarks in Table~\ref{tab:llama_decomposition_ext_bound}, Ministral-3-8B in Table~\ref{tab:mistral_decomposition_all_bound}, Qwen3-8B in Table~\ref{tab:qwen_decomposition_all_bound} (also referenced in Sec.~\ref{subsec:decompose}), and DeepSeek-R1-Distill-Llama-8B in Table~\ref{tab:deepseek_decomposition_all_bound}. The instruction-tuned counterparts of Llama, Ministral, and Qwen are reported in Tables~\ref{tab:llama_instruct_decomposition_all_bound}, \ref{tab:mistral_instruct_decomposition_all_bound}, and~\ref{tab:qwen_instruct_decomposition_all_bound} respectively, and reproduce the same patterns as their base counterparts. Across ARC/MMLU/SQuAD/TriviaQA the per-benchmark mean Spearman lies in $r_s\in[0.76, 0.89]$ (Table~\ref{tab:gsm8k_outlier}), and GSM8K reaches $r_s\approx 0.73$. GSM8K's long teacher-forced chain-of-thought answer spans dilute per-token loss, shrinking mean $|\Delta\mathcal{R}|$ to $\approx 0.019$, four to nine times below the other benchmarks ($0.07$--$0.16$), so it is the benchmark with the least headroom between predictive signal and measurement noise, but the bound still ranks its variants. The single cell where that headroom is exhausted is Qwen3-8B-Instruct, whose mean $|\Delta\mathcal{R}|$ on GSM8K is only $0.0036$: at that scale per-variant differences are dominated by measurement noise rather than the bound's predictive signal, and the per-cell $r_s$ is correspondingly near zero.

\begin{table}[t]
\centering
\caption{Per-benchmark mean empirical risk gap $|\Delta\mathcal{R}|$ (averaged over per-method aggregated quantization rows, matching the per-model body tables) and per-benchmark mean Spearman $r_s(\mathcal{B},|\Delta\mathcal{R}|)$ (over raw variants, matching the figure's per-subplot $r_s$) across seven 8B model families. GSM8K's small $|\Delta\mathcal{R}|$---driven by long teacher-forced answer spans diluting per-token loss---leaves the least headroom between predictive signal and measurement noise; the effect shows up per family rather than in the benchmark mean.}
\label{tab:gsm8k_outlier}
\setlength{\tabcolsep}{5pt}
\begin{tabular}{l ccccc}
\toprule
 & \multicolumn{5}{c}{\textbf{Empirical risk gap} $\boldsymbol{|\Delta\mathcal{R}|}$} \\
\cmidrule(lr){2-6}
\multicolumn{1}{c}{\textbf{Model}} & \textbf{ARC} & \textbf{MMLU} & \textbf{SQuAD} & \textbf{TriviaQA} & \textbf{GSM8K} \\
\midrule
Llama-3.1-8B & 0.0766 & 0.0812 & 0.1030 & 0.1839 & 0.0311 \\
Ministral-3-8B & 0.0886 & 0.0761 & 0.0330 & 0.0711 & 0.0117 \\
Qwen3-8B & 0.1198 & 0.0893 & 0.1411 & 0.1412 & 0.0468 \\
DeepSeek-R1-8B & 0.0910 & 0.1112 & 0.1224 & 0.2680 & 0.0123 \\
Llama-3.1-8B-Instruct & 0.0709 & 0.1080 & 0.0744 & 0.2086 & 0.0204 \\
Ministral-3-8B-Instruct & 0.0274 & 0.0265 & 0.0372 & 0.1108 & 0.0052 \\
Qwen3-8B-Instruct & 0.0307 & 0.0275 & 0.0861 & 0.1559 & 0.0036 \\
\midrule
\textbf{Mean} $\boldsymbol{|\Delta\mathcal{R}|}$ & 0.0721 & 0.0743 & 0.0853 & 0.1628 & \cellcolor{red!18} \textbf{0.0187} \\
\midrule
\textbf{Mean} $\boldsymbol{r_s}$ & 0.758 & 0.786 & 0.794 & 0.889 & 0.726 \\
\bottomrule
\end{tabular}
\end{table}

\begingroup
\small
\setlength{\tabcolsep}{4pt}
\renewcommand{\arraystretch}{1.05}
\setlength{\LTcapwidth}{\textwidth}
\begin{longtable}{ll l rrrrrrr}
\caption{Geometric decomposition for \textbf{Llama-3.1-8B} under identity alignment ($W{=}I$) on extended benchmarks (excl.\ MMLU). Each benchmark section reports Spearman's $r_s(\mathcal{B},\,|\Delta\mathcal{R}|)$ across all quantization variants. Shading: \colorbox{red!35}{$\Omega{<}0.80$} / \colorbox{red!18}{$\Omega{<}0.95$} on $(\Omega,\delta,\mathcal{B},|\Delta\mathcal{R}|)$; \colorbox{cyan!12}{$\gamma{=}0$} when the head is preserved ($H_T{=}H_P$).}\label{tab:llama_decomposition_ext_bound}\\
\toprule
\multicolumn{1}{c}{\textbf{Dataset}} & \multicolumn{1}{c}{\textbf{Family}} & \multicolumn{1}{c}{\textbf{Method}} & \multicolumn{1}{c}{$\boldsymbol{\rho_T}$} & \multicolumn{1}{c}{$\boldsymbol{\rho_P}$} & \multicolumn{1}{c}{$\boldsymbol{\Omega}$} & \multicolumn{1}{c}{$\boldsymbol{\delta}$} & \multicolumn{1}{c}{$\boldsymbol{\gamma}$} & \multicolumn{1}{c}{\textbf{PRISM} $\boldsymbol{\mathcal{B}}$} & \multicolumn{1}{c}{$\boldsymbol{|\Delta\mathcal{R}|}$} \\
\midrule
\endfirsthead
\multicolumn{10}{l}{\small\itshape (continued from previous page)}\\
\toprule
\multicolumn{1}{c}{\textbf{Dataset}} & \multicolumn{1}{c}{\textbf{Family}} & \multicolumn{1}{c}{\textbf{Method}} & \multicolumn{1}{c}{$\boldsymbol{\rho_T}$} & \multicolumn{1}{c}{$\boldsymbol{\rho_P}$} & \multicolumn{1}{c}{$\boldsymbol{\Omega}$} & \multicolumn{1}{c}{$\boldsymbol{\delta}$} & \multicolumn{1}{c}{$\boldsymbol{\gamma}$} & \multicolumn{1}{c}{\textbf{PRISM} $\boldsymbol{\mathcal{B}}$} & \multicolumn{1}{c}{$\boldsymbol{|\Delta\mathcal{R}|}$} \\
\midrule
\endhead
\midrule
\multicolumn{10}{r}{\small\itshape continued on next page}\\
\endfoot
\bottomrule
\endlastfoot
ARC & -- & FP16 & 137.59 & 137.63 & 0.9999 & 6.0881 & \cellcolor{cyan!12} $0$ & 6.0882 & 0.0031 \\
{\small ($\mathbf{r_s{=}0.91}$)} & GGUF & Q8\_0 & 137.59 & 137.68 & 0.9993 & 13.2980 & 5.8902 & 19.1882 & 0.0005 \\
 &  & Q6\_K & 137.59 & 138.04 & 0.9958 & 33.1354 & 18.4531 & 51.5885 & 0.0024 \\
 &  & Q5\_K\_M & 137.59 & 138.57 & 0.9917 & 46.4894 & 18.5484 & 65.0377 & 0.0096 \\
 &  & Q4\_K\_M & 137.59 & 137.70 & 0.9812 & 69.7101 & 18.4173 & 88.1274 & 0.0087 \\
 &  & Q3\_K\_M & 137.59 & 138.67 & \cellcolor{red!18} 0.9424 & \cellcolor{red!18} 122.5642 & 18.5484 & \cellcolor{red!18} 141.1126 & \cellcolor{red!18} 0.0631 \\
 &  & Q2\_K & 137.59 & 144.67 & \cellcolor{red!35} 0.7413 & \cellcolor{red!35} 265.8843 & 19.6871 & \cellcolor{red!35} 285.5715 & \cellcolor{red!35} 0.4872 \\
 & BnB & INT8 & 137.59 & 137.98 & 0.9874 & 57.1181 & \cellcolor{cyan!12} $0$ & 57.1181 & 0.0168 \\
\end{longtable}
\endgroup


\begingroup
\small
\setlength{\tabcolsep}{4pt}
\renewcommand{\arraystretch}{1.05}
\setlength{\LTcapwidth}{\textwidth}
\begin{longtable}{ll l rrrrrrr}
\caption{Geometric decomposition for \textbf{Qwen3-8B-Base} under identity alignment ($W{=}I$) on all benchmarks. Each benchmark section reports Spearman's $r_s(\mathcal{B},\,|\Delta\mathcal{R}|)$ across all quantization variants. Shading: \colorbox{red!35}{$\Omega{<}0.80$} / \colorbox{red!18}{$\Omega{<}0.95$} on $(\Omega,\delta,\mathcal{B},|\Delta\mathcal{R}|)$; \colorbox{cyan!12}{$\gamma{=}0$} when the head is preserved ($H_T{=}H_P$).}\label{tab:qwen_decomposition_all_bound}\\
\toprule
\multicolumn{1}{c}{\textbf{Dataset}} & \multicolumn{1}{c}{\textbf{Family}} & \multicolumn{1}{c}{\textbf{Method}} & \multicolumn{1}{c}{$\boldsymbol{\rho_T}$} & \multicolumn{1}{c}{$\boldsymbol{\rho_P}$} & \multicolumn{1}{c}{$\boldsymbol{\Omega}$} & \multicolumn{1}{c}{$\boldsymbol{\delta}$} & \multicolumn{1}{c}{$\boldsymbol{\gamma}$} & \multicolumn{1}{c}{\textbf{PRISM} $\boldsymbol{\mathcal{B}}$} & \multicolumn{1}{c}{$\boldsymbol{|\Delta\mathcal{R}|}$} \\
\midrule
\endfirsthead
\multicolumn{10}{l}{\small\itshape (continued from previous page)}\\
\toprule
\multicolumn{1}{c}{\textbf{Dataset}} & \multicolumn{1}{c}{\textbf{Family}} & \multicolumn{1}{c}{\textbf{Method}} & \multicolumn{1}{c}{$\boldsymbol{\rho_T}$} & \multicolumn{1}{c}{$\boldsymbol{\rho_P}$} & \multicolumn{1}{c}{$\boldsymbol{\Omega}$} & \multicolumn{1}{c}{$\boldsymbol{\delta}$} & \multicolumn{1}{c}{$\boldsymbol{\gamma}$} & \multicolumn{1}{c}{\textbf{PRISM} $\boldsymbol{\mathcal{B}}$} & \multicolumn{1}{c}{$\boldsymbol{|\Delta\mathcal{R}|}$} \\
\midrule
\endhead
\midrule
\multicolumn{10}{r}{\small\itshape continued on next page}\\
\endfoot
\bottomrule
\endlastfoot
MMLU & -- & FP16 & 332.95 & 333.07 & 1.0000 & 6.9340 & \cellcolor{cyan!12} $0$ & 6.9341 & 0.0057 \\
{\small ($\mathbf{r_s{=}0.76}$)} & GGUF & Q8\_0 & 332.95 & 332.80 & 0.9999 & 11.9995 & 27.2373 & 39.2369 & 0.0044 \\
 &  & Q6\_K & 332.95 & 332.19 & 0.9996 & 32.0060 & 86.4603 & 118.4663 & 0.0087 \\
 &  & Q5\_K\_M & 332.95 & 332.09 & 0.9991 & 49.8368 & 86.4112 & 136.2480 & 0.0500 \\
 &  & Q4\_K\_M & 332.95 & 333.38 & 0.9984 & 64.6037 & 86.7488 & 151.3525 & 0.0754 \\
 &  & Q3\_K\_M & 332.95 & 324.48 & 0.9927 & 140.6362 & 84.4380 & 225.0742 & 0.0274 \\
 &  & Q2\_K & 332.95 & 336.96 & 0.9815 & 223.1666 & 87.5885 & 310.7551 & 0.4521 \\
 & BnB & INT8 & 332.95 & 332.30 & 0.9988 & 56.4563 & \cellcolor{cyan!12} $0$ & 56.4563 & 0.0104 \\
\end{longtable}
\endgroup


\begingroup
\small
\setlength{\tabcolsep}{4pt}
\renewcommand{\arraystretch}{1.05}
\setlength{\LTcapwidth}{\textwidth}
\begin{longtable}{ll l rrrrrrr}
\caption{Geometric decomposition for \textbf{Ministral-3-8B-Base} under identity alignment ($W{=}I$) on all benchmarks. Each benchmark section reports Spearman's $r_s(\mathcal{B},\,|\Delta\mathcal{R}|)$ across all quantization variants. Shading: \colorbox{red!35}{$\Omega{<}0.80$} / \colorbox{red!18}{$\Omega{<}0.95$} on $(\Omega,\delta,\mathcal{B},|\Delta\mathcal{R}|)$; \colorbox{cyan!12}{$\gamma{=}0$} when the head is preserved ($H_T{=}H_P$).}\label{tab:mistral_decomposition_all_bound}\\
\toprule
\multicolumn{1}{c}{\textbf{Dataset}} & \multicolumn{1}{c}{\textbf{Family}} & \multicolumn{1}{c}{\textbf{Method}} & \multicolumn{1}{c}{$\boldsymbol{\rho_T}$} & \multicolumn{1}{c}{$\boldsymbol{\rho_P}$} & \multicolumn{1}{c}{$\boldsymbol{\Omega}$} & \multicolumn{1}{c}{$\boldsymbol{\delta}$} & \multicolumn{1}{c}{$\boldsymbol{\gamma}$} & \multicolumn{1}{c}{\textbf{PRISM} $\boldsymbol{\mathcal{B}}$} & \multicolumn{1}{c}{$\boldsymbol{|\Delta\mathcal{R}|}$} \\
\midrule
\endfirsthead
\multicolumn{10}{l}{\small\itshape (continued from previous page)}\\
\toprule
\multicolumn{1}{c}{\textbf{Dataset}} & \multicolumn{1}{c}{\textbf{Family}} & \multicolumn{1}{c}{\textbf{Method}} & \multicolumn{1}{c}{$\boldsymbol{\rho_T}$} & \multicolumn{1}{c}{$\boldsymbol{\rho_P}$} & \multicolumn{1}{c}{$\boldsymbol{\Omega}$} & \multicolumn{1}{c}{$\boldsymbol{\delta}$} & \multicolumn{1}{c}{$\boldsymbol{\gamma}$} & \multicolumn{1}{c}{\textbf{PRISM} $\boldsymbol{\mathcal{B}}$} & \multicolumn{1}{c}{$\boldsymbol{|\Delta\mathcal{R}|}$} \\
\midrule
\endhead
\midrule
\multicolumn{10}{r}{\small\itshape continued on next page}\\
\endfoot
\bottomrule
\endlastfoot
MMLU & -- & FP16 & 191.05 & 191.05 & 0.9999 & 2.6652 & \cellcolor{cyan!12} $0$ & 2.6652 & 0.0033 \\
{\small ($\mathbf{r_s{=}0.87}$)} & GGUF & Q8\_0 & 191.05 & 191.03 & 0.9997 & 4.6604 & 4.7230 & 9.3834 & 0.0056 \\
 &  & Q6\_K & 191.05 & 191.12 & 0.9984 & 10.6224 & 14.7564 & 25.3788 & 0.0022 \\
 &  & Q5\_K\_M & 191.05 & 191.55 & 0.9954 & 18.0046 & 14.7573 & 32.7619 & 0.0280 \\
 &  & Q4\_K\_M & 191.05 & 191.11 & 0.9867 & 30.6699 & 14.7318 & 45.4017 & 0.0187 \\
 &  & Q3\_K\_M & 191.05 & 190.43 & 0.9556 & 55.9842 & 14.6636 & 70.6477 & 0.0374 \\
 &  & Q2\_K & 191.05 & 193.99 & \cellcolor{red!18} 0.8484 & \cellcolor{red!18} 104.4459 & 14.8564 & \cellcolor{red!18} 119.3023 & \cellcolor{red!18} 0.5035 \\
 & BnB & INT8 & 191.05 & 191.21 & 0.9957 & 17.5506 & \cellcolor{cyan!12} $0$ & 17.5506 & 0.0156 \\
\end{longtable}
\endgroup


\begingroup
\small
\setlength{\tabcolsep}{4pt}
\renewcommand{\arraystretch}{1.05}
\setlength{\LTcapwidth}{\textwidth}
\begin{longtable}{ll l rrrrrrr}
\caption{Geometric decomposition for \textbf{DeepSeek-R1-Distill-Llama-8B} under identity alignment ($W{=}I$) on all benchmarks. Each benchmark section reports Spearman's $r_s(\mathcal{B},\,|\Delta\mathcal{R}|)$ across all quantization variants. Shading: \colorbox{red!35}{$\Omega{<}0.80$} / \colorbox{red!18}{$\Omega{<}0.95$} on $(\Omega,\delta,\mathcal{B},|\Delta\mathcal{R}|)$; \colorbox{cyan!12}{$\gamma{=}0$} when the head is preserved ($H_T{=}H_P$).}\label{tab:deepseek_decomposition_all_bound}\\
\toprule
\multicolumn{1}{c}{\textbf{Dataset}} & \multicolumn{1}{c}{\textbf{Family}} & \multicolumn{1}{c}{\textbf{Method}} & \multicolumn{1}{c}{$\boldsymbol{\rho_T}$} & \multicolumn{1}{c}{$\boldsymbol{\rho_P}$} & \multicolumn{1}{c}{$\boldsymbol{\Omega}$} & \multicolumn{1}{c}{$\boldsymbol{\delta}$} & \multicolumn{1}{c}{$\boldsymbol{\gamma}$} & \multicolumn{1}{c}{\textbf{PRISM} $\boldsymbol{\mathcal{B}}$} & \multicolumn{1}{c}{$\boldsymbol{|\Delta\mathcal{R}|}$} \\
\midrule
\endfirsthead
\multicolumn{10}{l}{\small\itshape (continued from previous page)}\\
\toprule
\multicolumn{1}{c}{\textbf{Dataset}} & \multicolumn{1}{c}{\textbf{Family}} & \multicolumn{1}{c}{\textbf{Method}} & \multicolumn{1}{c}{$\boldsymbol{\rho_T}$} & \multicolumn{1}{c}{$\boldsymbol{\rho_P}$} & \multicolumn{1}{c}{$\boldsymbol{\Omega}$} & \multicolumn{1}{c}{$\boldsymbol{\delta}$} & \multicolumn{1}{c}{$\boldsymbol{\gamma}$} & \multicolumn{1}{c}{\textbf{PRISM} $\boldsymbol{\mathcal{B}}$} & \multicolumn{1}{c}{$\boldsymbol{|\Delta\mathcal{R}|}$} \\
\midrule
\endhead
\midrule
\multicolumn{10}{r}{\small\itshape continued on next page}\\
\endfoot
\bottomrule
\endlastfoot
MMLU & -- & FP16 & 130.42 & 130.31 & 0.9998 & 5.9618 & \cellcolor{cyan!12} $0$ & 5.9618 & 0.0029 \\
{\small ($\mathbf{r_s{=}0.78}$)} & GGUF & Q8\_0 & 130.42 & 130.04 & 0.9987 & 17.3076 & 5.5805 & 22.8882 & 0.0038 \\
 &  & Q6\_K & 130.42 & 130.50 & 0.9957 & 31.3821 & 17.3731 & 48.7552 & 0.0205 \\
 &  & Q5\_K\_M & 130.42 & 130.31 & 0.9925 & 41.3732 & 17.3657 & 58.7389 & 0.0273 \\
 &  & Q4\_K\_M & 130.42 & 130.72 & 0.9803 & 67.3439 & 17.4300 & 84.7740 & 0.0121 \\
 &  & Q3\_K\_M & 130.42 & 133.79 & \cellcolor{red!18} 0.9453 & \cellcolor{red!18} 113.8700 & 17.8649 & \cellcolor{red!18} 131.7349 & \cellcolor{red!18} 0.1651 \\
 &  & Q2\_K & 130.42 & 134.50 & \cellcolor{red!18} 0.8356 & \cellcolor{red!18} 197.6929 & 17.8740 & \cellcolor{red!18} 215.5669 & \cellcolor{red!18} 0.4977 \\
 & BnB & INT8 & 130.42 & 131.29 & 0.9872 & 54.5359 & \cellcolor{cyan!12} $0$ & 54.5359 & 0.0277 \\
\end{longtable}
\endgroup


\begingroup
\small
\setlength{\tabcolsep}{4pt}
\renewcommand{\arraystretch}{1.05}
\setlength{\LTcapwidth}{\textwidth}
\begin{longtable}{ll l rrrrrrr}
\caption{Geometric decomposition for \textbf{Llama-3.1-8B-Instruct} under identity alignment ($W{=}I$) on all benchmarks. Each benchmark section reports Spearman's $r_s(\mathcal{B},\,|\Delta\mathcal{R}|)$ across all quantization variants. Shading: \colorbox{red!35}{$\Omega{<}0.80$} / \colorbox{red!18}{$\Omega{<}0.95$} on $(\Omega,\delta,\mathcal{B},|\Delta\mathcal{R}|)$; \colorbox{cyan!12}{$\gamma{=}0$} when the head is preserved ($H_T{=}H_P$).}\label{tab:llama_instruct_decomposition_all_bound}\\
\toprule
\multicolumn{1}{c}{\textbf{Dataset}} & \multicolumn{1}{c}{\textbf{Family}} & \multicolumn{1}{c}{\textbf{Method}} & \multicolumn{1}{c}{$\boldsymbol{\rho_T}$} & \multicolumn{1}{c}{$\boldsymbol{\rho_P}$} & \multicolumn{1}{c}{$\boldsymbol{\Omega}$} & \multicolumn{1}{c}{$\boldsymbol{\delta}$} & \multicolumn{1}{c}{$\boldsymbol{\gamma}$} & \multicolumn{1}{c}{\textbf{PRISM} $\boldsymbol{\mathcal{B}}$} & \multicolumn{1}{c}{$\boldsymbol{|\Delta\mathcal{R}|}$} \\
\midrule
\endfirsthead
\multicolumn{10}{l}{\small\itshape (continued from previous page)}\\
\toprule
\multicolumn{1}{c}{\textbf{Dataset}} & \multicolumn{1}{c}{\textbf{Family}} & \multicolumn{1}{c}{\textbf{Method}} & \multicolumn{1}{c}{$\boldsymbol{\rho_T}$} & \multicolumn{1}{c}{$\boldsymbol{\rho_P}$} & \multicolumn{1}{c}{$\boldsymbol{\Omega}$} & \multicolumn{1}{c}{$\boldsymbol{\delta}$} & \multicolumn{1}{c}{$\boldsymbol{\gamma}$} & \multicolumn{1}{c}{\textbf{PRISM} $\boldsymbol{\mathcal{B}}$} & \multicolumn{1}{c}{$\boldsymbol{|\Delta\mathcal{R}|}$} \\
\midrule
\endhead
\midrule
\multicolumn{10}{r}{\small\itshape continued on next page}\\
\endfoot
\bottomrule
\endlastfoot
MMLU & -- & FP16 & 144.62 & 144.62 & 0.9997 & 8.6575 & \cellcolor{cyan!12} $0$ & 8.6575 & 0.0006 \\
{\small ($\mathbf{r_s{=}0.71}$)} & GGUF & Q8\_0 & 144.62 & 144.52 & 0.9955 & 35.4732 & 6.1943 & 41.6675 & 0.0225 \\
 &  & Q6\_K & 144.62 & 144.62 & 0.9924 & 46.2823 & 19.4739 & 65.7563 & 0.0147 \\
 &  & Q5\_K\_M & 144.62 & 144.61 & 0.9852 & 64.7792 & 19.4811 & 84.2603 & 0.0134 \\
 &  & Q4\_K\_M & 144.62 & 143.10 & \cellcolor{red!18} 0.9402 & \cellcolor{red!18} 129.4149 & 19.1528 & \cellcolor{red!18} 148.5677 & \cellcolor{red!18} 0.1047 \\
 &  & Q3\_K\_M & 144.62 & 141.89 & \cellcolor{red!18} 0.8549 & \cellcolor{red!18} 200.7990 & 18.9551 & \cellcolor{red!18} 219.7541 & \cellcolor{red!18} 0.0306 \\
 &  & Q2\_K & 144.62 & 144.23 & \cellcolor{red!35} 0.7332 & \cellcolor{red!35} 274.3229 & 19.4327 & \cellcolor{red!35} 293.7556 & \cellcolor{red!35} 0.6372 \\
 & BnB & INT8 & 144.62 & 144.43 & 0.9772 & 80.2094 & \cellcolor{cyan!12} $0$ & 80.2094 & 0.0423 \\
\end{longtable}
\endgroup


\begingroup
\small
\setlength{\tabcolsep}{3pt}
\renewcommand{\arraystretch}{1.05}
\setlength{\LTcapwidth}{\textwidth}
\begin{longtable}{ll l rrrrrrr}
\caption{Geometric decomposition for \textbf{Qwen3-8B-Instruct} under identity alignment ($W{=}I$) on all benchmarks. Each benchmark section reports Spearman's $r_s(\mathcal{B},\,|\Delta\mathcal{R}|)$ across all quantization variants. Shading: \colorbox{red!35}{$\Omega{<}0.80$} / \colorbox{red!18}{$\Omega{<}0.95$} on $(\Omega,\delta,\mathcal{B},|\Delta\mathcal{R}|)$; \colorbox{cyan!12}{$\gamma{=}0$} when the head is preserved ($H_T{=}H_P$).}\label{tab:qwen_instruct_decomposition_all_bound}\\
\toprule
\multicolumn{1}{c}{\textbf{Dataset}} & \multicolumn{1}{c}{\textbf{Family}} & \multicolumn{1}{c}{\textbf{Method}} & \multicolumn{1}{c}{$\boldsymbol{\rho_T}$} & \multicolumn{1}{c}{$\boldsymbol{\rho_P}$} & \multicolumn{1}{c}{$\boldsymbol{\Omega}$} & \multicolumn{1}{c}{$\boldsymbol{\delta}$} & \multicolumn{1}{c}{$\boldsymbol{\gamma}$} & \multicolumn{1}{c}{\textbf{PRISM} $\boldsymbol{\mathcal{B}}$} & \multicolumn{1}{c}{$\boldsymbol{|\Delta\mathcal{R}|}$} \\
\midrule
\endfirsthead
\multicolumn{10}{l}{\small\itshape (continued from previous page)}\\
\toprule
\multicolumn{1}{c}{\textbf{Dataset}} & \multicolumn{1}{c}{\textbf{Family}} & \multicolumn{1}{c}{\textbf{Method}} & \multicolumn{1}{c}{$\boldsymbol{\rho_T}$} & \multicolumn{1}{c}{$\boldsymbol{\rho_P}$} & \multicolumn{1}{c}{$\boldsymbol{\Omega}$} & \multicolumn{1}{c}{$\boldsymbol{\delta}$} & \multicolumn{1}{c}{$\boldsymbol{\gamma}$} & \multicolumn{1}{c}{\textbf{PRISM} $\boldsymbol{\mathcal{B}}$} & \multicolumn{1}{c}{$\boldsymbol{|\Delta\mathcal{R}|}$} \\
\midrule
\endhead
\midrule
\multicolumn{10}{r}{\small\itshape continued on next page}\\
\endfoot
\bottomrule
\endlastfoot
MMLU & -- & FP16 & 159.13 & 159.08 & 0.9999 & 8.2992 & \cellcolor{cyan!12} $0$ & 8.2992 & 0.0044 \\
{\small ($\mathbf{r_s{=}0.72}$)} & GGUF & Q8\_0 & 159.13 & 159.36 & 0.9996 & 14.7119 & 12.7219 & 27.4338 & 0.0013 \\
 &  & Q6\_K & 159.13 & 160.11 & 0.9978 & 36.5955 & 41.0117 & 77.6072 & 0.0219 \\
 &  & Q5\_K\_M & 159.13 & 159.95 & 0.9951 & 53.8170 & 40.9531 & 94.7701 & 0.0075 \\
 &  & Q4\_K\_M & 159.13 & 157.24 & 0.9887 & 81.2492 & 40.2340 & 121.4832 & 0.0812 \\
 & BnB & INT8 & 159.13 & 161.48 & 0.9926 & 66.9423 & \cellcolor{cyan!12} $0$ & 66.9423 & 0.0584 \\
 &  & NF4 & 159.13 & 161.58 & 0.9772 & 117.0377 & \cellcolor{cyan!12} $0$ & 117.0377 & 0.0248 \\
 &  & FP4 & 159.13 & 161.02 & \cellcolor{red!18} 0.9462 & \cellcolor{red!18} 179.2144 & \cellcolor{cyan!12} $0$ & \cellcolor{red!18} 179.2144 & \cellcolor{red!18} 0.0456 \\
\end{longtable}
\endgroup


\begingroup
\small
\setlength{\tabcolsep}{4pt}
\renewcommand{\arraystretch}{1.05}
\setlength{\LTcapwidth}{\textwidth}
\begin{longtable}{ll l rrrrrrr}
\caption{Geometric decomposition for \textbf{Ministral-3-8B-Instruct} under identity alignment ($W{=}I$) on all benchmarks. Each benchmark section reports Spearman's $r_s(\mathcal{B},\,|\Delta\mathcal{R}|)$ across all quantization variants. Shading: \colorbox{red!35}{$\Omega{<}0.80$} / \colorbox{red!18}{$\Omega{<}0.95$} on $(\Omega,\delta,\mathcal{B},|\Delta\mathcal{R}|)$; \colorbox{cyan!12}{$\gamma{=}0$} when the head is preserved ($H_T{=}H_P$).}\label{tab:mistral_instruct_decomposition_all_bound}\\
\toprule
\multicolumn{1}{c}{\textbf{Dataset}} & \multicolumn{1}{c}{\textbf{Family}} & \multicolumn{1}{c}{\textbf{Method}} & \multicolumn{1}{c}{$\boldsymbol{\rho_T}$} & \multicolumn{1}{c}{$\boldsymbol{\rho_P}$} & \multicolumn{1}{c}{$\boldsymbol{\Omega}$} & \multicolumn{1}{c}{$\boldsymbol{\delta}$} & \multicolumn{1}{c}{$\boldsymbol{\gamma}$} & \multicolumn{1}{c}{\textbf{PRISM} $\boldsymbol{\mathcal{B}}$} & \multicolumn{1}{c}{$\boldsymbol{|\Delta\mathcal{R}|}$} \\
\midrule
\endfirsthead
\multicolumn{10}{l}{\small\itshape (continued from previous page)}\\
\toprule
\multicolumn{1}{c}{\textbf{Dataset}} & \multicolumn{1}{c}{\textbf{Family}} & \multicolumn{1}{c}{\textbf{Method}} & \multicolumn{1}{c}{$\boldsymbol{\rho_T}$} & \multicolumn{1}{c}{$\boldsymbol{\rho_P}$} & \multicolumn{1}{c}{$\boldsymbol{\Omega}$} & \multicolumn{1}{c}{$\boldsymbol{\delta}$} & \multicolumn{1}{c}{$\boldsymbol{\gamma}$} & \multicolumn{1}{c}{\textbf{PRISM} $\boldsymbol{\mathcal{B}}$} & \multicolumn{1}{c}{$\boldsymbol{|\Delta\mathcal{R}|}$} \\
\midrule
\endhead
\midrule
\multicolumn{10}{r}{\small\itshape continued on next page}\\
\endfoot
\bottomrule
\endlastfoot
MMLU & -- & FP16 & 189.49 & 189.49 & 1.0000 & 0.0000 & \cellcolor{cyan!12} $0$ & 0.0000 & 0.0000 \\
{\small ($\mathbf{r_s{=}0.76}$)} & GGUF & Q8\_0 & 189.49 & 189.43 & 0.9969 & 13.8312 & 4.4501 & 18.2813 & 0.0037 \\
 &  & Q6\_K & 189.49 & 189.43 & 0.9960 & 15.7838 & 13.9847 & 29.7685 & 0.0170 \\
 &  & Q5\_K\_M & 189.49 & 189.77 & 0.9936 & 19.8389 & 13.9930 & 33.8319 & 0.0267 \\
 &  & Q4\_K\_M & 189.49 & 189.46 & 0.9836 & 31.8793 & 13.9821 & 45.8614 & 0.0119 \\
 &  & Q3\_K\_M & 189.49 & 187.05 & 0.9626 & 47.8645 & 13.8643 & 61.7289 & 0.0236 \\
 &  & Q2\_K & 189.49 & 187.52 & \cellcolor{red!18} 0.8843 & \cellcolor{red!18} 84.1785 & 13.9004 & \cellcolor{red!18} 98.0789 & \cellcolor{red!18} 0.1117 \\
 & BnB & INT8 & 189.49 & 189.33 & 0.9953 & 16.9969 & \cellcolor{cyan!12} $0$ & 16.9969 & 0.0089 \\
\end{longtable}
\endgroup

\section{Similarity baselines, component ablation, and how small the reference set can be}
\label{app:similarity_baselines}

\subsection{Full-size comparison}
\label{app:similarity_fullsize}

\begin{table}[t]
\centering
\caption{\textbf{Component-wise ablation and similarity baselines: every metric ranks variants about equally well.} Per-cell Spearman $r_s$ between each degradation score and $|\Delta\mathcal{R}|$, averaged over the five benchmarks of each family; higher is better, and all columns are computed on identical features, identical variants and identical $|\Delta\mathcal{R}|$ targets. Top block: representational-similarity baselines. Middle/bottom blocks: the nested PRISM quantities (shape $\to +$ scale $\to +$ head) under the identity alignment $W{=}I$ and the Procrustes-optimal $W{=}W_N$. \textbf{Bold}: best in column (ties bolded). The last column is the SEM of the mean over the ten cells; it exceeds every between-metric gap in the table, so $\mathcal{B}_N$'s lead over the baselines is not resolved at this sample size (Appendix~\ref{app:similarity_baselines}). The main text adopts $W{=}I$ for SVD-free differentiability and the $H_T{=}H_P$ head-term simplification under a frozen \texttt{lm\_head}.}
\label{tab:baseline_combined}
\small
\setlength{\tabcolsep}{4pt}
\scalebox{0.95}{\begin{tabular}{l cc cc}
\toprule
\multicolumn{1}{c}{\textbf{Metric}} & \multicolumn{1}{c}{\textbf{Llama}} & \multicolumn{1}{c}{\textbf{Qwen3}} & \multicolumn{1}{c}{\textbf{Mean} $\boldsymbol{r_s}$} & \multicolumn{1}{c}{$\pm$\,\textbf{SEM}} \\
 & \scriptsize (5 benchmarks) & \scriptsize (5 benchmarks) & \scriptsize (10 cells) & \scriptsize (10 cells) \\
\midrule
\makebox[5.6em][l]{$1{-}$CKA}{\scriptsize (linear CKA; \citealp{kornblith2019similarity})} & 0.924 & \textbf{0.873} & 0.899 & 0.025 \\
\makebox[5.6em][l]{$1{-}$SVCCA}{\scriptsize (SVCCA; \citealp{raghu2017svcca})} & 0.929 & 0.869 & 0.899 & 0.024 \\
\midrule
\makebox[5.6em][l]{$\Omega$}{\scriptsize (shape similarity; baseline)} & 0.915 & 0.865 & 0.890 & 0.024 \\
\makebox[5.6em][l]{$\delta$}{\scriptsize (+ scale; feature alignment error)} & 0.910 & 0.865 & 0.888 & 0.024 \\
\makebox[5.6em][l]{$\mathcal{B}$}{\scriptsize (+ head; full bound, $W{=}I$)} & 0.916 & 0.835 & 0.875 & 0.045 \\
\midrule
\makebox[5.6em][l]{$\Omega_N$}{\scriptsize (shape similarity; baseline)} & \textbf{0.933} & 0.864 & 0.898 & 0.027 \\
\makebox[5.6em][l]{$\delta_N$}{\scriptsize (+ scale; feature alignment error)} & \textbf{0.933} & 0.864 & 0.898 & 0.027 \\
\makebox[5.6em][l]{$\mathcal{B}_N$}{\scriptsize (+ head; full bound, $W{=}W_N$)} & 0.923 & \textbf{0.889} & \textbf{0.906} & 0.021 \\
\bottomrule
\end{tabular}}
\end{table}

Table~\ref{tab:baseline_combined} scores two standard representational-similarity metrics against the nested PRISM quantities on exactly the same footing: the same base models, an identical variant set for every column within each cell, feature matrices $Z_T, Z_P$ extracted on $\mathcal{D}_{\mathrm{diag}}$ under the identical protocol, and the same $|\Delta\mathcal{R}|$ targets. The two extraction runs are independent but agree closely: the $\Omega$ values recomputed here match the main pipeline's to three decimals (median absolute difference $2.7\times10^{-5}$), and the $|\Delta\mathcal{R}|$ targets are bit-identical, so the blocks of the table are paired cell by cell. Every entry averages Spearman correlations between a degradation score and $|\Delta\mathcal{R}|$ computed within one (family, benchmark) cell, so higher is better and all columns of a row are computed over an identical variant set. Linear CKA~\citep{kornblith2019similarity} and SVCCA~\citep{raghu2017svcca} are similarities, so we score $1{-}\mathrm{CKA}$ and $1{-}\mathrm{SVCCA}$; we also ran the Procrustes distance $\min_{W\in\mathcal{O}(d)}\|Z_T - Z_P W\|_F/\sqrt{n}$, which is already a distance, and it lands at $0.898$ over the ten cells, indistinguishable from the other two.

The baselines rank variants about as well as PRISM does ($0.899$, $0.899$ and $0.898$ over the ten cells, against $0.875$ for $\mathcal{B}$ and $0.906$ for $\mathcal{B}_N$), and none of the eighteen paired comparisons (three baselines $\times$ six PRISM quantities, $n{=}10$ paired cells) survives a multiplicity correction. The full bound is the clearest case: $|t| \le 1.2$ and $p \ge 0.26$ against all three baselines under both alignments. The single nominally significant comparison is $1{-}$CKA over $\Omega$, a gap of $0.009$ at uncorrected $p{=}0.043$ (Wilcoxon $p{=}0.062$), which is what one expects to see once among eighteen tests. This is expected rather than disappointing: within one cell the variants differ mainly in how far the feature manifold has been distorted, and any statistic that responds monotonically to that distortion will produce nearly the same order.

\paragraph{What the nested ablation does and does not measure.}
Under $W{=}I$, the shape-only similarity $\Omega$ already reaches $0.890$, adding scale to form $\delta$ leaves it at $0.888$, and adding the head term to form $\mathcal{B}$ costs a further $0.013$. Reading that as ``the extra axes are redundant'' would mistake what the ablation measures. It asks a single question, which nested quantity best orders variants \emph{inside} one cell, and the answer is governed by what PTQ happens to do to these models, not by whether an axis is informative. Scale is the clearest case: across the $360$ rows of the grid the shape term exceeds the scale term by a median factor of ${\sim}7{,}700$ and dominates in all but $5$ rows, so $\Delta\rho$ has almost no variance to contribute to a pooled PTQ order. That is a finding about post-training quantization, not a defect of the axis: the same axis is the discriminating channel exactly where it does move (Qwen3-8B-Base Q2\_K on GSM8K, Sec.~\ref{subsec:decompose}) and under LoRA drift (Sec.~\ref{subsec:forget_exp}), which is the point of keeping it separate rather than folded into a scalar.

\paragraph{The head term's effect on ranking depends on the alignment.}
Under $W{=}I$ the head term costs $0.013$; under $W{=}W_N$, where the alignment absorbs the $H_T \to H_P$ rotation, the same term \emph{adds} $0.008$ ($\delta_N{=}0.898 \to \mathcal{B}_N{=}0.906$). So $\gamma$ carries genuine signal, and what a pooled $W{=}I$ order penalizes is not head geometry but a protocol offset: $\gamma{>}0$ on the $200$ of $360$ rows whose GGUF k-quant tier quantizes \texttt{lm\_head} and $\gamma{\equiv}0$ on GPTQ/BnB, nearly constant across bit-widths within a family. Interleaving two such regimes in one rank order is the questionable operation. Qwen3-8B-Base ARC, the cell with the largest gap against a baseline, shows it cleanly: pooled over both protocols $\mathcal{B}$ ranks poorly, yet within the FP16-head variants alone it ranks almost perfectly and within the GGUF variants alone clearly better than pooled. The split is not a general repair, since in some cells the GGUF subgroup is itself noisy, but it reveals protocol-dependent ranking behavior hidden by the pooled correlation, which a scalar similarity score cannot do. $\mathcal{B}$ remains our default: it is Theorem~\ref{thm:unified_bound}'s certified upper bound and the only one of these metrics that retains that guarantee when $\gamma{>}0$. $\Omega$ and $\delta$ remain useful ranking scores, but $\delta$ alone does not retain the risk-bound guarantee when the head term is nonzero. The $W{=}I$ trace specialization trades ${\sim}0.03$ Spearman for SVD-free differentiability, frozen-\texttt{lm\_head} head-term simplification, and consistency with the Sec.~\ref{subsec:shape_reg} regularizer, so the choice between the two alignments is primarily design-driven.

\subsection{How the comparison changes with the size of the reference set}
\label{app:refset_size}

Appendix~\ref{app:similarity_fullsize} compares the scores at the paper's operating point, $512$ sequences. The practically interesting question is the opposite one: how small can $\mathcal{D}_{\mathrm{diag}}$ be before a score stops ordering variants correctly? This matters because it sets how much reference data, and how many forward passes, a screen needs.

\paragraph{Protocol.}
An independent round on Llama-3.1-8B over its $12$ variants, with three seeds ($43/44/45$) not used for the main quantization grid. Reference slices of $n \in \{8, 32, 128, 512\}$ sequences are nested by construction (shuffle, then take the first $n$), so a small slice is a subset of the large one and no distribution shift is introduced. Each entry is the Spearman correlation between the score computed on the $n$-sequence slice and the risk gap measured on the \emph{full} $512$-sequence slice (the decision-relevant question: can a cheap slice recover the ordering the full protocol would give?), reported as mean $\pm$ sd over the three seeds.

\begin{table}[h]
\centering
\caption{\textbf{Eight reference sequences already fix the ordering; the risk gap measured on the same slice does not.} Llama-3.1-8B, $12$ variants, seeds $43/44/45$; every column is $r_s$ against $|\Delta\mathcal{R}|$ measured on the full $512$-sequence slice, mean $\pm$ sd over seeds. The first column is $|\Delta\mathcal{R}|$ itself measured on the small slice; at $n{=}512$ it is the target and equals $1$ by definition.}
\label{tab:refset_size}
\small
\resizebox{\textwidth}{!}{%
\begin{tabular}{lccccc}
\toprule
\textbf{Reference size} & $|\Delta\mathcal{R}|$ on the slice & $1{-}$CKA & $1{-}$SVCCA & \textbf{PRISM} $\delta$ ($W{=}I$) & \textbf{PRISM} $\mathcal{B}_N$ \\
\midrule
$8$ sequences & $0.795 \pm 0.059$ & $0.903 \pm 0.015$ & $0.083 \pm 0.036$ & $0.930 \pm 0.012$ & $\mathbf{0.932 \pm 0.016}$ \\ 
$32$ & $0.825 \pm 0.071$ & $0.932 \pm 0.010$ & $0.821 \pm 0.100$ & $0.927 \pm 0.012$ & $0.931 \pm 0.009$ \\ 
$128$ & $0.934 \pm 0.033$ & $0.934 \pm 0.013$ & $0.894 \pm 0.001$ & $0.927 \pm 0.011$ & $0.933 \pm 0.011$ \\ 
$512$ \scriptsize(paper protocol) & $1$ \scriptsize(target) & $0.931 \pm 0.008$ & $0.941 \pm 0.014$ & $0.925 \pm 0.011$ & $0.932 \pm 0.011$ \\ 
\bottomrule
\end{tabular}}
\end{table}

\paragraph{The bound does not degrade; the risk gap and CKA do, and SVCCA leaves its domain of validity.}
$\mathcal{B}_N$ is flat at $r_s \approx 0.932$ from $8$ sequences to $512$, and the $W{=}I$ feature term $\delta$, computed from the same stored features, is as flat ($0.930$ at $n{=}8$, $0.925$ at $n{=}512$). The weakest small-sample statistic in the table is the risk gap measured on the slice itself: $0.795$ at $n{=}8$ and $0.825$ at $n{=}32$, reaching the bound's level only at $n{=}128$. This is the direct answer to why one would compute a bound when the same forward pass yields the gap (Sec.~\ref{subsec:cost}): on a small slice the bound orders the variants and the gap does not.
The feature arm behaves the same way ($\delta_N$: $0.924 \pm 0.007$ at $n{=}8$, $0.944 \pm 0.012$ at $n{=}512$). $1{-}$CKA loses $0.03$ at $n{=}8$ but degrades gracefully. $1{-}$SVCCA reads $0.083$, which we do not count as a comparison: canonical correlations need more samples than dimensions, and with fewer than $d$ rows both centered feature matrices span the same subspace, so every canonical correlation is $1$ by construction ($1{-}$SVCCA is numerically zero on ARC, MMLU, SQuAD and TriviaQA at $n{=}8$). The statistic recovers as the slice grows, and GSM8K, whose eight sequences already supply hundreds of tokens, is unaffected at every size.

\paragraph{Sequences are not tokens.}
The metrics live on answer-region tokens, and the answer is a single token for multiple-choice tasks, so a fixed number of \emph{sequences} buys very different numbers of \emph{rows} per benchmark. At $512$ sequences on Llama we measure $511$ answer tokens on MMLU, $527$ on ARC, $1{,}524$ on TriviaQA, $2{,}330$ on SQuAD and $52{,}184$ on GSM8K (Qwen3-8B-Base: $511$, $527$, $1{,}556$, $2{,}515$, $63{,}056$). This is exactly why SVCCA works at every size on GSM8K and nowhere else at $n{=}8$: eight GSM8K sequences already supply ${\sim}815$ rows, while eight MMLU sequences supply eight. The requirement is therefore best stated in answer tokens rather than in sequences, and stating it that way also sharpens the eight-sequence result: on the multiple-choice benchmarks eight sequences means on the order of eight rows against $d{=}4096$, and the bound still recovers the full-slice ordering from them, which is the regime where a canonical-correlation statistic cannot function at all.

\paragraph{Qwen3-8B replication.}
The same round on Qwen3-8B-Base ($11$ variants, same seeds and sizes) gives the same picture at a lower ceiling: $\mathcal{B}_N$ is flat at $0.827 \pm 0.035$ ($n{=}8$) against $0.827 \pm 0.031$ ($n{=}512$); the risk gap on the slice reads $0.755 \pm 0.109$ at $n{=}8$, $0.878$ at $n{=}32$ and $0.928$ at $n{=}128$; $1{-}$CKA reads $0.819 \pm 0.024$ and $1{-}$SVCCA $0.427 \pm 0.017$ at $n{=}8$.

\paragraph{Provenance.}
This is a separate round from Table~\ref{tab:baseline_combined}: different seeds, different $512$-sequence draws, and, for the full bound, the $W{=}W_N$ arm only; the $W{=}I$ head term needs the quantized heads, which this round did not store, so the $W{=}I$ column reports the feature term $\delta$. Its $n{=}512$ row therefore reproduces rather than repeats the Llama column of Table~\ref{tab:baseline_combined} ($1{-}$CKA $0.931$ vs.\ $0.924$, $1{-}$SVCCA $0.941$ vs.\ $0.929$, $\mathcal{B}_N$ $0.932$ vs.\ $0.923$), which also serves as a seed-robustness check on that table.

\section{Teacher forcing versus decoding: the cost of one screening pass}
\label{app:cost}

This appendix substantiates the efficiency claim of Sec.~\ref{subsec:predict} with a controlled measurement. All timings are wall-clock seconds on one RTX~5090 in the same unoptimized HuggingFace runner, with no inference server such as vLLM on either side.

\paragraph{Teacher forcing versus decoding, measured directly.}
GSM8K generation dominates any benchmark run that includes it, so we measure it directly rather than assuming a generation budget. On $256$ GSM8K prompts at batch size $4$, with the same model and the same prompts on both sides:

\begin{center}
\small
\begin{tabular}{lrr}
\toprule
& \textbf{wall clock} & \textbf{tokens} \\
\midrule
Teacher-forced scoring $\times 1$ (the PRISM extraction call) & $8.9$\,s & $26{,}341$ scored \\
Greedy decode $\times 1$ to natural EOS (ceiling $1024$) & $556.7$\,s & $53{,}457$ generated \\
\midrule
ratio & $\ge 62.6\times$ & \\
\bottomrule
\end{tabular}
\end{center}

\noindent Three points on how to read this number.
\emph{(i) It is not a per-token comparison, and should not be.} The decode side emits roughly twice as many tokens as the teacher-forced side scores, because the model's natural chain of thought averages $208.8$ tokens per prompt against $102$ for the gold answer. That is not a confound to be normalized away: obtaining a GSM8K score \emph{requires} generating the chain of thought, and its length is not known in advance. The comparison is per screening decision, which is the unit a practitioner actually pays for.
\emph{(ii) It is a lower bound.} $34/256$ prompts ($13.3\%$) were still generating at the $1024$-token ceiling, base models that never emit EOS; the remaining $222$ averaged only $84$ tokens. Truncation removes work from the decode side, so raising the ceiling would only increase the ratio. We report $\ge 62.6\times$ rather than re-running with a larger ceiling.
\emph{(iii) The gap is structural.} Teacher forcing scores a whole sequence in one forward pass; decoding needs one forward pass per generated token, about $209$ in sequence here. Batching and a faster serving stack shrink both sides by a constant but leave this ratio. Greedy decoding is deterministic, and a repeat of the measurement agreed to within $0.1$\,s.


\section{Qwen3-8B forgetting and shape-regularization replications}
\label{app:qwen_forgetting}

This appendix reports the Qwen3-8B replications of Sec.~\ref{subsec:forget_exp} and Sec.~\ref{subsec:shape_reg_exp} along with the remaining decomposition tables (Llama BBQ and both Qwen fine-tuning tasks). Qwen3-8B differs from Llama-3.1-8B in depth (36 vs.\ 32 layers), vocabulary size ($151{,}936$ vs.\ $128{,}256$), and pre-training corpus (multilingual, $\sim$36T tokens vs.\ English-dominant, $\sim$15T tokens); reproducing the main-text patterns under these shifts provides cross-family evidence for the claims of Secs.~\ref{subsec:forget_exp}--\ref{subsec:shape_reg_exp}.

\subsection{Forgetting: bound tracks cross-task drift on Qwen3-8B}
\label{app:qwen_forget_grid}

\begin{figure}[h]
    \centering
    \includegraphics[width=1.0\linewidth]{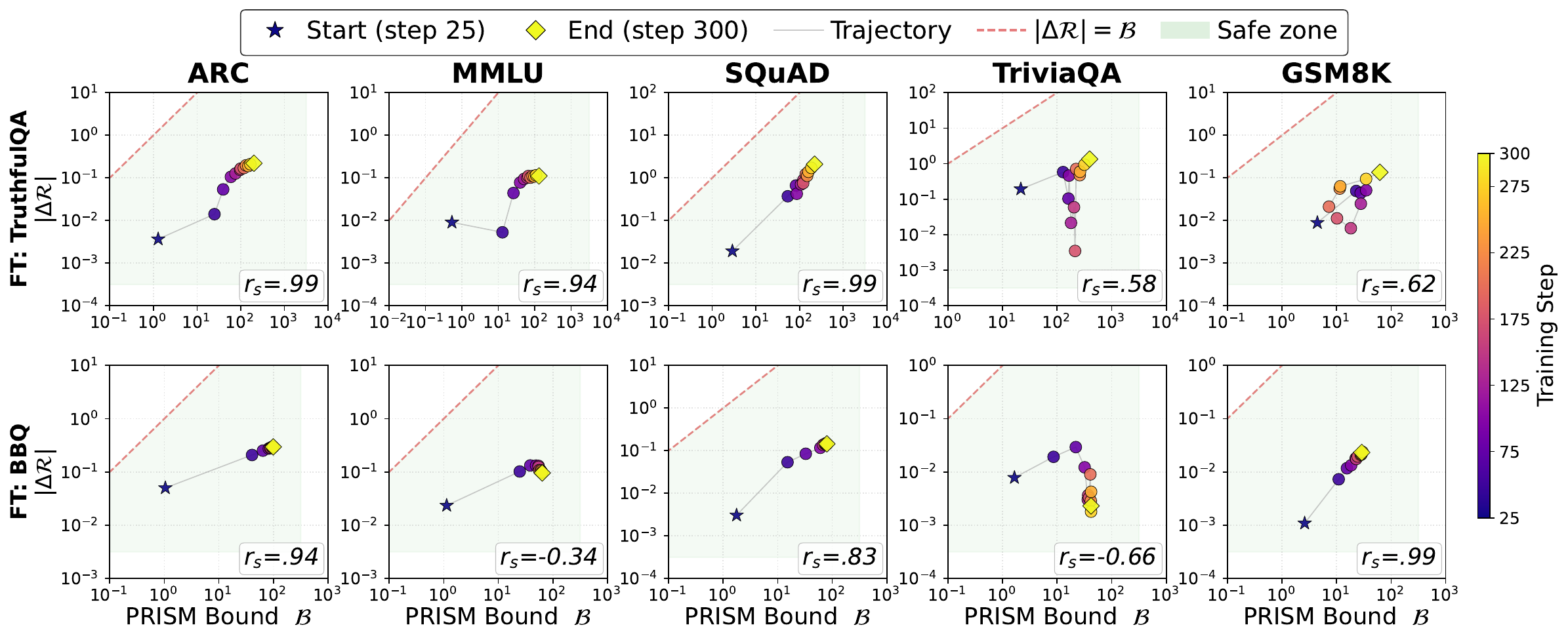}
    \caption{\textbf{Qwen3-8B: the PRISM bound tracks catastrophic forgetting across LoRA fine-tuning steps.} Each subplot scatters the PRISM bound $\mathcal{B}$ (x-axis, log) against the empirical forgetting $|\Delta\mathcal{R}|$ (y-axis, log) on a held-out benchmark, with one point per LoRA checkpoint colored by training step. Rows: fine-tuning task (TruthfulQA, BBQ). Columns: held-out evaluation benchmark (ARC, MMLU, SQuAD, TriviaQA, GSM8K). Because LoRA keeps the \texttt{lm\_head} frozen, the head-discrepancy term $\gamma$ vanishes at $W{=}I$ and the bound is governed entirely by backbone scale ($\Delta\rho$) and shape ($1-\Omega$) drift, as derived in Sec.~\ref{subsec:shape_reg}. Spearman $r_s$ per subplot is annotated in each panel.}
    \label{fig:forget_grid_qwen}
\end{figure}

Figure~\ref{fig:forget_grid_qwen} reproduces the two structural observations of the Llama main-text result (Fig.~\ref{fig:forget_grid_llama}) on Qwen3-8B. First, \emph{the PRISM bound rises in lockstep with empirical forgetting across most $(task, benchmark)$ combinations}: as LoRA fine-tuning proceeds on TruthfulQA or BBQ, $\Omega(Z_0, Z_t)$ drifts away from $1$ and $\rho_t$ from $\rho_0$, and this backbone drift tracks the empirical cross-entropy drift on most held-out benchmarks, including benchmarks (e.g., GSM8K) whose distributions differ substantially from the fine-tuning data. Second, \emph{different source tasks induce qualitatively different drift geometries}, matching Llama: TruthfulQA fine-tuning drives predominantly shape drift while BBQ moves both scale and shape, a distinction the scale--shape split of Theorem~\ref{thm:unified_bound} makes visible whereas a unified distance would collapse it. The replication holds in $8$ of $10$ cells; the two exceptions are analyzed next.

\paragraph{Per-subplot variability under low-forgetting regimes.}
A small subset of subplots in Fig.~\ref{fig:forget_grid_qwen}, most notably Qwen3-8B fine-tuned on BBQ evaluated on MMLU and TriviaQA, shows weak or slightly negative per-subplot $r_s$. In both cells the checkpoints have little to rank. On TriviaQA, $|\Delta\mathcal{R}|$ at $\lambda{=}0$ stays at order $10^{-3}$ nats per token (${\approx}0.0035$ at step $300$, against ${\approx}0.288$ on ARC under the same fine-tuning run; Table~\ref{tab:reg_compare_qwen_bbq}), two orders of magnitude below the scale at which the bound is informative across the rest of the grid. On MMLU, forgetting rises to ${\approx}0.13$ by step $75$ and then stays within $0.11$--$0.13$, so later checkpoints differ only by noise. In both cases the per-checkpoint $r_s$ reflects checkpoint-trajectory noise rather than a failure of the bound, and is a property of the (model, fine-tuning task, benchmark) triple rather than of the diagnostic. The same noise-dominated effect appears in Llama-3.1-8B BBQ$\to$SQuAD ($r_s\approx 0.30$, Fig.~\ref{fig:forget_grid_llama}): $|\Delta\mathcal{R}|$ saturates early in fine-tuning, so late checkpoints sit at roughly the same level and the per-checkpoint Spearman is driven by checkpoint noise rather than continued drift.

\subsection{Shape regularization: Qwen3-8B replication}
\label{app:qwen_shape_reg}

\begin{figure}[h]
    \centering
    \includegraphics[width=1.0\linewidth]{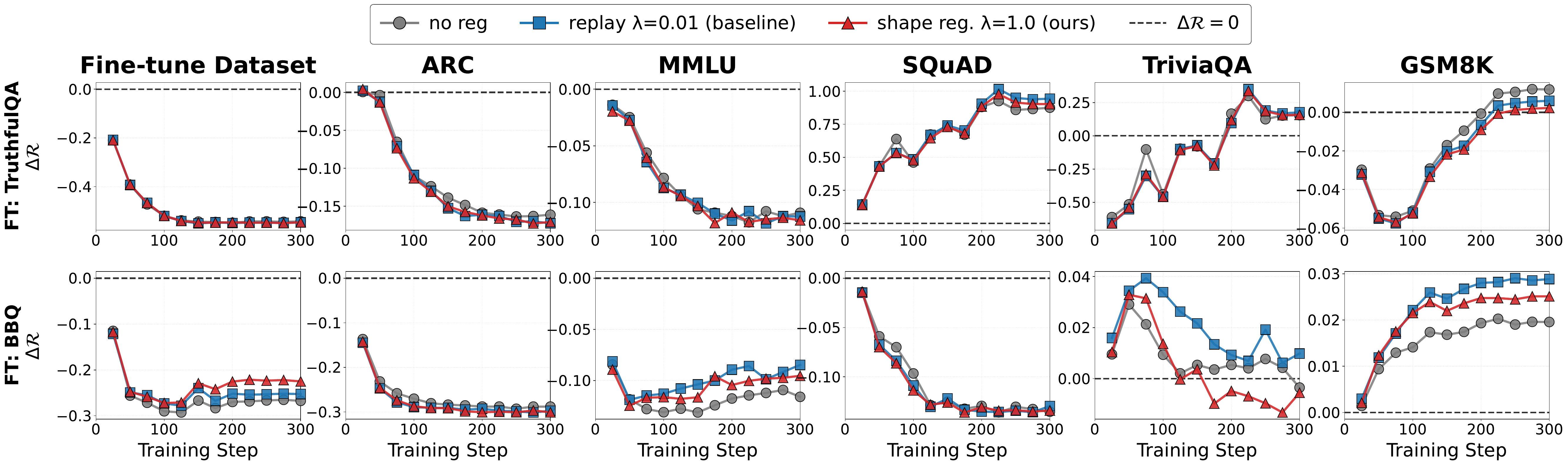}
    \caption{\textbf{Shape regularization vs.\ replay-CE on Qwen3-8B (replication of Fig.~\ref{fig:shape_reg_combined_llama}).} Same three configurations as Llama: \emph{no reg}, \emph{replay} ($\lambda{=}0.01$), and \emph{shape regularizer} ($\lambda{=}1.0$). Qwen3-8B's baseline forgetting is already small (mean $|\Delta\mathcal{R}|<0.27$ on both fine-tuning tasks vs.\ Llama's $0.84$ on TruthfulQA), leaving little room for improvement; the regularizer-vs-replay gap narrows accordingly, and neither method changes forgetting beyond noise. Per-config decompositions: Tables~\ref{tab:reg_compare_qwen_truthfulqa} and~\ref{tab:reg_compare_qwen_bbq}.}
    \label{fig:shape_reg_combined_qwen}
\end{figure}

Figure~\ref{fig:shape_reg_combined_qwen} reproduces the main-text Llama comparison (Fig.~\ref{fig:shape_reg_combined_llama}) on Qwen3-8B. Because Qwen3-8B's baseline forgetting is already small, both methods stay near the no-regularization level: the shape regularizer moves $\Omega$ by less than $10^{-3}$, and neither method changes forgetting beyond noise. Per-config decompositions are in Tables~\ref{tab:reg_compare_qwen_truthfulqa} and~\ref{tab:reg_compare_qwen_bbq} below.

\subsection{Additional decomposition tables}
\label{app:trace_norm_tables}

Tables~\ref{tab:reg_compare_llama_truthfulqa}--\ref{tab:reg_compare_qwen_bbq} give the full per-$(\text{model}, \text{fine-tuning task})$ decomposition at step $300$ for the three configurations (no reg, replay $\lambda{=}0.01$, shape regularizer $\lambda{=}1.0$); the main-text Table~\ref{tab:reg_compact_llama_truthfulqa} compresses the Llama-TruthfulQA combination to $\Omega$ and $|\Delta\mathcal{R}|$ only.

\begin{table}[t]
\centering
\caption{Regularization comparison for \textbf{Llama-3.1-8B} fine-tuned on \textbf{TruthfulQA} under identity alignment ($W{=}I$). Rows group by evaluation benchmark; each benchmark block lists metrics at step 300 for: no reg, replay $\lambda{=}0.01$ (baseline), shape regularizer $\lambda{=}1.0$ (ours). Bold / underline: 1st / 2nd-best smallest $|\Delta\mathcal{R}|$ and largest $\Omega$ per block. Shading: \colorbox{red!35}{$\Omega{<}0.80$} / \colorbox{red!18}{$\Omega{<}0.95$} on $(\Omega,\delta,\mathcal{B},|\Delta\mathcal{R}|)$; \colorbox{cyan!12}{$\gamma{=}0$} under frozen \texttt{lm\_head}.}
\label{tab:reg_compare_llama_truthfulqa}
\resizebox{0.95\textwidth}{!}{%
\begin{tabular}{l l rrrrrrr}
\toprule
\multicolumn{1}{c}{\textbf{Dataset}} & \multicolumn{1}{c}{\textbf{Method} ($\boldsymbol{\lambda}$)} & \multicolumn{1}{c}{$\boldsymbol{\rho_T}$} & \multicolumn{1}{c}{$\boldsymbol{\rho_P}$} & \multicolumn{1}{c}{$\boldsymbol{\Omega}$} & \multicolumn{1}{c}{$\boldsymbol{\delta}$} & \multicolumn{1}{c}{$\boldsymbol{\gamma}$} & \multicolumn{1}{c}{\textbf{PRISM} $\boldsymbol{\mathcal{B}}$} & \multicolumn{1}{c}{$\boldsymbol{|\Delta\mathcal{R}|}$} \\
\midrule
ARC & no reg & 137.55 & 137.75 & \cellcolor{red!18} 0.9158 & \cellcolor{red!18} 147.6081 & \cellcolor{cyan!12} $0$ & \cellcolor{red!18} 147.6081 & \cellcolor{red!18} 0.0290 \\
 & replay $\lambda{=}0.01$ (baseline) & 137.55 & 137.35 & \cellcolor{red!18} \underline{0.9185} & \cellcolor{red!18} 145.0853 & \cellcolor{cyan!12} $0$ & \cellcolor{red!18} 145.0853 & \cellcolor{red!18} \textbf{0.0108} \\
 & shape reg.\ $\lambda{=}1.0$ (ours) & 137.55 & 137.32 & \cellcolor{red!18} \textbf{0.9319} & \cellcolor{red!18} 132.5490 & \cellcolor{cyan!12} $0$ & \cellcolor{red!18} 132.5490 & \cellcolor{red!18} \underline{0.0199} \\
\midrule
MMLU & no reg & 138.91 & 139.50 & \cellcolor{red!18} 0.9426 & \cellcolor{red!18} 123.2722 & \cellcolor{cyan!12} $0$ & \cellcolor{red!18} 123.2722 & \cellcolor{red!18} 0.1339 \\
 & replay $\lambda{=}0.01$ (baseline) & 138.91 & 138.94 & \cellcolor{red!18} \underline{0.9451} & \cellcolor{red!18} 120.3033 & \cellcolor{cyan!12} $0$ & \cellcolor{red!18} 120.3033 & \cellcolor{red!18} \textbf{0.1031} \\
 & shape reg.\ $\lambda{=}1.0$ (ours) & 138.91 & 139.09 & \textbf{0.9542} & 109.9868 & \cellcolor{cyan!12} $0$ & 109.9868 & \underline{0.1116} \\
\midrule
SQuAD & no reg & 138.20 & 140.07 & \cellcolor{red!18} 0.9231 & \cellcolor{red!18} 142.6639 & \cellcolor{cyan!12} $0$ & \cellcolor{red!18} 142.6639 & \cellcolor{red!18} 1.3372 \\
 & replay $\lambda{=}0.01$ (baseline) & 138.20 & 139.58 & \cellcolor{red!18} \underline{0.9345} & \cellcolor{red!18} 131.4408 & \cellcolor{cyan!12} $0$ & \cellcolor{red!18} 131.4408 & \cellcolor{red!18} \underline{1.1911} \\
 & shape reg.\ $\lambda{=}1.0$ (ours) & 138.20 & 139.04 & \cellcolor{red!18} \textbf{0.9478} & \cellcolor{red!18} 117.0684 & \cellcolor{cyan!12} $0$ & \cellcolor{red!18} 117.0684 & \cellcolor{red!18} \textbf{1.0543} \\
\midrule
TriviaQA & no reg & 140.93 & 146.85 & \cellcolor{red!35} 0.7593 & \cellcolor{red!35} 261.3595 & \cellcolor{cyan!12} $0$ & \cellcolor{red!35} 261.3595 & \cellcolor{red!35} 2.5829 \\
 & replay $\lambda{=}0.01$ (baseline) & 140.93 & 146.18 & \cellcolor{red!35} \underline{0.7810} & \cellcolor{red!35} 248.6873 & \cellcolor{cyan!12} $0$ & \cellcolor{red!35} 248.6873 & \cellcolor{red!35} \underline{2.3879} \\
 & shape reg.\ $\lambda{=}1.0$ (ours) & 140.93 & 145.10 & \cellcolor{red!18} \textbf{0.8210} & \cellcolor{red!18} 223.8704 & \cellcolor{cyan!12} $0$ & \cellcolor{red!18} 223.8704 & \cellcolor{red!18} \textbf{2.1242} \\
\midrule
GSM8K & no reg & 137.15 & 138.32 & 0.9905 & 49.6934 & \cellcolor{cyan!12} $0$ & 49.6934 & 0.1342 \\
 & replay $\lambda{=}0.01$ (baseline) & 137.15 & 138.01 & \underline{0.9957} & 33.3099 & \cellcolor{cyan!12} $0$ & 33.3099 & \underline{0.1261} \\
 & shape reg.\ $\lambda{=}1.0$ (ours) & 137.15 & 137.65 & \textbf{1.0000} & 1.3192 & \cellcolor{cyan!12} $0$ & 1.3192 & \textbf{0.0966} \\
\bottomrule
\end{tabular}}
\end{table}

\begin{table}[t]
\centering
\caption{Regularization comparison for \textbf{Llama-3.1-8B} fine-tuned on \textbf{BBQ} under identity alignment ($W{=}I$). Rows group by evaluation benchmark; each benchmark block lists metrics at step 300 for: no reg, replay $\lambda{=}0.01$ (baseline), shape regularizer $\lambda{=}1.0$ (ours). Bold / underline: 1st / 2nd-best smallest $|\Delta\mathcal{R}|$ and largest $\Omega$ per block. Shading: \colorbox{red!35}{$\Omega{<}0.80$} / \colorbox{red!18}{$\Omega{<}0.95$} on $(\Omega,\delta,\mathcal{B},|\Delta\mathcal{R}|)$; \colorbox{cyan!12}{$\gamma{=}0$} under frozen \texttt{lm\_head}.}
\label{tab:reg_compare_llama_bbq}
\resizebox{0.95\textwidth}{!}{%
\begin{tabular}{l l rrrrrrr}
\toprule
\multicolumn{1}{c}{\textbf{Dataset}} & \multicolumn{1}{c}{\textbf{Method} ($\boldsymbol{\lambda}$)} & \multicolumn{1}{c}{$\boldsymbol{\rho_T}$} & \multicolumn{1}{c}{$\boldsymbol{\rho_P}$} & \multicolumn{1}{c}{$\boldsymbol{\Omega}$} & \multicolumn{1}{c}{$\boldsymbol{\delta}$} & \multicolumn{1}{c}{$\boldsymbol{\gamma}$} & \multicolumn{1}{c}{\textbf{PRISM} $\boldsymbol{\mathcal{B}}$} & \multicolumn{1}{c}{$\boldsymbol{|\Delta\mathcal{R}|}$} \\
\midrule
ARC & no reg & 137.55 & 133.78 & \cellcolor{red!18} \underline{0.8804} & \cellcolor{red!18} 173.7045 & \cellcolor{cyan!12} $0$ & \cellcolor{red!18} 173.7045 & \cellcolor{red!18} \textbf{0.0775} \\
 & replay $\lambda{=}0.01$ (baseline) & 137.55 & 135.25 & \cellcolor{red!18} 0.8802 & \cellcolor{red!18} 174.6273 & \cellcolor{cyan!12} $0$ & \cellcolor{red!18} 174.6273 & \cellcolor{red!18} 0.1799 \\
 & shape reg.\ $\lambda{=}1.0$ (ours) & 137.55 & 137.06 & \textbf{0.9762} & 78.3406 & \cellcolor{cyan!12} $0$ & 78.3406 & \underline{0.1536} \\
\midrule
MMLU & no reg & 138.91 & 135.04 & \cellcolor{red!18} 0.8744 & \cellcolor{red!18} 179.6944 & \cellcolor{cyan!12} $0$ & \cellcolor{red!18} 179.6944 & \cellcolor{red!18} \textbf{0.3531} \\
 & replay $\lambda{=}0.01$ (baseline) & 138.91 & 136.63 & \cellcolor{red!18} \underline{0.8766} & \cellcolor{red!18} 178.9974 & \cellcolor{cyan!12} $0$ & \cellcolor{red!18} 178.9974 & \cellcolor{red!18} 0.6061 \\
 & shape reg.\ $\lambda{=}1.0$ (ours) & 138.91 & 138.66 & \textbf{0.9758} & 79.8590 & \cellcolor{cyan!12} $0$ & 79.8590 & \underline{0.4995} \\
\midrule
SQuAD & no reg & 138.20 & 139.04 & \cellcolor{red!18} 0.9281 & \cellcolor{red!18} 137.3737 & \cellcolor{cyan!12} $0$ & \cellcolor{red!18} 137.3737 & \cellcolor{red!18} \underline{0.2608} \\
 & replay $\lambda{=}0.01$ (baseline) & 138.20 & 139.16 & \cellcolor{red!18} \underline{0.9366} & \cellcolor{red!18} 129.0819 & \cellcolor{cyan!12} $0$ & \cellcolor{red!18} 129.0819 & \cellcolor{red!18} \textbf{0.2127} \\
 & shape reg.\ $\lambda{=}1.0$ (ours) & 138.20 & 138.34 & \textbf{0.9694} & 89.4630 & \cellcolor{cyan!12} $0$ & 89.4630 & 0.2903 \\
\midrule
TriviaQA & no reg & 140.93 & 141.47 & 0.9780 & 77.4262 & \cellcolor{cyan!12} $0$ & 77.4262 & \underline{0.1435} \\
 & replay $\lambda{=}0.01$ (baseline) & 140.93 & 141.60 & \underline{0.9784} & 76.6896 & \cellcolor{cyan!12} $0$ & 76.6896 & 0.1506 \\
 & shape reg.\ $\lambda{=}1.0$ (ours) & 140.93 & 140.99 & \textbf{0.9955} & 34.9074 & \cellcolor{cyan!12} $0$ & 34.9074 & \textbf{0.0171} \\
\midrule
GSM8K & no reg & 137.15 & 137.41 & \textbf{1.0000} & 0.6918 & \cellcolor{cyan!12} $0$ & 0.6918 & 0.0609 \\
 & replay $\lambda{=}0.01$ (baseline) & 137.15 & 137.41 & \underline{1.0000} & 0.6895 & \cellcolor{cyan!12} $0$ & 0.6895 & \underline{0.0581} \\
 & shape reg.\ $\lambda{=}1.0$ (ours) & 137.15 & 137.10 & 1.0000 & 0.1255 & \cellcolor{cyan!12} $0$ & 0.1255 & \textbf{0.0126} \\
\bottomrule
\end{tabular}}
\end{table}

\begin{table}[t]
\centering
\caption{Regularization comparison for \textbf{Qwen3-8B-Base} fine-tuned on \textbf{TruthfulQA} under identity alignment ($W{=}I$). Rows group by evaluation benchmark; each benchmark block lists metrics at step 300 for: no reg, replay $\lambda{=}0.01$ (baseline), shape regularizer $\lambda{=}1.0$ (ours). Bold / underline: 1st / 2nd-best smallest $|\Delta\mathcal{R}|$ and largest $\Omega$ per block. Shading: \colorbox{red!35}{$\Omega{<}0.80$} / \colorbox{red!18}{$\Omega{<}0.95$} on $(\Omega,\delta,\mathcal{B},|\Delta\mathcal{R}|)$; \colorbox{cyan!12}{$\gamma{=}0$} under frozen \texttt{lm\_head}.}
\label{tab:reg_compare_qwen_truthfulqa}
\resizebox{0.95\textwidth}{!}{%
\begin{tabular}{l l rrrrrrr}
\toprule
\multicolumn{1}{c}{\textbf{Dataset}} & \multicolumn{1}{c}{\textbf{Method} ($\boldsymbol{\lambda}$)} & \multicolumn{1}{c}{$\boldsymbol{\rho_T}$} & \multicolumn{1}{c}{$\boldsymbol{\rho_P}$} & \multicolumn{1}{c}{$\boldsymbol{\Omega}$} & \multicolumn{1}{c}{$\boldsymbol{\delta}$} & \multicolumn{1}{c}{$\boldsymbol{\gamma}$} & \multicolumn{1}{c}{\textbf{PRISM} $\boldsymbol{\mathcal{B}}$} & \multicolumn{1}{c}{$\boldsymbol{|\Delta\mathcal{R}|}$} \\
\midrule
ARC & no reg & 332.98 & 327.46 & \underline{0.9963} & 100.0550 & \cellcolor{cyan!12} $0$ & 100.0550 & \textbf{0.1612} \\
 & replay $\lambda{=}0.01$ (baseline) & 332.98 & 327.21 & 0.9959 & 105.4953 & \cellcolor{cyan!12} $0$ & 105.4953 & 0.1725 \\
 & shape reg.\ $\lambda{=}1.0$ (ours) & 332.98 & 327.29 & \textbf{0.9965} & 97.1908 & \cellcolor{cyan!12} $0$ & 97.1908 & \underline{0.1708} \\
\midrule
MMLU & no reg & 333.53 & 328.76 & \underline{0.9984} & 67.4051 & \cellcolor{cyan!12} $0$ & 67.4051 & \textbf{0.1092} \\
 & replay $\lambda{=}0.01$ (baseline) & 333.53 & 328.53 & 0.9981 & 71.7638 & \cellcolor{cyan!12} $0$ & 71.7638 & \underline{0.1125} \\
 & shape reg.\ $\lambda{=}1.0$ (ours) & 333.53 & 328.70 & \textbf{0.9985} & 64.5258 & \cellcolor{cyan!12} $0$ & 64.5258 & 0.1158 \\
\midrule
SQuAD & no reg & 297.68 & 292.04 & \underline{0.9930} & 122.0758 & \cellcolor{cyan!12} $0$ & 122.0758 & \textbf{0.8757} \\
 & replay $\lambda{=}0.01$ (baseline) & 297.68 & 292.31 & 0.9926 & 125.2384 & \cellcolor{cyan!12} $0$ & 125.2384 & 0.9437 \\
 & shape reg.\ $\lambda{=}1.0$ (ours) & 297.68 & 293.13 & \textbf{0.9932} & 120.3011 & \cellcolor{cyan!12} $0$ & 120.3011 & \underline{0.9025} \\
\midrule
TriviaQA & no reg & 240.42 & 224.70 & \underline{0.9667} & 214.4331 & \cellcolor{cyan!12} $0$ & 214.4331 & \textbf{0.1552} \\
 & replay $\lambda{=}0.01$ (baseline) & 240.42 & 223.97 & 0.9647 & 220.7355 & \cellcolor{cyan!12} $0$ & 220.7355 & 0.1767 \\
 & shape reg.\ $\lambda{=}1.0$ (ours) & 240.42 & 225.12 & \textbf{0.9674} & 212.1221 & \cellcolor{cyan!12} $0$ & 212.1221 & \underline{0.1575} \\
\midrule
GSM8K & no reg & 269.06 & 271.54 & \textbf{1.0000} & 8.5518 & \cellcolor{cyan!12} $0$ & 8.5518 & 0.0118 \\
 & replay $\lambda{=}0.01$ (baseline) & 269.06 & 272.36 & \underline{1.0000} & 11.3864 & \cellcolor{cyan!12} $0$ & 11.3864 & \underline{0.0060} \\
 & shape reg.\ $\lambda{=}1.0$ (ours) & 269.06 & 272.50 & 1.0000 & 11.8835 & \cellcolor{cyan!12} $0$ & 11.8835 & \textbf{0.0023} \\
\bottomrule
\end{tabular}}
\end{table}

\begin{table}[t]
\centering
\caption{Regularization comparison for \textbf{Qwen3-8B-Base} fine-tuned on \textbf{BBQ} under identity alignment ($W{=}I$). Rows group by evaluation benchmark; each benchmark block lists metrics at step 300 for: no reg, replay $\lambda{=}0.01$ (baseline), shape regularizer $\lambda{=}1.0$ (ours). Bold / underline: 1st / 2nd-best smallest $|\Delta\mathcal{R}|$ and largest $\Omega$ per block. Shading: \colorbox{red!35}{$\Omega{<}0.80$} / \colorbox{red!18}{$\Omega{<}0.95$} on $(\Omega,\delta,\mathcal{B},|\Delta\mathcal{R}|)$; \colorbox{cyan!12}{$\gamma{=}0$} under frozen \texttt{lm\_head}.}
\label{tab:reg_compare_qwen_bbq}
\resizebox{0.95\textwidth}{!}{%
\begin{tabular}{l l rrrrrrr}
\toprule
\multicolumn{1}{c}{\textbf{Dataset}} & \multicolumn{1}{c}{\textbf{Method} ($\boldsymbol{\lambda}$)} & \multicolumn{1}{c}{$\boldsymbol{\rho_T}$} & \multicolumn{1}{c}{$\boldsymbol{\rho_P}$} & \multicolumn{1}{c}{$\boldsymbol{\Omega}$} & \multicolumn{1}{c}{$\boldsymbol{\delta}$} & \multicolumn{1}{c}{$\boldsymbol{\gamma}$} & \multicolumn{1}{c}{\textbf{PRISM} $\boldsymbol{\mathcal{B}}$} & \multicolumn{1}{c}{$\boldsymbol{|\Delta\mathcal{R}|}$} \\
\midrule
ARC & no reg & 332.98 & 325.46 & \underline{0.9976} & 82.4401 & \cellcolor{cyan!12} $0$ & 82.4401 & \textbf{0.2881} \\
 & replay $\lambda{=}0.01$ (baseline) & 332.98 & 324.38 & 0.9967 & 96.6146 & \cellcolor{cyan!12} $0$ & 96.6146 & \underline{0.2975} \\
 & shape reg.\ $\lambda{=}1.0$ (ours) & 332.98 & 327.55 & \textbf{0.9987} & 60.7069 & \cellcolor{cyan!12} $0$ & 60.7069 & 0.3009 \\
\midrule
MMLU & no reg & 333.53 & 328.73 & \underline{0.9991} & 52.3490 & \cellcolor{cyan!12} $0$ & 52.3490 & 0.1159 \\
 & replay $\lambda{=}0.01$ (baseline) & 333.53 & 328.17 & 0.9987 & 61.9216 & \cellcolor{cyan!12} $0$ & 61.9216 & \textbf{0.0847} \\
 & shape reg.\ $\lambda{=}1.0$ (ours) & 333.53 & 330.08 & \textbf{0.9992} & 46.8071 & \cellcolor{cyan!12} $0$ & 46.8071 & \underline{0.0954} \\
\midrule
SQuAD & no reg & 297.68 & 286.65 & \underline{0.9985} & 67.1415 & \cellcolor{cyan!12} $0$ & 67.1415 & 0.1353 \\
 & replay $\lambda{=}0.01$ (baseline) & 297.68 & 285.88 & 0.9982 & 72.9559 & \cellcolor{cyan!12} $0$ & 72.9559 & \textbf{0.1296} \\
 & shape reg.\ $\lambda{=}1.0$ (ours) & 297.68 & 287.67 & \textbf{0.9991} & 55.9438 & \cellcolor{cyan!12} $0$ & 55.9438 & \underline{0.1342} \\
\midrule
TriviaQA & no reg & 240.42 & 235.59 & \underline{0.9992} & 36.1152 & \cellcolor{cyan!12} $0$ & 36.1152 & \textbf{0.0035} \\
 & replay $\lambda{=}0.01$ (baseline) & 240.42 & 235.16 & 0.9989 & 42.8307 & \cellcolor{cyan!12} $0$ & 42.8307 & 0.0098 \\
 & shape reg.\ $\lambda{=}1.0$ (ours) & 240.42 & 236.13 & \textbf{0.9994} & 33.0855 & \cellcolor{cyan!12} $0$ & 33.0855 & \underline{0.0055} \\
\midrule
GSM8K & no reg & 269.06 & 261.82 & \textbf{1.0000} & 25.0486 & \cellcolor{cyan!12} $0$ & 25.0486 & \textbf{0.0196} \\
 & replay $\lambda{=}0.01$ (baseline) & 269.06 & 260.60 & \underline{1.0000} & 29.2693 & \cellcolor{cyan!12} $0$ & 29.2693 & 0.0288 \\
 & shape reg.\ $\lambda{=}1.0$ (ours) & 269.06 & 262.96 & 1.0000 & 21.1073 & \cellcolor{cyan!12} $0$ & 21.1073 & \underline{0.0251} \\
\bottomrule
\end{tabular}}
\end{table}

\section{Additional regularization baselines: EWC, L2-SP and layer freezing}
\label{app:reg_baselines}

Sec.~\ref{subsec:action} compares the shape regularizer against replay-CE, which isolates shape preservation from data re-fitting because both methods see the same $\mathcal{D}_{\mathrm{ref}}$. Replay is not, however, the strongest continual-learning baseline available. This appendix adds the two standard weight-space penalties, EWC~\citep{kirkpatrick2017overcoming} with a diagonal Fisher and L2-SP~\citep{li2018explicit}, and one architectural control, layer freezing (LoRA adapters on the top-$K$ layers only, $K \in \{4,8,16\}$).

\paragraph{Protocol, and why plasticity has to be matched.}
Every method runs under the paper's exact fine-tuning protocol (Sec.~\ref{subsec:exp_setup}): LoRA rank $32$, lr $10^{-5}$, $300$ steps, identical $\mathcal{D}_{\mathrm{ref}}$ and identical checkpoint cadence, differing only in the penalty. EWC's Fisher is accumulated over the LoRA-trainable parameters on $\mathcal{D}_{\mathrm{ref}}$ at the adapter's initialization, where the model equals the base model and is therefore at the optimum of the ``old task''; the Fisher is normalized to unit mean so that EWC and L2-SP share one $\lambda$ scale. A forgetting number is meaningless without the target-task loss beside it, since any method can drive downstream drift to zero by not learning, so each method is reported at the sweep configuration whose target-task loss comes closest to the shape-regularized run's, and the target loss is printed in the table. Three seeds ($42/43/44$); downstream forgetting is the unweighted mean $|\Delta\mathcal{R}|$ over the five held-out benchmarks at step $300$.

\begin{table}[h]
\centering
\caption{\textbf{Regularizer baselines at matched plasticity} (Llama-3.1-8B, TruthfulQA, three seeds, mean $\pm$ sd). Downstream forgetting is the mean $|\Delta\mathcal{R}|$ over the five held-out benchmarks; lower is better. Target-task loss is the answer-only CE on the fine-tuning task, and is what makes the comparison honest: layer freezing reaches a lower gap only by learning the task less.}
\label{tab:reg_baselines}
\small
\begin{tabular}{lcc}
\toprule
\textbf{Method} (TruthfulQA, 3 seeds) & \textbf{Downstream forgetting} $\downarrow$ & \textbf{Target-task loss} \\
\midrule
no regularization & $0.815 \pm 0.042$ & $0.950 \pm 0.008$ \\
replay-CE, $\lambda{=}0.01$ & $0.771 \pm 0.027$ & $0.899 \pm 0.007$ \\
L2-SP, $\lambda{=}0.01$ & $0.763 \pm 0.019$ & $0.901 \pm 0.003$ \\
EWC (diagonal Fisher), $\lambda{=}0.1$ & $0.751 \pm 0.015$ & $0.905 \pm 0.010$ \\
\textbf{shape regularizer (ours)}, $\boldsymbol{\lambda{=}1.0}$ & $\mathbf{0.680 \pm 0.016}$ & $\mathbf{0.872 \pm 0.005}$ \\
\midrule
layer freezing, top $16$ & $0.404 \pm 0.017$ & $0.924 \pm 0.001$ \\
\bottomrule
\end{tabular}
\end{table}

\paragraph{Result.}
Among the methods that learn the task to a comparable level, the shape regularizer gives the lowest downstream forgetting, it wins in each of the three seeds individually by $0.050$ to $0.157$, and it does so while carrying the \emph{lowest} target-task loss in the table ($0.872$ against $0.899$--$0.950$), so the reduction does not come at the expense of plasticity. The ordering among the baselines is itself informative: the two weight-space penalties (L2-SP $0.763$, EWC $0.751$) improve on replay only slightly, which is what the decomposition predicts. Under frozen-\texttt{lm\_head} LoRA the bound reduces to $\rho$ and $\Omega$ on the backbone, and constraining weights constrains features only indirectly, whereas the shape regularizer contracts the shape term itself.

Layer freezing is the exception that proves the point of matching plasticity. Its raw number, $0.404$, is by far the best in the table, but at a target-task loss of $0.924$ against the shape-regularized run's $0.872$: it reaches retention by learning the fine-tuning task less. Reported without the second column it would look like the strongest method here; reported with it, it is a different point on the plasticity--stability trade-off rather than a better one.

\paragraph{Consistency with the main text.}
These runs are a separate three-seed round; the paper's Table~\ref{tab:reg_compact_llama_truthfulqa} reports seed $42$ alone. The two agree: seed-42 no-reg / replay / shape regularizer read $0.843$ / $0.764$ / $0.681$ against the three-seed $0.815 \pm 0.042$ / $0.771 \pm 0.027$ / $0.680 \pm 0.016$, i.e.\ every main-text value falls inside its three-seed band, and the shape-regularizer figure the paper's headline rests on reproduces to $0.001$.

\paragraph{Sensitivity to the reference draw.}
The main-text run uses one draw of $32$ reference sequences. Holding everything else fixed (Llama-3.1-8B, TruthfulQA, $\lambda{=}1.0$, seed $42$, step $300$), we redraw $\mathcal{D}_{\mathrm{ref}}$ as four pairwise-disjoint windows of one shuffle of the task's held-out split, at $n \in \{8, 16, 32\}$. Mean downstream $|\Delta\mathcal{R}|$ over the four draws is $0.690 \pm 0.019$ at $n{=}8$, $0.685 \pm 0.011$ at $n{=}16$ and $0.676 \pm 0.008$ at $n{=}32$; all twelve runs fall in $0.666$--$0.716$, i.e.\ $15$--$21\%$ below the no-regularization $0.843$, and the spread shrinks with $n$ as sampling variation should. Re-running the paper's own draw gives $0.683$ against the reported $0.681$, so the draw effect is small but real, and the specific $32$ sequences are not load-bearing.

\paragraph{Scope.}
Methods for sequential continual instruction tuning target a stream of tasks rather than the single fine-tuning run studied here, so a fair comparison would require their benchmarks and protocols; that setting is listed among the future directions of Appendix~\ref{app:future_work}.

\section{When shape regularization helps: task dependence}
\label{app:reg_task_dependence}

TruthfulQA-FT, where shape drift dominates downstream forgetting, illustrates the regularizer's intended use case: the shape regularizer cuts mean $|\Delta\mathcal{R}|$ from 0.843 (no reg) to 0.681 ($-19\%$), outperforming experience replay (0.764, $-9\%$). BBQ-FT serves as a robustness check in the small-$|\Delta\mathcal{R}|$ regime: averaging the per-benchmark $|\Delta\mathcal{R}|$ values from Table~\ref{tab:reg_compare_llama_bbq} gives means of 0.179 (no reg), 0.195 (shape reg.), and 0.241 (replay), substantially smaller than TruthfulQA's 0.843, reflecting BBQ's modest baseline shape drift (mean baseline $\Omega = 0.932$ vs $0.906$). The regularizer is therefore ahead of replay but, in aggregate, $9\%$ above no regularization; the rest of this appendix explains why, and why the same regularizer is neutral on Qwen3-8B.

\paragraph{Geometric mechanism remains intact.}
On BBQ, the contrast between the shape regularizer and experience replay is most visible in the geometric signal: the shape regularizer lifts $\Omega$ from 0.93 to 0.98 (a $73\%$ reduction in $1-\Omega$) while experience replay leaves $\Omega$ flat (0.93 $\to$ 0.93). The regularizer's targeted quantity moves as designed.

\paragraph{Why BBQ is worse in aggregate.}
The shape regularizer's largest $|\Delta\mathcal{R}|$ reductions appear where forgetting is substantive: on BBQ-FT (Table~\ref{tab:reg_compare_llama_bbq}), TriviaQA $|\Delta\mathcal{R}|$ drops from 0.143 to 0.017 ($-88\%$) and GSM8K from 0.061 to 0.013 ($-79\%$); on TruthfulQA-FT (Table~\ref{tab:reg_compact_llama_truthfulqa}), SQuAD from 1.337 to 1.054 ($-21\%$) and TriviaQA from 2.583 to 2.124 ($-18\%$). The BBQ aggregate is dominated by ARC and MMLU, where the no-reg run shows substantial shape drift ($\Omega$ drops to 0.88 and 0.87) but small forgetting ($|\Delta\mathcal{R}|$ of 0.077 and 0.353): on these cells shape drift does not translate into $|\Delta\mathcal{R}|$, so restoring shape has little forgetting to remove, and both regularizers add some (the shape regularizer less than replay). A regularizer that targets $1-\Omega$ helps on the cells where $1-\Omega$ and $|\Delta\mathcal{R}|$ move together, which the decomposition of Sec.~\ref{subsec:decompose} identifies before training.

\paragraph{Qwen3-8B replication.}
The Qwen3-8B replication (Tables~\ref{tab:reg_compare_qwen_truthfulqa}--\ref{tab:reg_compare_qwen_bbq}, Fig.~\ref{fig:shape_reg_combined_qwen}) extends the regime dependence from \emph{tasks} (TruthfulQA-FT vs.\ BBQ-FT on Llama) to \emph{models}. Qwen3-8B is markedly more robust to LoRA forgetting on these two fine-tuning sources than Llama-3.1-8B: across-benchmark mean $|\Delta\mathcal{R}|$ at $\lambda{=}0$ is $0.263$ on Qwen TruthfulQA-FT vs.\ Llama's $0.843$ ($3.2\times$ smaller), and $0.112$ on Qwen BBQ-FT vs.\ Llama's $0.179$. Baseline shape preservation is correspondingly closer to its ceiling on Qwen ($\Omega$ at $\lambda{=}0$: $0.991$ on Qwen TruthfulQA-FT, $0.999$ on Qwen BBQ-FT vs.\ Llama's $0.906$ and $0.932$), leaving the shape regularizer with essentially no drift to repair: $1-\Omega$ collapses from Llama's $0.094$/$0.068$ to Qwen's $0.009$/$0.001$, an order of magnitude smaller. The differences across no-reg, replay, and the shape regularizer correspondingly shrink to the evaluation-noise floor: across-benchmark mean $|\Delta\mathcal{R}|$ on Qwen TruthfulQA-FT is $0.263$ / $0.282$ / $0.270$ (no-reg / replay / shape reg.), and on Qwen BBQ-FT is $0.112$ / $0.110$ / $0.112$, with method-to-method spread under $0.02$ on both. We report the Qwen3-8B regularization numbers in full for transparency and reproducibility, but emphasize that the shape regularizer's mechanism, contracting $1-\Omega$ at the source of LoRA backbone drift, can demonstrate a measurable benefit only when $1-\Omega$ has substantive headroom to contract. On Qwen3-8B there is no shape drift to repair, so these numbers do not contradict the Llama TruthfulQA-FT result; they confirm the regime dependence the bound itself predicts.

\section{Future work}
\label{app:future_work}

\begin{itemize}[leftmargin=*]
    \item \textbf{Beyond LoRA forgetting.} PRISM's diagnostic and regularization roles (Sec.~\ref{subsec:action}) extend naturally to broader fine-tuning (full SFT, distillation, continual learning), where backbone drift is substantially larger. The shape regularizer constrains drift; $\Omega$ and $\Delta\rho$ together monitor it per step in ways validation loss does not.
    \item \textbf{Diagnostic applications.} The single-forward-pass computation of $\mathcal{B}$ enables three deployment uses: \emph{per-sample OOD detection} via scale and shape residuals on individual inputs; \emph{cross-scale hyperparameter transfer} between a small proxy and its scaled target to predict tuning robustness; and \emph{production drift monitoring} on live traffic.
    \item \textbf{Beyond LLMs.} Vision Transformers and contrastive image encoders share LLMs' backbone-then-linear-head structure, with final-layer LayerNorm (and strict L2 normalization in CLIP/SigLIP) enforcing tight feature scales, making the PRISM scale axis naturally small. Natural targets include ViT quantization, vision-encoder distillation, and cross-modal alignment (consistent with the Platonic Representation Hypothesis~\citep{huh2024platonic}).
\end{itemize}

\end{document}